\documentclass{article}

\PassOptionsToPackage{numbers, compress}{natbib}

\usepackage[preprint]{neurips_2025}

\usepackage[utf8]{inputenc} 
\usepackage[T1]{fontenc}    
\usepackage{hyperref}       
\usepackage{url}            
\usepackage{booktabs}       
\usepackage{amsfonts}       
\usepackage{nicefrac}       
\usepackage{microtype}      
\usepackage{xcolor}    
\usepackage{amsmath}
\usepackage{amsthm}
\usepackage{graphicx}
\usepackage{amssymb}
\usepackage{graphicx}
\usepackage{multirow}
\usepackage{booktabs}
\usepackage{makecell}
\usepackage{subcaption}
\usepackage{changepage}

\newtheorem{theorem}{Theorem}

\newtheorem{assumption}{Assumption}

\usepackage{algorithm}
\usepackage{algpseudocode}
\newcommand{\methodname}{DRAUN}

\usepackage{etoolbox}
\AtEndEnvironment{algorithm}{\vspace{-0.1em}}

\title{\textit{\methodname}: An Algorithm-Agnostic Data Reconstruction Attack on Federated Unlearning Systems}

\author{Hithem Lamri,  Manaar Alam,  Haiyan Jiang,  and Michail Maniatakos \\
Center for Cyber Security, New York University Abu Dhabi, Abu Dhabi, United Arab Emirates \\
\texttt{\{hithem.lamri, alam.manaar, haiyan.jiang, michail.maniatakos\}@nyu.edu}
}

\begin{document}

\maketitle

\begin{abstract}
Federated Unlearning (FU) enables clients to remove the influence of specific data from a collaboratively trained shared global model, addressing regulatory requirements such as GDPR and CCPA. However, this unlearning process introduces a new privacy risk: A malicious server may exploit unlearning updates to reconstruct the data requested for removal, a form of Data Reconstruction Attack~(DRA). While DRAs for machine unlearning have been studied extensively in centralized Machine Learning-as-a-Service (MLaaS) settings, their applicability to FU remains unclear due to the decentralized, client-driven nature of FU. This work presents \textit{\methodname}, the first attack framework to reconstruct unlearned data in FU systems. \textit{\methodname} targets optimization-based unlearning methods, which are widely adopted for their efficiency. We theoretically demonstrate why existing DRAs targeting machine unlearning in MLaaS fail in FU and show how \textit{\methodname} overcomes these limitations. We validate our approach through extensive experiments on four datasets and four model architectures, evaluating its performance against five popular unlearning methods, effectively demonstrating that state-of-the-art FU methods remain vulnerable to DRAs. \vspace{-0.3cm}
\end{abstract}

\section{Introduction}\vspace{-0.2cm}
Federated Learning (FL) enables multiple clients to collaboratively train a shared global model without exposing their raw data~\cite{DBLP:conf/aistats/McMahanMRHA17}. Each client trains locally on its private dataset and sends local model updates to a central server. The server aggregates these updates over multiple communication rounds to iteratively improve the global model. Although FL is designed to preserve data privacy by keeping sensitive information on the client side, recent research shows that a malicious server can exploit the local model updates to reconstruct private client data through Data Reconstruction Attacks~(DRA)~\cite{DBLP:conf/nips/ZhuLH19,DBLP:conf/nips/GeipingBD020,DBLP:conf/cvpr/YinMVAKM21,DBLP:conf/iclr/ZhuB21,DBLP:conf/icml/WenGFGG22,DBLP:conf/sp/ZhaoSEEAB24}. These studies indicate that local model updates can be vectors for sensitive data leakage. Federated Unlearning (FU) extends the FL framework by allowing a client or a portion of its sensitive data to be fully removed from the global model even after several training rounds~\cite{elasticsga,manaar,abl,halimi,DBLP:journals/tkde/ShaikTLXCZL24,DBLP:conf/asiaccs/WangTZ0023}. To exercise the \textit{right to be forgotten}, a client submits unlearning updates to erase the influence of the targeted data from the global model. In such a setting, \textit{a malicious server can analyze these updates to launch DRAs targeting the samples requested for deletion.} This creates a new privacy threat inherent to the unlearning process, potentially violating data protection regulations like the GDPR~\cite{gdpr} and CCPA~\cite{ccpa}. Consequently, a detailed analysis of DRAs in FU is critical for understanding the scope of this risk and guiding the development of more robust FU methods.

Recent studies have investigated the feasibility of reconstructing unlearned data in traditional Machine Learning-as-a-Service (MLaaS) settings. \citet{DBLP:conf/sp/HuWDX24} introduced unlearning inversion attacks, demonstrating that even with black-box access, adversaries can recover labels or features of deleted data by comparing original and unlearned models. \citet{DBLP:conf/nips/BertranTKM0W24} further showed that even simple models like linear regression are vulnerable to exact reconstruction of deleted data after unlearning, highlighting that privacy leakage can occur without sophisticated architectures or attack vectors. Both studies emphasize a key idea: \textit{Naive unlearning methods can leave behind patterns in the model that can be used to recover unlearned samples}. However, extending these attacks directly to FU scenarios is much more challenging. In FU, unlearning is performed locally by the client, and the server sees only the resulting weight vector. It therefore lacks critical information such as which unlearning algorithm or hyper‑parameters were applied. Unlike MLaaS, where model training and unlearning are centrally orchestrated and more easily observed, FU introduces several challenges: \textbf{(1)}~client participation is random and changes over time, \textbf{(2)}~the global model is overwritten after every aggregation step, so the signal from a single unlearning update is rapidly mixed with fresh training updates, and \textbf{(3)}~the server lacks access to the specifics of local updates or unlearning operations. These factors make it difficult for an adversary to link specific changes in the model to particular unlearning requests. Moreover, FU typically requires multiple rounds to take effect, and any signal from deleted data becomes diluted over time. These challenges make it significantly harder to adapt existing DRA strategies from MLaaS to the FU setting. A brief background on FU and DRAs is provided in Appendix~\ref{appendix:background}. In this work, we investigate the following research question: \textit{To what extent are federated unlearning algorithms susceptible to data reconstruction attacks by a malicious server?}

We present \textit{\methodname}, a novel DRA targeting FU systems. To our knowledge, this is the first work to investigate such attacks in the context of FU. \textit{\methodname}~focuses on \textit{optimization-based unlearning algorithms}, including both first-order~\cite{elasticsga,manaar,abl,halimi} and second-order~\cite{2ndord_fu} methods, which are known for their high efficiency and are widely adopted in FU~\cite{fu_survey}. A key strength of \textit{\methodname}~is its ability to reconstruct unlearned data without any knowledge of the client's unlearning algorithm, which we consider an \textbf{algorithm-agnostic} attack. 
\begin{figure}[!t]
  \centering
  \includegraphics[width=0.75\linewidth]{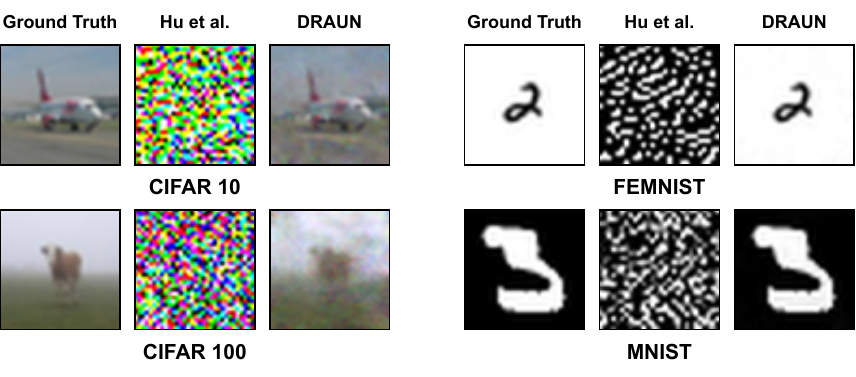}\vspace{-0.3cm}
  \caption{Reconstruction of deleted samples after unlearning using ABL~\cite{abl}. \textbf{Column~1}: Ground truth images (one randomly selected sample each from CIFAR10, CIFAR100, FEMNIST, and MNIST). \textbf{Column~2}: Reconstructions using a state-of-the-art DRA on MLaaS unlearning~(\citet{DBLP:conf/sp/HuWDX24}) directly extended to FU, which fails to recover any meaningful images. \textbf{Column~3}:~Reconstructions from \textit{\methodname} that closely resemble the ground truth.}\vspace{-0.4cm}
  \label{fig:reconstruction_comparison}
\end{figure}
To validate the effectiveness and generality of \textit{\methodname}, we conduct comprehensive experiments across \textit{four datasets} and \textit{four model architectures}. The qualitative results shown in Figure~\ref{fig:reconstruction_comparison} illustrate the reconstruction performance of different methods across representative unlearned samples. We also assess a range of potential defense strategies, providing insights to guide the development of more resilient FU techniques.

Our \textbf{\underline{contributions}} are as follows: \textbf{(1)} We introduce \textit{\methodname}, the first attack framework capable of reconstructing unlearned client data specifically in FU settings, whereas prior efforts have focused exclusively on centralized MLaaS environments. \textbf{(2)} We provide a formal theoretical analysis showing why conventional DRAs from MLaaS can be ineffective in optimization-based FU. \textbf{(3)} \textit{\methodname} operates in an algorithm-agnostic manner, requiring no knowledge of the client's unlearning method. \textbf{(4)} We evaluate \textit{\methodname} across four datasets and four model architectures, and demonstrate its effectiveness against five widely used unlearning algorithms, four first-order and one second-order. \textbf{(5)} We will open-source the implementation of \textit{\methodname}.\vspace{-0.2cm}

\section{\textit{\methodname} Methodology}\vspace{-0.2cm}
\label{section:methodology}
\subsection{Threat Model}\vspace{-0.2cm}
We consider a FL setting where a central server coordinates training across multiple clients. The system supports optimization-based FU, allowing clients to remove specific subsets of their local data from the global model. Following the standard approach used in the literature~\cite{fu_survey}, we consider the unlearning process as follows: \textbf{(1)} a client submits an unlearning request, \textbf{(2)} the server approves the request and ensures the client stays active in upcoming rounds, \textbf{(3)} in each round, the client performs one local unlearning epoch on the data to be removed and sends the updated model to the server, and \textbf{(4)} once the unlearning is done, the client notifies the server. As unlearning requires access to the data being erased, all computations occur on the client side.

\textbf{Client Capabilities:} Clients initiate unlearning and perform the necessary local computations. Each client selects its own unlearning algorithm. They also share the expected metadata with the server, including the total size of their local dataset and the size of the subset being deleted~\cite{fedavg_leak}.

\textbf{Server Capabilities:} We assume an honest-but-curious server that follows the protocol but may analyze client updates~\cite{DBLP:conf/nips/ZhuLH19,DBLP:conf/nips/GeipingBD020,DBLP:conf/cvpr/YinMVAKM21,DBLP:conf/iclr/ZhuB21,DBLP:conf/icml/WenGFGG22,DBLP:conf/sp/ZhaoSEEAB24,fedavg_leak}. The server has white-box access to the global model and all received updates. It manages unlearning requests, tracks client participation, and performs model aggregation. Importantly, the server does not know which unlearning algorithm each client employs. Following standard practice in the DRA literature, we assume that the server knows the labels associated with the client's data. This is a common assumption in Gradient Inversion Attack (GIA) literature~\cite{DBLP:conf/icml/WenGFGG22,DBLP:conf/nips/GeipingBD020}, and is justified by prior work showing that even if the labels are not explicitly available, they can often be accurately inferred or reconstructed from model gradients~\cite{fedavg_leak,seethrgrd}.\vspace{-0.175cm}

\subsection{Problem Landscape and Motivation}\label{sec:motivation}\vspace{-0.2cm}
We consider a server aiming to train a neural network $F(\theta_s, \cdot)$ with parameters $\theta_s$ for image classification, using a set of clients $\mathcal{C}$. Each client $c$ performs local training using Stochastic Gradient Descent (SGD) on its dataset $\mathcal{D}_c = \{(x_i, y_i)\}_{i=1}^n$, where $x_i$ is an input image and $y_i$ is its label. Upon receiving an unlearning request, the server sends the current global model $\theta_s$ to the requesting client, which aims to remove a subset $\mathcal{D}_u \subseteq \mathcal{D}_c$. The retained data is denoted $\mathcal{D}_r = \mathcal{D}_c \setminus \mathcal{D}_u$. Let $(x_u, y_u) \in \mathcal{D}_u$ represent the samples the client wishes to forget, and $(x_r, y_r) \in \mathcal{D}_r$ those it intends to retain. The client then applies a local unlearning algorithm $\mathcal{A}_c$, which takes $\mathcal{D}_u$, $\mathcal{D}_r$, and the current model $\theta_s$ as input, and returns an updated local model $\theta_c := \mathcal{A}_c(\mathcal{D}_u, \mathcal{D}_r, \theta_s)$. This algorithm typically involves optimization over a custom unlearning loss $\mathcal{L}_u$, controlled by hyperparameters such as the unlearning rate $\eta$, batch size $m$, and number of local epochs $\mathcal{E}$.

GIA is a widely studied data reconstruction attack that estimates the private training data by optimizing dummy inputs such that the resulting gradients match those observed for the target client~\cite{DBLP:conf/nips/ZhuLH19,DBLP:conf/nips/GeipingBD020,DBLP:conf/cvpr/YinMVAKM21}. This is typically achieved by minimizing a similarity loss $\mathcal{L}_{\text{sim}}$ between the true and dummy gradients, using metrics such as the $\ell_2$-norm~\cite{DBLP:conf/nips/ZhuLH19} or cosine similarity~\cite{DBLP:conf/nips/GeipingBD020}. A core assumption in GIA is that the attacker uses the same loss function as the target client during optimization (see Appendix~\ref{appendix:gia}). This ensures that the gradients used in the inversion process are meaningful approximations of the client's true gradients. However, this assumption no longer holds in the FU setting: The loss function varies across different unlearning algorithms, depending on how they handle $\mathcal{D}_u$ and $\mathcal{D}_r$. Some algorithms apply gradient ascent on $\mathcal{D}_u$~\cite{elasticsga,halimi}, while others aim to minimize the difference between gradients on $\mathcal{D}_u$ and $\mathcal{D}_r$~\cite{abl,manaar}. Table~\ref{tab:unlearning_algorithms} presents the loss functions and model updates for various optimization-based FU algorithms. As a result, the gradient landscape diverges from that of standard training, making GIAs substantially less effective and often unstable in FU scenarios.\vspace{-0.175cm}
\begin{table}[!t]
\centering
{\small
\caption{Unlearning losses and local update rules across FU methods, where $\alpha$, $\beta$, $\gamma$, $\delta$, and $\eta$ are algorithm-specific hyperparameters.}
\label{tab:unlearning_algorithms}
\begin{tabular}{lc}
\toprule
\textbf{Algorithm} & \textbf{Local Update Rules} \\
\midrule
\citet{elasticsga}  & $\theta_u \gets \theta_s +  \eta\nabla_{\theta_s} \mathcal{L}(x_u, y_u)$ \\
\citet{halimi} & $\theta_u \gets \theta_s + \eta \nabla_{\theta_s} \mathcal{L}(x_u, y_u) + \delta \|\theta_s - \theta_u\|$ \\
\citet{abl}     & $\theta_u \gets \theta_s - \eta \big( \nabla_{\theta_s} \mathcal{L}(x_r, y_r) - \nabla_{\theta_s} \mathcal{L}(x_u, y_u) \big)$ \\
\citet{manaar}   & $\theta_u \gets \theta_s - \eta \big( \alpha \nabla_{\theta_s} \mathcal{L}(x_r, y_r) - \beta \nabla_{\theta_s} \mathcal{L}(x_u, y_u) + \gamma \big\| \frac{\theta_u}{\theta_s} \big\| \big)$ \\  
\bottomrule
\end{tabular}}\vspace{-0.5cm}
\end{table}

\subsection{\textit{\methodname} Overview}\vspace{-0.2cm}
\textit{\methodname}~builds upon the standard structure of a GIA, where the attacker reconstructs private training data by minimizing the difference between true and simulated (dummy) gradients. Specifically, we adopt the similarity-based loss introduced by~\citet{DBLP:conf/nips/GeipingBD020}, formulated in Equation~\eqref{eq:invg}. The loss minimizes the cosine distance between the client's true gradient $\nabla_c$ and a dummy gradient $\tilde{\nabla}_c$, with an additional Total Variation (TV) penalty~\cite{RUDIN1992259}, controlled by hyperparameter $\lambda_{\text{TV}}$, that acts as a regularizer for image smoothness:

\begin{equation}
\label{eq:invg}
x = \arg\min\limits_{\tilde{x} \in \mathcal{X}} \mathcal{L}_{sim}(\nabla_c, \tilde{\nabla}_c) = \arg\min\limits_{\tilde{x} \in \mathcal{X}} \left(1 - \frac{\langle \nabla_c, \tilde{\nabla}_c \rangle}{|\nabla_c| \cdot |\tilde{\nabla}_c|} + \lambda_{TV} \cdot TV(\tilde{x})\right)
\end{equation}

As discussed in Section~\ref{sec:motivation}, this standard GIA objective is effective only under specific conditions, when the client's local loss $\mathcal{L}_u$ depends solely on the unlearn dataset $\mathcal{D}_u$. However, many FU algorithms use more general update rules that depend on both $\mathcal{D}_u$ and a retain dataset $\mathcal{D}_r$ (see Table~\ref{tab:unlearning_algorithms}). To extend GIAs to these more general settings, \textit{\methodname}~simulates the client's local update using a surrogate optimization procedure.

The server first estimates the client's average gradient using the received model $\theta_c$ and the global model $\theta_s$ as $\nabla \bar{\theta}_c = \frac{1}{U_c} (\theta_s - \theta_c)$ (Algorithm~\ref{alg:attack_overview}, line 3), where $U_c$ is the number of client steps.
\begin{algorithm}[!t]
    {\small
   \caption{Overview of \textit{\methodname}}
   \label{alg:attack_overview}
\begin{algorithmic}[1]
   \State \textbf{Input:} $\theta_s$, $\theta_c$, $|\mathcal{D}_u|$, $m$ $\mathcal{E}$, $\lambda_{TV}$, $\beta$, $\eta_{\text{unl}}$, $\eta_{\text{rec}}$, $y_r$, $y_u$
   \State $U_c \gets \frac{|\mathcal{D}_u| \cdot \mathcal{E}}{m}$ \Comment{Number of local steps}
   \State $\nabla \bar{\theta}_c \gets \frac{1}{U_c} (\theta_s - \theta_c)$
   \State $\tilde{x}_r, \tilde{x}_u \gets \mathcal{I}(|\mathcal{D}_u|)$
   \For{$t = 1$ to $T$}
       \State $\tilde{\nabla}^{(1)}, \tilde{\nabla}^{(0)} \gets \mathcal{A}_{\text{approx}}(\theta_s, \tilde{x}_u, y_u,\tilde{x}_r, y_r,\mathcal{E},\eta_{\text{unl}} )$
       \State $\ell_1 \gets \mathcal{L}_{\text{sim}}(\nabla \bar{\theta}_c, \tilde{\nabla}^{(1)}) + \lambda_{\text{TV}} \cdot \left( \beta \cdot TV(\tilde{x}_u) + (1 - \beta) \cdot TV(\tilde{x}_r) \right)$ \Comment{$\beta \approx 1$}
       \State $\ell_0 \gets \mathcal{L}_{\text{sim}}(\nabla \bar{\theta}_c, \tilde{\nabla}^{(0)}) + \lambda_{\text{TV}} \cdot \left( \beta \cdot TV(\tilde{x}_u) + (1 - \beta) \cdot TV(\tilde{x}_r) \right)$ \Comment{$\beta \approx 1$}
       \State $\ell \gets \min(\ell_0, \ell_1)$
       \State $\tilde{x}_u \gets \tilde{x}_u - \eta_{\text{rec}} \cdot \frac{\partial \ell}{\partial \tilde{x}_u}$
       \State $\tilde{x}_r \gets \tilde{x}_r - \eta_{\text{rec}} \cdot \frac{\partial \ell}{\partial \tilde{x}_r}$
   \EndFor
   \State \textbf{return} $\tilde{x}_u$
\end{algorithmic}}
\end{algorithm}
The server then attempts to reconstruct the unlearned input $x_u$ by optimizing dummy inputs $\tilde{x}_u$ and $\tilde{x}_r$, which represent proxy samples for the unlearn and retain subsets respectively. Since the server does not know the client's actual unlearning algorithm $\mathcal{A}_c$, it uses a surrogate procedure $\mathcal{A}_{approx}$ (Algorithm~\ref{alg:approx}) for its approximation.
\begin{algorithm}[!t]
    {\small
   \caption{\textit{\methodname}'s Algorithm Approximation with $\mathcal{A}_{\text{approx}}$ (Algorithm~\ref{alg:attack_overview}, line 6)}
   \label{alg:approx}
\begin{algorithmic}[1]
   \State \textbf{Input:} $\theta_s$, $\tilde{x}_u$, $y_u$, $\tilde{x}_r$, $y_r$, $\mathcal{E},\eta_{\text{unl}}$
   \State Initialize two dummy models: 
   $\tilde{\theta}_c^{(0)} \gets \theta_s$, $\tilde{\theta}_c^{(1)} \gets \theta_s$
   \For{$t = 1$ to $\mathcal{E}$}
       \State \textbf{Update $\tilde{\theta}_c^{(1)}$ with $\alpha = 1$:} \Comment{gradient difference}
       \Statex \hspace{1em}
       $\tilde{\theta}_c^{(1)} \gets \tilde{\theta}_c^{(1)} - \eta_{\text{unl}} \cdot \left[  \left( \nabla_{\theta} \mathcal{L}(\tilde{\theta}_c^{(1)}, \tilde{x}_r, \tilde{y}_r) - \nabla_{\theta} \mathcal{L}(\tilde{\theta}_c^{(1)}, \tilde{x}_u, \tilde{y}_u) \right) + \delta \cdot \nabla_{\theta} \|\tilde{\theta}_c^{(1)} - \theta_s\|_2 \right]$
       \State \textbf{Update $\tilde{\theta}_c^{(0)}$ with $\alpha = 0$:} \Comment{gradient ascent}
       \Statex \hspace{1em}
       $\tilde{\theta}_c^{(0)} \gets \tilde{\theta}_c^{(0)} + \eta_{\text{unl}} \cdot  \nabla_{\theta} \mathcal{L}(\tilde{\theta}_c^{(0)}, \tilde{x}_u, \tilde{y}_u) + \delta \cdot \nabla_{\theta} \|\tilde{\theta}_c^{(0)} - \theta_s\|_2 $
   \EndFor 
   \State \textbf{Return:}
   \Statex \hspace{1em}
   $\tilde{\nabla}^{(1)} \gets \tilde{\theta}_c^{(1)} - \theta_s$ \Comment{Averaged gradient for $\alpha = 1$}
   \Statex \hspace{1em}
   $\tilde{\nabla}^{(0)} \gets \tilde{\theta}_c^{(0)} - \theta_s$ \Comment{Averaged gradient for $\alpha = 0$}
\end{algorithmic}}
\end{algorithm}
This routine simulates the client update by minimizing a surrogate unlearning loss ($\tilde{\mathcal{L}}_u$), defined by the following constrained objective (Equation~\eqref{eq:minmax-unlearning}):
\begin{equation}
\label{eq:minmax-unlearning}
\min\limits_{\theta_c \in \Theta} \tilde{\mathcal{L}}_u = \alpha \cdot \mathcal{L}(\theta_s, x_r, y_r) - \mathcal{L}(\theta_s, x_u, y_u)
\quad \text{s.t.} \quad \|\theta_c - \theta_s\|_2 \leq \delta
\end{equation}
Here, $\alpha \in [0,1]$ determines the contribution of the retain dataset, $\delta$ constrains the model divergence from the global model, and $\mathcal{L}$ is the cross-entropy loss. The constraint is relaxed into an unconstrained objective using a Lagrange multiplier, resulting in Equation~\eqref{eq:simple-obj}:
\begin{equation}
\label{eq:simple-obj}
\min\limits_{\theta_c \in \Theta} \alpha \cdot \mathcal{L}(\theta_s, x_r, y_r) - \mathcal{L}(\theta_s, x_u, y_u) + \delta \cdot \|\theta_c - \theta_s\|_2
\end{equation}
Algorithm~\ref{alg:approx} optimizes two dummy models using an unlearning rate $\eta_{\text{unl}}$ under this objective: One for $\alpha = 1$ (gradient difference-based updates), and another for $\alpha = 0$ (gradient ascent-based updates). After $\mathcal{E}$ optimization steps,
it returns two surrogate gradients: $\tilde{\nabla}^{(1)}$ and $\tilde{\nabla}^{(0)}$. These are intended to approximate the gradient update a client would have produced under two extremes of unlearning behavior. Since the server has no knowledge of the client's actual unlearning algorithm, it uses both of these as candidates for matching the observed averaged gradient $\nabla \bar{\theta}_c$.

In each iteration of Algorithm~\ref{alg:attack_overview}, the server computes two loss values: $\ell_1$ and $\ell_0$, corresponding to the cosine similarity between $\nabla \bar{\theta}_c$ and the surrogate gradients $\tilde{\nabla}^{(1)}$ and $\tilde{\nabla}^{(0)}$, respectively (lines~7-8). These losses use the formulation in Equation~\eqref{eq:invg}, and each includes a TV regularization term applied separately to $\tilde{x}_u$ and $\tilde{x}_r$. A mixing coefficient $\beta \approx 1$ balances the contribution of the TV terms, prioritizing smoothness in $\tilde{x}_u$. The final loss $\ell$ is selected as the minimum of $\ell_0$ and $\ell_1$ (line 9). The dummy inputs are then updated via gradient descent with respect to this loss (lines 10-11),  using a reconstruction step size $\eta_{\text{rec}}$, and continuing for $T$ iterations.

A critical design choice in \textit{\methodname} is the initialization of dummy inputs. Since the server lacks access to the true structure of the client's data, the dummy inputs for both $\mathcal{D}_u$ (unlearn) and $\mathcal{D}_r$ (retain) are initialized as random tensors. Specifically, we sample two sets of size $|\mathcal{D}_u|$ each from a uniform distribution over $[0,1]$, denoted as $\mathcal{U}([0,1], |\mathcal{D}_u|)$ in Algorithm~\ref{alg:initialization}. Our goal is not to reconstruct the retain set, but to use it to help denoise the gradients from the unlearn set and thus facilitate accurate reconstruction. To improve optimization and avoid convergence to poor local minima, a known issue in some unlearning algorithms, we ensure that the dummy inputs for $\mathcal{D}_u$ and $\mathcal{D}_r$ are well separated at the start of reconstruction. This is achieved by iteratively adding multivariate Gaussian noise (with standard deviation $\sigma$ and dimensionality $d_x$, matching that of a single input sample) to each image in $\mathcal{D}_r$, until the minimum Frobenius distance between corresponding image pairs in $\mathcal{D}_u$ and $\mathcal{D}_r$ exceeds a predefined threshold $\Delta$. This guarantees that the convex hulls of the unlearn and retain dummy inputs are disjoint: $\text{Conv}\{\tilde{\mathcal{D}}_u\} \cap \text{Conv}\{\tilde{\mathcal{D}}_r\} = \varnothing$.
This initialization strategy, detailed in Algorithm~\ref{alg:initialization}, is particularly important for unlearning methods that rely on coupled update dynamics. A theoretical justification is provided in Appendix~\ref{appendix:theory}.\vspace{-0.2cm}

\begin{algorithm}[!t]
    {\small
   \caption{\textit{\methodname}'s Input Initialization with $\mathcal{I}$ (Algorithm~\ref{alg:attack_overview}, line 4)}
   \label{alg:initialization}
\begin{algorithmic}[1]
   \State \textbf{Input:} $|\mathcal{D}_u|$
   \State $\tilde{x_u} \gets \mathcal{U}([0,1],  |\mathcal{D}_u|)$ \Comment{Initialize dummy unlearn input}
   \State $\tilde{x_r} \gets \mathcal{U}([0,1],  |\mathcal{D}_u|)$ \Comment{Initialize dummy retain input }
   \For{$i = 1$ to $|\mathcal{D}_u|$} \Comment{The i-th samples in $\tilde{x_u}, \tilde{x_r}$ }
       \While{$\| \tilde{x_u}^{(i)} - \tilde{x_r}^{(i)} \|_F \leq \Delta$} \Comment{$\|.\|_F$ Frobenius norm}
           \State $\epsilon \sim \mathcal{N}(0, \sigma^2 \cdot I_{d_x})$
           \State $\tilde{x_r}^{(i)} \gets \tilde{x_r}^{(i)} + \epsilon$
       \EndWhile
   \EndFor
   \State \textbf{Return:} $\tilde{x_u}, \tilde{x_r}$
\end{algorithmic}}
\end{algorithm}

\section{Theoretical Analysis} \label{sec:theory_ana}
\vspace{-0.2cm}
In this section, we demonstrate why classical GIAs fail to reconstruct unlearned inputs from unlearning updates involving two datasets, such as gradient difference methods~\cite{abl, manaar}. Recall that $\mathcal{L}_u$ and $\tilde{\mathcal{L}}_u$ are client's local unlearning loss and the server's surrogate unlearning loss (used to generate dummy gradients), respectively. $\tilde{\mathcal{L}}_u$ depends on the dummy unlearn input $\tilde{x}_u$; hence, $\mathcal{L}_{\text{sim}}$ implicitly depends on $\tilde{x}_u$, not just model parameters $\theta$. For clarity, we denote $\tilde{x}_u$ by $x$. To avoid high-order tensor calculations, we treat inputs as flattened vectors of dimension $d_x$. While we assume single unlearn/retain inputs for simplicity, this analysis extends to multiple inputs via averaging:
\[
\mathcal{L}(\theta, x_r, y_r) = \frac{1}{|\mathcal{D}_r|} \sum_{x_i \in \mathcal{D}_r} \mathcal{L}(\theta, x_i, y_i), \quad  
\mathcal{L}(\theta, x_u, y_u) = \frac{1}{|\mathcal{D}_u|} \sum_{x_j \in \mathcal{D}_u} \mathcal{L}(\theta, x_j, y_j).
\]
In order to prove that classical GIA fails, we adopt the following assumptions:
\begin{assumption}
\label{assumption:1}
Without loss of generality, we assume the similarity loss $\mathcal{L}_{\text{sim}}$ from Equation~(\ref{eq:invg}) is the squared \( \ell_2 \)-norm:
\[
\mathcal{L}_{\text{sim}}(\nabla_{\theta} \mathcal{L}_u, \nabla_{\theta} \tilde{\mathcal{L}}_u) = \left\| \nabla_{\theta} \mathcal{L}_u - \nabla_{\theta} \tilde{\mathcal{L}}_u \right\|_2^2.
\]
Thus, the classical GIA objective becomes:
\[
x^{*} = \arg\min_{x \in \mathcal{X}} \left\| \nabla_{\theta} \mathcal{L}_u - \nabla_{\theta} \tilde{\mathcal{L}}_u \right\|_2^2.
\]
\end{assumption}

\begin{assumption} \label{assumption:2}
The surrogate loss $\tilde{\mathcal{L}}_u$ is twice differentiable in $x$ and $\theta$. The Jacobian
\[
J(x) = \frac{\partial \nabla_{\theta} \tilde{\mathcal{L}}_u}{\partial x} \in \mathbb{R}^{d_\theta \times d_x},
\]
where $d_x$ and $d_\theta$ are the dimensions of $x$ and $\theta$, exists everywhere.
\end{assumption}

\begin{assumption} \label{assumption:3}
We assume $d_\theta \geq d_x$ and that $J(x)$ has full column rank (i.e., $\text{rank}(J(x)) = d_x$) $\forall x$.
\end{assumption}

\begin{theorem} \label{theorem:1}
Suppose the client's unlearning loss $\mathcal{L}_u$ implicitly depends on both $x_u$ and $x_r$ through $\mathcal{L}_u = \mathcal{L}(\theta, x_r, y_r) - \mathcal{L}(\theta, x_u, y_u)$. Under Assumptions~\ref{assumption:1}-\ref{assumption:3}, the ground truth unlearn input $x_u$ is not a local minimizer of $\mathcal{L}_{\text{sim}}$.
\end{theorem}

\begin{proof}
We show that Fermat's stationary condition does not hold at $x = x_u$.

First, we compute the gradient of $\mathcal{L}_{\text{sim}}$ with respect to $x$:
\[
\frac{\partial \mathcal{L}_{\text{sim}}}{\partial x} = \frac{\partial}{\partial x} \left\| \nabla_{\theta} \mathcal{L}_u - \nabla_{\theta} \tilde{\mathcal{L}}_u \right\|_2^2 = -2 J(x)^T \left( \nabla_{\theta} \mathcal{L}_u - \nabla_{\theta} \tilde{\mathcal{L}}_u \right),
\]
where $J(x) = \frac{\partial \nabla_{\theta} \tilde{\mathcal{L}}_u}{\partial x}$. Substituting $\mathcal{L}_u = \mathcal{L}(\theta, x_r, y_r) - \mathcal{L}(\theta, x_u, y_u)$ and $\tilde{\mathcal{L}}_u = \mathcal{L}(\theta, x, y_u)$:
\[
\frac{\partial \mathcal{L}_{\text{sim}}}{\partial x} = -2 J(x)^T \left( \nabla_{\theta} \mathcal{L}(\theta, x_r, y_r) - \nabla_{\theta} \mathcal{L}(\theta, x_u, y_u) - \nabla_{\theta} \mathcal{L}(\theta, x, y_u) \right).
\]

Then, we evaluate the expression at $x = x_u$:
\[
\frac{\partial \mathcal{L}_{\text{sim}}}{\partial x} \bigg|_{x=x_u} = -2 J(x_u)^T \left( \nabla_{\theta} \mathcal{L}(\theta, x_r, y_r) - 2 \nabla_{\theta} \mathcal{L}(\theta, x_u, y_u) \right).
\]

By Assumption~\ref{assumption:3}, $J(x_u)$ has full column rank. Since $\nabla_{\theta} \mathcal{L}(\theta, x_r, y_r) \neq 2 \nabla_{\theta} \mathcal{L}(\theta, x_u, y_u)$ as $x_u\not = x_r$, the gradient at $x_u$ is non-zero. By Fermat's optimality condition, $x_u$ cannot be a local minimizer.
\end{proof}

\begin{assumption}
\label{assumption:4}
The surrogate loss $\tilde{\mathcal{L}}_u$ is:
$\mu_x$-smooth in $x$: $\left\| \nabla_x \tilde{\mathcal{L}}_u(x_1) - \nabla_x \tilde{\mathcal{L}}_u(x_2) \right\|_2 \leq \mu_x \|x_1 - x_2\|_2$,
 and $\mu_{\theta}$-smooth in $\theta$: $\left\| \nabla_{\theta} \tilde{\mathcal{L}}_u(\theta_1) - \nabla_{\theta} \tilde{\mathcal{L}}_u(\theta_2) \right\|_2 \leq \mu_{\theta} \|\theta_1 - \theta_2\|_2$.
\end{assumption}

\begin{theorem}
\label{theorem:2}
Let $x_u^{*}$ be a minimizer of the loss function of the classical GIA:
\[
x_u^{*} = \arg\min_{x \in \mathcal{X}} \mathcal{L}_{\text{sim}}\left( \nabla_{\theta} \mathcal{L}_u, \nabla_{\theta} \tilde{\mathcal{L}}_u \right),
\]

If Assumptions~\ref{assumption:1}-~\ref{assumption:4} hold, then the reconstruction error, defined as the \( \ell_2 \)-norm between the flattened minimizer and the ground truth input, satisfies:
\[
\| x_u^{*} - x_u \|_2 \geq \frac{ \| J^{T} \nabla_{\theta} \mathcal{L}(\theta, x_r, y_r) \|_2 }{ \mu_x \| J \|_F + 2 \mu_{\theta} \| \nabla_{\theta} \mathcal{L}_u \|_2 },
\]

\end{theorem}

The proof of Theorem~\ref{theorem:2} and a more detailed theoretical analysis is provided in Appendix~\ref{appendix:theory}.

\textbf{Interpretation:} Theorem~\ref{theorem:1} shows classical GIA fails to converge to $x_u$ when $x_r$ influences $\mathcal{L}_u$ and Theorem~\ref{theorem:2} provides a lower bound on its reconstruction error, proving perfect reconstruction ($\|x_u^{*} - x_u\|_2 = 0$) is impossible.

\vspace{-0.2cm}
\section{Experiments}\vspace{-0.2cm}
\label{sec:experiments}
\subsection{Experimental Setup}\vspace{-0.2cm}
\label{subsec: exper_setup}
\textbf{Datasets and Models:}
We evaluate \textit{\methodname}~on four standard image classification datasets: CIFAR10~\cite{cifar10_100}, CIFAR100~\cite{cifar10_100}, FEMNIST~\cite{caldas2018leaf}, and MNIST~\cite{lecun1998mnist} (see Appendix~\ref{appendix:dataset} for details on datasets). To ensure a fair comparison with the reconstruction performance of~\citet{DBLP:conf/sp/HuWDX24}, we use an 8-layer ConvNet for all main experiments (see Appendix~\ref{appendix:convnet} for its architecture). Reconstruction results using other model architectures, such as MLP, LeNet~\cite{lecun1998mnist}, and ResNet18~\cite{resnet}, are reported in Appendix~\ref{appendix:recon_results}.

\textbf{Federated Unlearning Setup:}
We simulate an FL environment with 100 clients. In each communication round, 10 clients are selected uniformly at random to participate. Each selected client trains a local model using stochastic gradient descent with a learning rate of 0.1, for 2 local epochs, and a batch size of 128. The global model is trained for 100 rounds and converges by round 90 based on clean accuracy. After convergence, we simulate FU by targeting the data of a single client for removal. We evaluate both single-step and multi-step FU methods, where the number of unlearning steps is denoted by $\mathcal{E} \in \{1, 2, 4\}$. We consider $\mathcal{E}=1$ for the main experiments; results for $\mathcal{E}=2$ and $\mathcal{E}=4$ are provided in Appendix~\ref{appendix:ablation_unlearn_steps}. To evaluate the effect of target data size, we vary the size of the unlearned dataset as $|\mathcal{D}_u| \in 2^i, \text{where } i \in [0, 7]$. We evaluate four FU methods based on first-order optimization proposed by~\citet{elasticsga}, \citet{halimi}, \citet{manaar}, and ABL~\cite{abl}.

\textbf{Reconstruction Setup:}
For reconstruction, \textit{\methodname} uses the Adam optimizer with $\eta_{\text{rec}} = 0.1$, $\lambda_{\text{reg}} = 10^{-6}$, and $\beta = 0.9$. The number of reconstruction iterations $T$ varies between 6,000 and 24,000 across experiments. The other hyperparameters used in Algorithm~\ref{alg:approx} and Algorithm~\ref{alg:initialization} are $\eta_{\text{unl}}=0.1$, $\delta=10.0$, $\Delta=5.0$, and $\sigma=1$.

\textbf{Evaluation Metrics:}
To evaluate the quality of reconstructed images, we use standard image similarity metrics that assess luminance, contrast, and structure with the ground truth. Specifically, we use the Structural Similarity Index Measure (SSIM)~\cite{wang2004image}, the Learned Perceptual Image Patch Similarity (LPIPS)~\cite{zhang2018unreasonable}, and the Peak Signal-to-Noise Ratio (PSNR)~\cite{DBLP:conf/sp/HuWDX24}. Lower LPIPS scores and higher PSNR and SSIM values (with SSIM close to 1) indicate better reconstruction quality.

\textbf{Compute Resources:} We used a cluster with 8 NVIDIA A100 80GB GPUs, 255 AMD EPYC 7763 64-Core Processor CPUs, and 2TB of RAM to run all experiments.

\subsection{\textit{\methodname}~Reconstruction Efficiency}\vspace{-0.2cm}
\textbf{Visual Reconstruction Results:}
We evaluate the effectiveness of \textit{\methodname}~in reconstructing unlearned images from local updates produced by four unlearning algorithms. Figure~\ref{fig:cifar_methods_comparison} shows one randomly selected example from CIFAR10 dataset, with reconstructions generated by \textit{\methodname} using local updates obtained from \citet{elasticsga}, ABL~\cite{abl}, \citet{manaar}, and~\citet{halimi}. The reconstructions closely resemble the original images, as indicated by high SSIM and PSNR values and low LPIPS scores. These results show that the unlearning methods leave substantial residual information in the local updates, which \textit{\methodname} can exploit to reconstruct the original data accurately. Additional reconstructed examples for CIFAR100, FEMNIST and MNIST are provided in Appendix~\ref{appendix:reconstructed_images}.
\begin{figure}[!t]
    \centering
    \resizebox{0.85\textwidth}{!}{%
    \begin{tabular}{c}
        \begin{subfigure}[t]{0.17\linewidth}
            \centering
            \textbf{Ground Truth}\vspace{0.05cm}
            \fbox{\includegraphics[width=\linewidth]{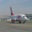}}
        \end{subfigure}
        \hspace{0.05cm}
        \begin{subfigure}[t]{0.17\linewidth}
            \centering
            \textbf{\citet{elasticsga}}
            \fbox{\includegraphics[width=\linewidth]{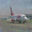}}\\
            {$\uparrow$SSIM: 0.93\\$\uparrow$PSNR: 33.9\\$\downarrow$LPIPS: 0.04}
        \end{subfigure}
        \hspace{0.05cm}
        \begin{subfigure}[t]{0.17\linewidth}
            \centering
            \textbf{ABL~\cite{abl}}
            \fbox{\includegraphics[width=\linewidth]{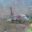}}\\
            {$\uparrow$SSIM: 0.80\\$\uparrow$PSNR: 26.2\\$\downarrow$LPIPS: 0.17}
        \end{subfigure}
        \hspace{0.05cm}
        \begin{subfigure}[t]{0.17\linewidth}
            \centering
            \textbf{\citet{manaar}}
            \fbox{\includegraphics[width=\linewidth]{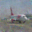}}\\
            {$\uparrow$SSIM: 0.86\\$\uparrow$PSNR: 30.7\\$\downarrow$LPIPS: 0.10}
        \end{subfigure}
        \hspace{0.05cm}
        \begin{subfigure}[t]{0.17\linewidth}
            \centering
            \textbf{\citet{halimi}}
            \fbox{\includegraphics[width=\linewidth]{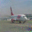}}\\
            {$\uparrow$SSIM: 0.94\\$\uparrow$PSNR: 32.6\\$\downarrow$LPIPS: 0.05}
        \end{subfigure}
    \end{tabular}
    }\vspace{-0.1cm}
    \caption{\textit{\methodname} reconstructions from local updates of four unlearning methods on CIFAR10. Higher SSIM and PSNR, and lower LPIPS scores indicate strong similarity to the original images, revealing residual information in the updates.\vspace{-0.3cm}}
    \label{fig:cifar_methods_comparison}
\end{figure}

\textbf{Quantitative Comparison with State-of-the-Art:}
Table~\ref{table:reconstruction_comparison} provides a quantitative comparison between \textit{\methodname} and the state-of-the-art DRA by~\citet{DBLP:conf/sp/HuWDX24}, originally developed for MLaaS unlearning and directly adapted here to the FU setting. The evaluation includes all four unlearning algorithms discussed earlier, applied to CIFAR10, CIFAR100, MNIST, and FEMNIST datasets. Each entry shows the average performance over reconstructions of 10 randomly selected images per dataset, using the same metrics described previously.
On all datasets, \textit{\methodname} clearly outperforms~\citet{DBLP:conf/sp/HuWDX24} when using updates from ABL~\cite{abl} and~\citet{manaar}. For example, on CIFAR10 with ABL~\cite{abl},~\citet{DBLP:conf/sp/HuWDX24} achieves near-zero SSIM (0.0038) and very low PSNR (5.48) values and extremely high LPIPS (0.9358), indicating failed reconstructions (see Figure~\ref{fig:reconstruction_comparison} for examples of such failed reconstructions). In contrast, \textit{\methodname} achieves a better reconstruction (SSIM: 0.6407, PSNR:~21.11, LPIPS: 0.3175). A similar gap is observed for~\citet{manaar}. However, for the unlearning algorithms by~\citet{elasticsga} and~\citet{halimi},~\citet{DBLP:conf/sp/HuWDX24} slightly outperforms or matches \textit{\methodname} in some cases (e.g., CIFAR100 with~\citet{elasticsga}). This is due to the design of these specific unlearning algorithms, which do not rely on coupled update dynamics of the retained and unlearned datasets (see Table~\ref{tab:unlearning_algorithms}). As a result,~\citet{DBLP:conf/sp/HuWDX24} can still reconstruct meaningful images in these cases. Nonetheless, \textit{\methodname} consistently achieves better performance than~\citet{DBLP:conf/sp/HuWDX24} across most datasets and unlearning algorithms, especially on complex datasets such as CIFAR10 and CIFAR100. These results validate the robustness and generalizability of \textit{\methodname} in FU contexts.
\begin{table}[!t]
\centering
\caption{Reconstruction quality comparison between \textit{\methodname} and the method by~\citet{DBLP:conf/sp/HuWDX24} across four unlearning algorithms and datasets. Higher SSIM and PSNR, and lower LPIPS scores indicate better reconstruction quality.\vspace{0.1cm}}
\label{table:reconstruction_comparison}
\resizebox{\linewidth}{!}{
\begin{tabular}{c|c|ccc|ccc|ccc|ccc}
\hline
\multirow{2}{*}{\textbf{Dataset}} & \multirow{2}{*}{\textbf{Methods}} & \multicolumn{3}{c|}{\citet{elasticsga}} & \multicolumn{3}{c|}{ABL~\cite{abl}} & \multicolumn{3}{c|}{\citet{manaar}} & \multicolumn{3}{c}{\citet{halimi}} \\ \cline{3-14} 
 &  & \multicolumn{1}{c|}{$\uparrow$SSIM} & \multicolumn{1}{c|}{$\uparrow$PSNR} & $\downarrow$LPIPS & \multicolumn{1}{c|}{$\uparrow$SSIM} & \multicolumn{1}{c|}{$\uparrow$PSNR} & $\downarrow$LPIPS & \multicolumn{1}{c|}{$\uparrow$SSIM} & \multicolumn{1}{c|}{$\uparrow$PSNR} & $\downarrow$LPIPS & \multicolumn{1}{c|}{$\uparrow$SSIM} & \multicolumn{1}{c|}{$\uparrow$PSNR} & $\downarrow$LPIPS \\ \hline \hline
\multirow{2}{*}{\textbf{CIFAR10}} & \citet{DBLP:conf/sp/HuWDX24} & \multicolumn{1}{c|}{0.8698} & \multicolumn{1}{c|}{27.17} & 0.1350 & \multicolumn{1}{c|}{0.0038} & \multicolumn{1}{c|}{5.48} & 0.9358 & \multicolumn{1}{c|}{-0.0062} & \multicolumn{1}{c|}{5.43} & 0.9201 & \multicolumn{1}{c|}{0.8488} & \multicolumn{1}{c|}{27.04} & 0.1523 \\ \cline{2-14} 
 & \textit{\methodname} & \multicolumn{1}{c|}{0.8503} & \multicolumn{1}{c|}{26.06} & 0.1509 & \multicolumn{1}{c|}{0.6407} & \multicolumn{1}{c|}{21.11} & 0.3175 & \multicolumn{1}{c|}{0.6114} & \multicolumn{1}{c|}{20.85} & 0.3328 & \multicolumn{1}{c|}{0.8374} & \multicolumn{1}{c|}{25.49} & 0.1669 \\ \hline
\multirow{2}{*}{\textbf{CIFAR100}} & \citet{DBLP:conf/sp/HuWDX24} & \multicolumn{1}{c|}{0.8786} & \multicolumn{1}{c|}{25.68} & 0.1294 & \multicolumn{1}{c|}{0.0008} & \multicolumn{1}{c|}{4.67} & 0.9007 & \multicolumn{1}{c|}{-0.0066} & \multicolumn{1}{c|}{4.87} & 0.8966 & \multicolumn{1}{c|}{0.7883} & \multicolumn{1}{c|}{23.42} & 0.2122 \\ \cline{2-14} 
 & \textit{\methodname} & \multicolumn{1}{c|}{0.8775} & \multicolumn{1}{c|}{24.98} & 0.1374 & \multicolumn{1}{c|}{0.6478} & \multicolumn{1}{c|}{19.32} & 0.3357 & \multicolumn{1}{c|}{0.7341} & \multicolumn{1}{c|}{21.86} & 0.2704 & \multicolumn{1}{c|}{0.8795} & \multicolumn{1}{c|}{26.26} & 0.1170 \\ \hline
\multirow{2}{*}{\textbf{MNIST}} & \citet{DBLP:conf/sp/HuWDX24} & \multicolumn{1}{c|}{0.9402} & \multicolumn{1}{c|}{37.93} & 0.0041 & \multicolumn{1}{c|}{-0.0148} & \multicolumn{1}{c|}{2.04} & 0.8080 & \multicolumn{1}{c|}{-0.0097} & \multicolumn{1}{c|}{2.19} & 0.8209 & \multicolumn{1}{c|}{0.9425} & \multicolumn{1}{c|}{38.94} & 0.0029 \\ \cline{2-14} 
 & \textit{\methodname} & \multicolumn{1}{c|}{0.9247} & \multicolumn{1}{c|}{36.39} & 0.0075 & \multicolumn{1}{c|}{0.7603} & \multicolumn{1}{c|}{28.46} & 0.1400 & \multicolumn{1}{c|}{0.7626} & \multicolumn{1}{c|}{29.01} & 0.1157 & \multicolumn{1}{c|}{0.9283} & \multicolumn{1}{c|}{36.67} & 0.0056 \\ \hline
\multirow{2}{*}{\textbf{FEMNIST}} & \citet{DBLP:conf/sp/HuWDX24} & \multicolumn{1}{c|}{0.9986} & \multicolumn{1}{c|}{53.57} & 0.0001 & \multicolumn{1}{c|}{-0.0000} & \multicolumn{1}{c|}{1.77} & 0.8880 & \multicolumn{1}{c|}{0.0016} & \multicolumn{1}{c|}{1.39} & 0.8991 & \multicolumn{1}{c|}{0.9620} & \multicolumn{1}{c|}{45.36} & 0.0070 \\ \cline{2-14} 
 & \textit{\methodname} & \multicolumn{1}{c|}{0.9978} & \multicolumn{1}{c|}{52.25} & 0.0003 & \multicolumn{1}{c|}{0.9929} & \multicolumn{1}{c|}{44.27} & 0.0007 & \multicolumn{1}{c|}{0.9947} & \multicolumn{1}{c|}{45.44} & 0.0004 & \multicolumn{1}{c|}{0.9970} & \multicolumn{1}{c|}{51.15} & 0.0003 \\ \hline
\end{tabular}}
\vspace{-0.4cm}
\end{table}

\textbf{Reconstruction Convergence Analysis:}
Figure~\ref{fig:loss_comparison} shows how the reconstruction loss changes over reconstruction steps for \textit{\methodname}~and~\citet{DBLP:conf/sp/HuWDX24}, evaluated with four unlearning algorithms on the CIFAR10 dataset. \textit{\methodname}~consistently achieves low reconstruction loss across all algorithms, showing that it can efficiently recover the data. In contrast,~\citet{DBLP:conf/sp/HuWDX24} fails to converge for ABL~\cite{abl} and~\citet{manaar}, getting stuck in high-loss regions. This suggests that \textit{\methodname}~can effectively approximate the client's unlearning algorithm (see $\mathcal{A}_{\text{approx}}$ in Algorithm~\ref{alg:approx}), regardless of which specific method is used, and produce accurate reconstructions.
\begin{figure}[!t]
    \centering
    \begin{minipage}[t]{0.48\textwidth}
        \centering
        \includegraphics[width=\linewidth]{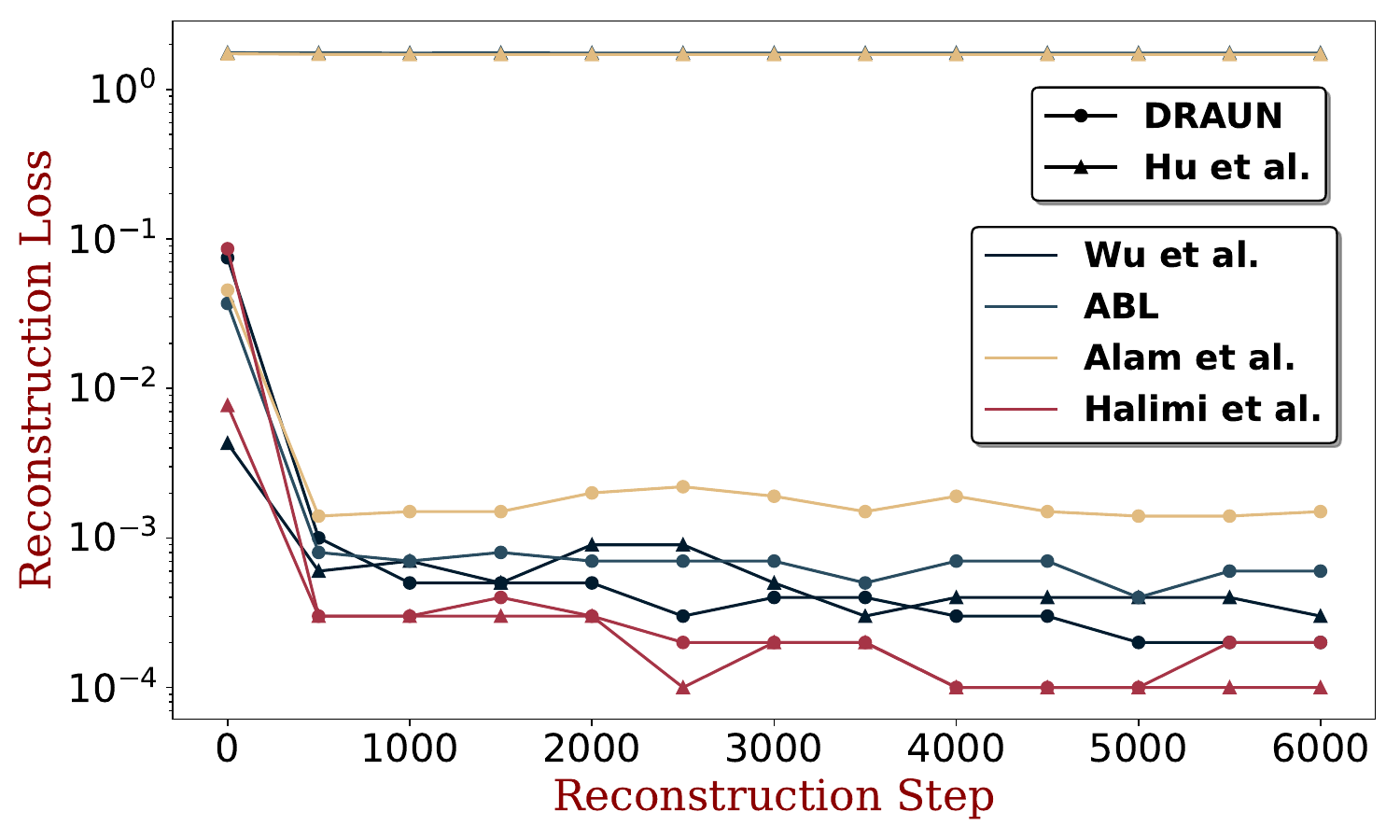}\vspace{-0.2cm}
        \caption{Reconstruction loss over reconstruction steps for \textit{\methodname} and~\citet{DBLP:conf/sp/HuWDX24} across four unlearning algorithms on CIFAR10.}\vspace{-0.2cm}
        \label{fig:loss_comparison}
    \end{minipage}%
    \hfill
    \begin{minipage}[t]{0.48\textwidth}
        \centering
        \includegraphics[width=0.75\linewidth]{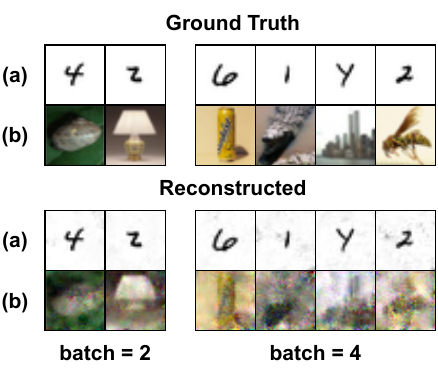}\vspace{-0.2cm}
        \caption{Reconstructed image batches using \textit{\methodname} from ABL~\cite{abl} with batch sizes of 2 and 4 for: \textbf{(a)} FEMNIST and \textbf{(b)} CIFAR100.}\vspace{-0.2cm}
        \label{fig:rec_grid}
    \end{minipage}\vspace{-0.2cm}
\end{figure}

\textbf{Batch Image Reconstruction:}
Batch reconstruction is a well-known challenge in DRA literature~\cite{DBLP:conf/nips/GeipingBD020}. When gradients come from multiple images, their signals mix, making it difficult to separate and recover individual examples. To explore the limits of \textit{\methodname}, we evaluate its performance in batch unlearning scenarios. Previously, all experiments focused on unlearning and reconstructing a single image. In Figure~\ref{fig:rec_grid}, we show how \textit{\methodname} performs on the FEMNIST and CIFAR100 datasets when the client unlearns a randomly selected batch of images (with batch sizes of 2 and 4). ABL~\cite{abl} is used to perform the unlearning, and \textit{\methodname} attempts to reconstruct every image in the batch. On FEMNIST, which consists of grayscale images with low visual complexity, \textit{\methodname} reconstructs images with high accuracy, even for batch sizes larger than one. On CIFAR100, which contains more complex and diverse images, reconstructions become increasingly noisy as the batch size grows from 2 to 4. Still, the recovered images preserve some structure and visual features from the ground truth. \citet{DBLP:conf/sp/HuWDX24}, which fail to produce meaningful reconstructions even in the single-image case, are not applicable in this setting. Since they cannot recover a single image, they do not generalize to the more difficult problem of reconstructing multiple images from mixed gradients. Additional results of batch reconstructions for other datasets and batch sizes are provided in Appendix~\ref{appendix:grid_rec}.\vspace{-0.1cm}

\section{Discussion}\vspace{-0.2cm}
\label{section:discussion}
\textbf{\textit{\methodname}'s Effectiveness against Second-Order Unlearning:} Optimization-based FU typically relies on first-order methods, as second-order approaches involve costly Hessian inversions. Nonetheless,~\citet{2ndord_fu} has recently explored second-order unlearning in FU. We evaluate \textit{\methodname}'s ability to reconstruct data from such second-order unlearning updates. Reconstruction from second-order unlearning follows the same overall approach as in the first-order case, but instead of matching gradients, the server simulates Newton updates by computing Hessian-vector products (HVPs) using dummy inputs and minimizes the differences between the true and simulated HVPs to recover the original client data. A detailed discussion on this second-order reconstruction mechanism is provided in Appendix~\ref{appendix:second_order}. Figure~\ref{fig:discussion}(a) shows two random examples from the CIFAR10 dataset reconstructed using \textit{\methodname} from second-order unlearned updates, demonstrating strong visual reconstruction. The model used in this experiment is an MLP, consistent with the setup used by~\citet{2ndord_fu}.

\begin{figure}[!t]
    \centering
    \includegraphics[width=0.95\linewidth]{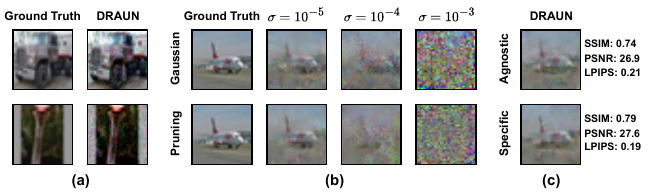}\vspace{-0.3cm}
    \caption{\textbf{(a)} Examples of CIFAR10 images reconstructed by \textit{\methodname} from second-order unlearning updates~\cite{2ndord_fu}. \textbf{(b)} Effects of two defense strategies: Gaussian noise addition (top row) and gradient pruning (bottom row) at varying noise levels and pruning thresholds $\sigma$. \textbf{(c)} Comparison of \textit{\methodname}'s reconstruction performance under algorithm-agnostic and algorithm-specific settings.}\vspace{-0.4cm}
    \label{fig:discussion}
\end{figure}

\textbf{Potential Defense against \textit{\methodname}:} To evaluate defenses against \textit{\methodname}, we consider two techniques: (1) adding Gaussian noise to unlearning updates before sending to the server~\cite{DBLP:conf/nips/ZhuLH19}, and (2) threshold-based pruning, which discards gradient components with magnitudes below a fixed threshold~\cite{grad_prune}. We use the unlearning algorithm by~\citet{manaar} on the CIFAR10 dataset. Figure~\ref{fig:discussion}(b) shows how increasing the noise level or pruning threshold $\sigma$ impacts reconstruction quality. As $\sigma$ increases from $10^{-5}$ to $10^{-3}$, the reconstructions degrade significantly, indicating that both defenses reduce \textit{\methodname}'s effectiveness. However, these defenses also degrade the utility of the model. A higher noise level or aggressive pruning negatively impacts its accuracy. This introduces a trade-off between privacy and performance, which we discuss in more detail in Appendix~\ref{appendix:defense}.

\textbf{\textit{\methodname}'s Performance in Algorithm-Specific Setting:} We also consider a relaxed threat model where the server knows the client's unlearning algorithm, referred to as the \textbf{algorithm-specific} setting. Here, the server simulates the exact unlearning updates, removing the need for surrogate optimization and providing an upper bound on \textit{\methodname}'s reconstruction performance. Figure~\ref{fig:discussion}(c) shows reconstructions from CIFAR10 in both algorithm-agnostic and algorithm-specific modes using the unlearning method of~\citet{manaar}. The algorithm-specific mode produces a significantly better reconstruction of the unlearned image, closely resembling the ground truth. This confirms that knowing the exact unlearning dynamics allows the attacker to generate more accurate gradient surrogates, leading to improved data reconstruction. A detailed comparison between algorithm-specific and algorithm-agnostic modes of \textit{\methodname} is provided in Appendix~\ref{appendix:algo_spec}.

\textbf{Limitations and Future Work:} \textbf{(1)} \textit{\methodname} is specifically designed for optimization-based federated unlearning algorithms. However, there also exist non-optimization-based unlearning techniques~\cite{fu_survey} that fall outside the current scope of our attack. In future work, we plan to extend \textit{\methodname} to support these alternative classes of unlearning algorithms. \textbf{(2)} \textit{\methodname} relies on GIA as its backbone and therefore inherits its limitations, including reduced effectiveness on batch data reconstruction. We plan to develop a version of \textit{\methodname} that does not depend on GIA to address these issues.\vspace{-0.2cm}

\section{Conclusion}\vspace{-0.2cm}
We present \textit{\methodname}, the first DRA specifically designed for FU systems. Existing DRAs for unlearning were developed in centralized MLaaS settings and do not directly translate to federated learning. \textit{\methodname}~addresses the unique challenges of FU and is capable of reconstructing deleted data from local unlearning updates without requiring knowledge of the client's unlearning method. Our theoretical and empirical analysis shows that optimization-based unlearning methods leak exploitable signals, which \textit{\methodname}~effectively uses to recover the deleted data. We demonstrate \textit{\methodname}'s effectiveness across a range of datasets, model architectures, and unlearning algorithms.

\bibliographystyle{plainnat}
\bibliography{draun_preprint}

\begin{thebibliography}{64}
\providecommand{\natexlab}[1]{#1}
\providecommand{\url}[1]{\texttt{#1}}
\expandafter\ifx\csname urlstyle\endcsname\relax
  \providecommand{\doi}[1]{doi: #1}\else
  \providecommand{\doi}{doi: \begingroup \urlstyle{rm}\Url}\fi

\bibitem[Alam et~al.(2024)Alam, Lamri, and Maniatakos]{manaar}
Manaar Alam, Hithem Lamri, and Michail Maniatakos.
\newblock Get rid of your trail: Remotely erasing backdoors in federated learning.
\newblock \emph{{IEEE} Trans. Artif. Intell.}, 5\penalty0 (12):\penalty0 6683--6698, 2024.
\newblock \doi{10.1109/TAI.2024.3465441}.
\newblock URL \url{https://doi.org/10.1109/TAI.2024.3465441}.

\bibitem[Balle et~al.(2022)Balle, Cherubin, and Hayes]{balle_nnreconstructor}
Borja Balle, Giovanni Cherubin, and Jamie Hayes.
\newblock Reconstructing training data with informed adversaries.
\newblock In \emph{43rd {IEEE} Symposium on Security and Privacy, {SP} 2022, San Francisco, CA, USA, May 22-26, 2022}, pages 1138--1156. {IEEE}, 2022.
\newblock \doi{10.1109/SP46214.2022.9833677}.
\newblock URL \url{https://doi.org/10.1109/SP46214.2022.9833677}.

\bibitem[Balunovic et~al.(2022)Balunovic, Dimitrov, Staab, and Vechev]{bayes_eth}
Mislav Balunovic, Dimitar~Iliev Dimitrov, Robin Staab, and Martin~T. Vechev.
\newblock Bayesian framework for gradient leakage.
\newblock In \emph{The Tenth International Conference on Learning Representations, {ICLR} 2022, Virtual Event, April 25-29, 2022}. OpenReview.net, 2022.
\newblock URL \url{https://openreview.net/forum?id=f2lrIbGx3x7}.

\bibitem[Bertran et~al.(2024)Bertran, Tang, Kearns, Morgenstern, Roth, and Wu]{DBLP:conf/nips/BertranTKM0W24}
Mart{\'{\i}}n Bertran, Shuai Tang, Michael Kearns, Jamie~H. Morgenstern, Aaron Roth, and Steven~Z. Wu.
\newblock Reconstruction attacks on machine unlearning: Simple models are vulnerable.
\newblock In \emph{Advances in Neural Information Processing Systems 38: Annual Conference on Neural Information Processing Systems 2024, NeurIPS 2024, Vancouver, BC, Canada, December 10 - 15, 2024}, 2024.
\newblock URL \url{http://papers.nips.cc/paper\_files/paper/2024/hash/bd996108ed57d388866ca6deb7acf6cb-Abstract-Conference.html}.

\bibitem[Boenisch et~al.(2023)Boenisch, Dziedzic, Schuster, Shamsabadi, Shumailov, and Papernot]{franz}
Franziska Boenisch, Adam Dziedzic, Roei Schuster, Ali~Shahin Shamsabadi, Ilia Shumailov, and Nicolas Papernot.
\newblock When the curious abandon honesty: Federated learning is not private.
\newblock In \emph{8th {IEEE} European Symposium on Security and Privacy, EuroS{\&}P 2023, Delft, Netherlands, July 3-7, 2023}, pages 175--199. {IEEE}, 2023.
\newblock \doi{10.1109/EUROSP57164.2023.00020}.
\newblock URL \url{https://doi.org/10.1109/EuroSP57164.2023.00020}.

\bibitem[Caldas et~al.(2018)Caldas, Wu, Li, Konecny, McMahan, Smith, and Talwalkar]{caldas2018leaf}
Sebastian Caldas, Peter Wu, Tian Li, Jakub Konecny, H~Brendan McMahan, Virginia Smith, and Ameet Talwalkar.
\newblock Leaf: A benchmark for federated settings.
\newblock In \emph{Workshop on Federated Learning for Data Privacy and Confidentiality, NeurIPS}, 2018.
\newblock URL \url{https://arxiv.org/abs/1812.01097}.

\bibitem[Cao and Yang(2015)]{firstunlearning}
Yinzhi Cao and Junfeng Yang.
\newblock {Towards Making Systems Forget with Machine Unlearning}.
\newblock In \emph{2015 {IEEE} Symposium on Security and Privacy, {SP} 2015, San Jose, CA, USA, May 17-21, 2015}, pages 463--480. {IEEE} Computer Society, 2015.
\newblock \doi{10.1109/SP.2015.35}.
\newblock URL \url{https://doi.org/10.1109/SP.2015.35}.

\bibitem[Dimitrov et~al.(2024)Dimitrov, Baader, M{\"{u}}ller, and Vechev]{spear}
Dimitar~I. Dimitrov, Maximilian Baader, Mark~Niklas M{\"{u}}ller, and Martin~T. Vechev.
\newblock {SPEAR:} exact gradient inversion of batches in federated learning.
\newblock \emph{CoRR}, abs/2403.03945, 2024.
\newblock \doi{10.48550/ARXIV.2403.03945}.
\newblock URL \url{https://doi.org/10.48550/arXiv.2403.03945}.

\bibitem[Dimitrov et~al.(2022)Dimitrov, Balunovic, Konstantinov, and Vechev]{fedavg_leak}
Dimitar~Iliev Dimitrov, Mislav Balunovic, Nikola Konstantinov, and Martin~T. Vechev.
\newblock Data leakage in federated averaging.
\newblock \emph{Trans. Mach. Learn. Res.}, 2022, 2022.
\newblock URL \url{https://openreview.net/forum?id=e7A0B99zJf}.

\bibitem[{European Commission}(2024)]{gdpr}
{European Commission}.
\newblock {Data protection in the EU}, 2024.
\newblock \url{https://commission.europa.eu/law/law-topic/data-protection/data-protection-eu_en}.

\bibitem[Fowl et~al.(2022)Fowl, Geiping, Czaja, Goldblum, and Goldstein]{model_modif1}
Liam~H. Fowl, Jonas Geiping, Wojciech Czaja, Micah Goldblum, and Tom Goldstein.
\newblock Robbing the fed: Directly obtaining private data in federated learning with modified models.
\newblock In \emph{The Tenth International Conference on Learning Representations, {ICLR} 2022, Virtual Event, April 25-29, 2022}. OpenReview.net, 2022.
\newblock URL \url{https://openreview.net/forum?id=fwzUgo0FM9v}.

\bibitem[Gao et~al.(2021)Gao, Guo, Zhang, Qiu, Wen, and Liu]{ats}
Wei Gao, Shangwei Guo, Tianwei Zhang, Han Qiu, Yonggang Wen, and Yang Liu.
\newblock Privacy-preserving collaborative learning with automatic transformation search.
\newblock In \emph{{IEEE} Conference on Computer Vision and Pattern Recognition, {CVPR} 2021, virtual, June 19-25, 2021}, pages 114--123. Computer Vision Foundation / {IEEE}, 2021.
\newblock \doi{10.1109/CVPR46437.2021.00018}.
\newblock URL \url{https://openaccess.thecvf.com/content/CVPR2021/html/Gao\_Privacy-Preserving\_Collaborative\_Learning\_With\_Automatic\_Transformation\_Search\_CVPR\_2021\_paper.html}.

\bibitem[Garov et~al.(2024)Garov, Dimitrov, Jovanovic, and Vechev]{seer}
Kostadin Garov, Dimitar~Iliev Dimitrov, Nikola Jovanovic, and Martin~T. Vechev.
\newblock Hiding in plain sight: Disguising data stealing attacks in federated learning.
\newblock In \emph{The Twelfth International Conference on Learning Representations, {ICLR} 2024, Vienna, Austria, May 7-11, 2024}. OpenReview.net, 2024.
\newblock URL \url{https://openreview.net/forum?id=krx55l2A6G}.

\bibitem[Geiping et~al.(2020{\natexlab{a}})Geiping, Bauermeister, Dr{\"{o}}ge, and Moeller]{DBLP:conf/nips/GeipingBD020}
Jonas Geiping, Hartmut Bauermeister, Hannah Dr{\"{o}}ge, and Michael Moeller.
\newblock Inverting gradients - how easy is it to break privacy in federated learning?
\newblock In \emph{Advances in Neural Information Processing Systems 33: Annual Conference on Neural Information Processing Systems 2020, NeurIPS 2020, December 6-12, 2020, virtual}, 2020{\natexlab{a}}.
\newblock URL \url{https://proceedings.neurips.cc/paper/2020/hash/c4ede56bbd98819ae6112b20ac6bf145-Abstract.html}.

\bibitem[Geiping et~al.(2020{\natexlab{b}})Geiping, Bauermeister, Dr{\"{o}}ge, and Moeller]{invgrad}
Jonas Geiping, Hartmut Bauermeister, Hannah Dr{\"{o}}ge, and Michael Moeller.
\newblock Inverting gradients - how easy is it to break privacy in federated learning?
\newblock In Hugo Larochelle, Marc'Aurelio Ranzato, Raia Hadsell, Maria{-}Florina Balcan, and Hsuan{-}Tien Lin, editors, \emph{Advances in Neural Information Processing Systems 33: Annual Conference on Neural Information Processing Systems 2020, NeurIPS 2020, December 6-12, 2020, virtual}, 2020{\natexlab{b}}.
\newblock URL \url{https://proceedings.neurips.cc/paper/2020/hash/c4ede56bbd98819ae6112b20ac6bf145-Abstract.html}.

\bibitem[Geng et~al.(2021)Geng, Mou, Li, Li, Beyan, Decker, and Rong]{gendlg}
Jiahui Geng, Yongli Mou, Feifei Li, Qing Li, Oya Beyan, Stefan Decker, and Chunming Rong.
\newblock Towards general deep leakage in federated learning.
\newblock \emph{CoRR}, abs/2110.09074, 2021.
\newblock URL \url{https://arxiv.org/abs/2110.09074}.

\bibitem[Goodfellow et~al.(2016)Goodfellow, Bengio, and Courville]{Goodfellow-et-al-2016}
Ian Goodfellow, Yoshua Bengio, and Aaron Courville.
\newblock \emph{Deep Learning}.
\newblock MIT Press, 2016.
\newblock \url{http://www.deeplearningbook.org}.

\bibitem[Goodfellow et~al.(2014)Goodfellow, Pouget{-}Abadie, Mirza, Xu, Warde{-}Farley, Ozair, Courville, and Bengio]{Goodfellowgans}
Ian~J. Goodfellow, Jean Pouget{-}Abadie, Mehdi Mirza, Bing Xu, David Warde{-}Farley, Sherjil Ozair, Aaron~C. Courville, and Yoshua Bengio.
\newblock Generative adversarial networks.
\newblock \emph{CoRR}, abs/1406.2661, 2014.
\newblock URL \url{http://arxiv.org/abs/1406.2661}.

\bibitem[Haim et~al.(2022)Haim, Vardi, Yehudai, Shamir, and Irani]{ntk_attack}
Niv Haim, Gal Vardi, Gilad Yehudai, Ohad Shamir, and Michal Irani.
\newblock Reconstructing training data from trained neural networks.
\newblock In Sanmi Koyejo, S.~Mohamed, A.~Agarwal, Danielle Belgrave, K.~Cho, and A.~Oh, editors, \emph{Advances in Neural Information Processing Systems 35: Annual Conference on Neural Information Processing Systems 2022, NeurIPS 2022, New Orleans, LA, USA, November 28 - December 9, 2022}, 2022.
\newblock URL \url{http://papers.nips.cc/paper\_files/paper/2022/hash/906927370cbeb537781100623cca6fa6-Abstract-Conference.html}.

\bibitem[Halimi et~al.(2022)Halimi, Kadhe, Rawat, and Baracaldo]{halimi}
Anisa Halimi, Swanand Kadhe, Ambrish Rawat, and Nathalie Baracaldo.
\newblock Federated unlearning: How to efficiently erase a client in fl?
\newblock In \emph{Updatable Machine Learning (part of ICML 2022), UpML 2022, Baltimore, USA, July 23, 2022}, 2022.
\newblock \doi{10.48550/ARXIV.2207.05521}.
\newblock URL \url{https://doi.org/10.48550/arXiv.2207.05521}.

\bibitem[He et~al.(2016)He, Zhang, Ren, and Sun]{resnet}
Kaiming He, Xiangyu Zhang, Shaoqing Ren, and Jian Sun.
\newblock Deep residual learning for image recognition.
\newblock In \emph{2016 {IEEE} Conference on Computer Vision and Pattern Recognition, {CVPR} 2016, Las Vegas, NV, USA, June 27-30, 2016}, pages 770--778. {IEEE} Computer Society, 2016.
\newblock \doi{10.1109/CVPR.2016.90}.
\newblock URL \url{https://doi.org/10.1109/CVPR.2016.90}.

\bibitem[He et~al.(2019)He, Zhang, and Lee]{nn_reconst2}
Zecheng He, Tianwei Zhang, and Ruby~B. Lee.
\newblock Model inversion attacks against collaborative inference.
\newblock In David~M. Balenson, editor, \emph{Proceedings of the 35th Annual Computer Security Applications Conference, {ACSAC} 2019, San Juan, PR, USA, December 09-13, 2019}, pages 148--162. {ACM}, 2019.
\newblock \doi{10.1145/3359789.3359824}.
\newblock URL \url{https://doi.org/10.1145/3359789.3359824}.

\bibitem[Hitaj et~al.(2017)Hitaj, Ateniese, and P{\'{e}}rez{-}Cruz]{hitaj_dragan}
Briland Hitaj, Giuseppe Ateniese, and Fernando P{\'{e}}rez{-}Cruz.
\newblock Deep models under the {GAN:} information leakage from collaborative deep learning.
\newblock In Bhavani Thuraisingham, David Evans, Tal Malkin, and Dongyan Xu, editors, \emph{Proceedings of the 2017 {ACM} {SIGSAC} Conference on Computer and Communications Security, {CCS} 2017, Dallas, TX, USA, October 30 - November 03, 2017}, pages 603--618. {ACM}, 2017.
\newblock \doi{10.1145/3133956.3134012}.
\newblock URL \url{https://doi.org/10.1145/3133956.3134012}.

\bibitem[Hu et~al.(2024)Hu, Wang, Dong, and Xue]{DBLP:conf/sp/HuWDX24}
Hongsheng Hu, Shuo Wang, Tian Dong, and Minhui Xue.
\newblock Learn what you want to unlearn: Unlearning inversion attacks against machine unlearning.
\newblock In \emph{{IEEE} Symposium on Security and Privacy, {SP} 2024, San Francisco, CA, USA, May 19-23, 2024}, pages 3257--3275. {IEEE}, 2024.
\newblock \doi{10.1109/SP54263.2024.00248}.
\newblock URL \url{https://doi.org/10.1109/SP54263.2024.00248}.

\bibitem[Huang et~al.(2021)Huang, Gupta, Song, Li, and Arora]{eval_dra}
Yangsibo Huang, Samyak Gupta, Zhao Song, Kai Li, and Sanjeev Arora.
\newblock Evaluating gradient inversion attacks and defenses in federated learning.
\newblock In Marc'Aurelio Ranzato, Alina Beygelzimer, Yann~N. Dauphin, Percy Liang, and Jennifer~Wortman Vaughan, editors, \emph{Advances in Neural Information Processing Systems 34: Annual Conference on Neural Information Processing Systems 2021, NeurIPS 2021, December 6-14, 2021, virtual}, pages 7232--7241, 2021.
\newblock URL \url{https://proceedings.neurips.cc/paper/2021/hash/3b3fff6463464959dcd1b68d0320f781-Abstract.html}.

\bibitem[Jacot et~al.(2018)Jacot, Hongler, and Gabriel]{ntk}
Arthur Jacot, Cl{\'{e}}ment Hongler, and Franck Gabriel.
\newblock Neural tangent kernel: Convergence and generalization in neural networks.
\newblock In Samy Bengio, Hanna~M. Wallach, Hugo Larochelle, Kristen Grauman, Nicol{\`{o}} Cesa{-}Bianchi, and Roman Garnett, editors, \emph{Advances in Neural Information Processing Systems 31: Annual Conference on Neural Information Processing Systems 2018, NeurIPS 2018, December 3-8, 2018, Montr{\'{e}}al, Canada}, pages 8580--8589, 2018.
\newblock URL \url{https://proceedings.neurips.cc/paper/2018/hash/5a4be1fa34e62bb8a6ec6b91d2462f5a-Abstract.html}.

\bibitem[Jeon et~al.(2021)Jeon, Kim, Lee, Oh, and Ok]{dragan1}
Jinwoo Jeon, Jaechang Kim, Kangwook Lee, Sewoong Oh, and Jungseul Ok.
\newblock Gradient inversion with generative image prior.
\newblock In Marc'Aurelio Ranzato, Alina Beygelzimer, Yann~N. Dauphin, Percy Liang, and Jennifer~Wortman Vaughan, editors, \emph{Advances in Neural Information Processing Systems 34: Annual Conference on Neural Information Processing Systems 2021, NeurIPS 2021, December 6-14, 2021, virtual}, pages 29898--29908, 2021.
\newblock URL \url{https://proceedings.neurips.cc/paper/2021/hash/fa84632d742f2729dc32ce8cb5d49733-Abstract.html}.

\bibitem[Ji and Telgarsky(2020)]{NEURIPS2020_c76e4b2f}
Ziwei Ji and Matus Telgarsky.
\newblock Directional convergence and alignment in deep learning.
\newblock In H.~Larochelle, M.~Ranzato, R.~Hadsell, M.F. Balcan, and H.~Lin, editors, \emph{Advances in Neural Information Processing Systems}, volume~33, pages 17176--17186. Curran Associates, Inc., 2020.
\newblock URL \url{https://proceedings.neurips.cc/paper_files/paper/2020/file/c76e4b2fa54f8506719a5c0dc14c2eb9-Paper.pdf}.

\bibitem[Jin et~al.(2024)Jin, Chen, Zhang, and Li]{2ndord_fu}
Ruinan Jin, Minghui Chen, Qiong Zhang, and Xiaoxiao Li.
\newblock Forgettable federated linear learning with certified data unlearning, 2024.
\newblock URL \url{https://arxiv.org/abs/2306.02216}.

\bibitem[Krizhevsky(2009)]{cifar10_100}
Alex Krizhevsky.
\newblock Learning multiple layers of features from tiny images.
\newblock Technical report, University of Toronto, 2009.
\newblock URL \url{https://www.cs.toronto.edu/~kriz/learning-features-2009-TR.pdf}.

\bibitem[Kuhn and Tucker(1951)]{kkt}
H.~W. Kuhn and A.~W. Tucker.
\newblock Nonlinear programming.
\newblock In \emph{Proceedings of the {S}econd {B}erkeley {S}ymposium on {M}athematical {S}tatistics and {P}robability, 1950}, pages 481--492, Berkeley and Los Angeles, 1951. University of California Press.

\bibitem[LeCun et~al.(1998)LeCun, Bottou, Bengio, and Haffner]{lecun1998mnist}
Yann LeCun, L{\'{e}}on Bottou, Yoshua Bengio, and Patrick Haffner.
\newblock Gradient-based learning applied to document recognition.
\newblock \emph{Proc. {IEEE}}, 86\penalty0 (11):\penalty0 2278--2324, 1998.
\newblock \doi{10.1109/5.726791}.
\newblock URL \url{https://doi.org/10.1109/5.726791}.

\bibitem[Li et~al.(2021)Li, Lyu, Koren, Lyu, Li, and Ma]{abl}
Yige Li, Xixiang Lyu, Nodens Koren, Lingjuan Lyu, Bo~Li, and Xingjun Ma.
\newblock Anti-backdoor learning: Training clean models on poisoned data.
\newblock In \emph{Advances in Neural Information Processing Systems 34: Annual Conference on Neural Information Processing Systems 2021, NeurIPS 2021, December 6-14, 2021, virtual}, pages 14900--14912, 2021.
\newblock URL \url{https://proceedings.neurips.cc/paper/2021/hash/7d38b1e9bd793d3f45e0e212a729a93c-Abstract.html}.

\bibitem[Lin et~al.(2018)Lin, Han, Mao, Wang, and Dally]{grad_prune}
Yujun Lin, Song Han, Huizi Mao, Yu~Wang, and Bill Dally.
\newblock Deep gradient compression: Reducing the communication bandwidth for distributed training.
\newblock In \emph{6th International Conference on Learning Representations, {ICLR} 2018, Vancouver, BC, Canada, April 30 - May 3, 2018, Conference Track Proceedings}. OpenReview.net, 2018.
\newblock URL \url{https://openreview.net/forum?id=SkhQHMW0W}.

\bibitem[Liu et~al.(2021)Liu, Ma, Yang, Wang, and Liu]{federaser_fu3}
Gaoyang Liu, Xiaoqiang Ma, Yang Yang, Chen Wang, and Jiangchuan Liu.
\newblock Federaser: Enabling efficient client-level data removal from federated learning models.
\newblock In \emph{2021 IEEE/ACM 29th International Symposium on Quality of Service (IWQOS)}, pages 1--10, 2021.
\newblock \doi{10.1109/IWQOS52092.2021.9521274}.

\bibitem[Liu et~al.(2025)Liu, Jiang, Shen, Peng, Lam, Yuan, and Liu]{fu_survey}
Ziyao Liu, Yu~Jiang, Jiyuan Shen, Minyi Peng, Kwok{-}Yan Lam, Xingliang Yuan, and Xiaoning Liu.
\newblock A survey on federated unlearning: Challenges, methods, and future directions.
\newblock \emph{{ACM} Comput. Surv.}, 57\penalty0 (1):\penalty0 2:1--2:38, 2025.
\newblock \doi{10.1145/3679014}.
\newblock URL \url{https://doi.org/10.1145/3679014}.

\bibitem[Lyu and Li(2020)]{Lyu2020Gradient}
Kaifeng Lyu and Jian Li.
\newblock Gradient descent maximizes the margin of homogeneous neural networks.
\newblock In \emph{International Conference on Learning Representations}, 2020.
\newblock URL \url{https://openreview.net/forum?id=SJeLIgBKPS}.

\bibitem[McMahan et~al.(2017)McMahan, Moore, Ramage, Hampson, and y~Arcas]{DBLP:conf/aistats/McMahanMRHA17}
Brendan McMahan, Eider Moore, Daniel Ramage, Seth Hampson, and Blaise~Ag{\"{u}}era y~Arcas.
\newblock Communication-efficient learning of deep networks from decentralized data.
\newblock In \emph{Proceedings of the 20th International Conference on Artificial Intelligence and Statistics, {AISTATS} 2017, 20-22 April 2017, Fort Lauderdale, FL, {USA}}, volume~54 of \emph{Proceedings of Machine Learning Research}, pages 1273--1282. {PMLR}, 2017.
\newblock URL \url{http://proceedings.mlr.press/v54/mcmahan17a.html}.

\bibitem[Noorbakhsh et~al.(2024)Noorbakhsh, Zhang, Hong, and Wang]{infguard}
Sayedeh~Leila Noorbakhsh, Binghui Zhang, Yuan Hong, and Binghui Wang.
\newblock Inf2guard: An information-theoretic framework for learning privacy-preserving representations against inference attacks.
\newblock In Davide Balzarotti and Wenyuan Xu, editors, \emph{33rd {USENIX} Security Symposium, {USENIX} Security 2024, Philadelphia, PA, USA, August 14-16, 2024}. {USENIX} Association, 2024.
\newblock URL \url{https://www.usenix.org/conference/usenixsecurity24/presentation/noorbakhsh}.

\bibitem[{Office of the Attorney General, State of California}(2024)]{ccpa}
{Office of the Attorney General, State of California}.
\newblock {California Consumer Privacy Act (CCPA)}, 2024.
\newblock \url{https://oag.ca.gov/privacy/ccpa}.

\bibitem[Rudin et~al.(1992)Rudin, Osher, and Fatemi]{RUDIN1992259}
Leonid~I. Rudin, Stanley Osher, and Emad Fatemi.
\newblock Nonlinear total variation based noise removal algorithms.
\newblock \emph{Physica D: Nonlinear Phenomena}, 60\penalty0 (1):\penalty0 259--268, 1992.
\newblock ISSN 0167-2789.
\newblock \doi{https://doi.org/10.1016/0167-2789(92)90242-F}.
\newblock URL \url{https://www.sciencedirect.com/science/article/pii/016727899290242F}.

\bibitem[Salem et~al.(2020)Salem, Bhattacharya, Backes, Fritz, and Zhang]{salemmicrosoft}
Ahmed Salem, Apratim Bhattacharya, Michael Backes, Mario Fritz, and Yang Zhang.
\newblock Updates-leak: Data set inference and reconstruction attacks in online learning.
\newblock In Srdjan Capkun and Franziska Roesner, editors, \emph{29th {USENIX} Security Symposium, {USENIX} Security 2020, August 12-14, 2020}, pages 1291--1308. {USENIX} Association, 2020.
\newblock URL \url{https://www.usenix.org/conference/usenixsecurity20/presentation/salem}.

\bibitem[Scheliga et~al.(2022)Scheliga, M{\"{a}}der, and Seeland]{precode}
Daniel Scheliga, Patrick M{\"{a}}der, and Marco Seeland.
\newblock {PRECODE} - {A} generic model extension to prevent deep gradient leakage.
\newblock In \emph{{IEEE/CVF} Winter Conference on Applications of Computer Vision, {WACV} 2022, Waikoloa, HI, USA, January 3-8, 2022}, pages 3605--3614. {IEEE}, 2022.
\newblock \doi{10.1109/WACV51458.2022.00366}.
\newblock URL \url{https://doi.org/10.1109/WACV51458.2022.00366}.

\bibitem[Shaik et~al.(2024)Shaik, Tao, Li, Xie, Cai, Zhu, and Li]{DBLP:journals/tkde/ShaikTLXCZL24}
Thanveer Shaik, Xiaohui Tao, Lin Li, Haoran Xie, Taotao Cai, Xiaofeng Zhu, and Qing Li.
\newblock {FRAMU:} attention-based machine unlearning using federated reinforcement learning.
\newblock \emph{{IEEE} Trans. Knowl. Data Eng.}, 36\penalty0 (10):\penalty0 5153--5167, 2024.
\newblock \doi{10.1109/TKDE.2024.3382726}.
\newblock URL \url{https://doi.org/10.1109/TKDE.2024.3382726}.

\bibitem[Sun et~al.(2021)Sun, Li, Wang, Yang, Li, and Chen]{soteria}
Jingwei Sun, Ang Li, Binghui Wang, Huanrui Yang, Hai Li, and Yiran Chen.
\newblock Soteria: Provable defense against privacy leakage in federated learning from representation perspective.
\newblock In \emph{{IEEE} Conference on Computer Vision and Pattern Recognition, {CVPR} 2021, virtual, June 19-25, 2021}, pages 9311--9319. Computer Vision Foundation / {IEEE}, 2021.
\newblock \doi{10.1109/CVPR46437.2021.00919}.
\newblock URL \url{https://openaccess.thecvf.com/content/CVPR2021/html/Sun\_Soteria\_Provable\_Defense\_Against\_Privacy\_Leakage\_in\_Federated\_Learning\_From\_CVPR\_2021\_paper.html}.

\bibitem[Wang et~al.(2023{\natexlab{a}})Wang, Tian, Zhang, Liu, and Yu]{DBLP:conf/asiaccs/WangTZ0023}
Weiqi Wang, Zhiyi Tian, Chenhan Zhang, An~Liu, and Shui Yu.
\newblock {BFU:} bayesian federated unlearning with parameter self-sharing.
\newblock In \emph{Proceedings of the 2023 {ACM} Asia Conference on Computer and Communications Security, {ASIA} {CCS} 2023, Melbourne, VIC, Australia, July 10-14, 2023}, pages 567--578. {ACM}, 2023{\natexlab{a}}.
\newblock \doi{10.1145/3579856.3590327}.
\newblock URL \url{https://doi.org/10.1145/3579856.3590327}.

\bibitem[Wang et~al.(2019)Wang, Song, Zhang, Song, Wang, and Qi]{dragan2}
Zhibo Wang, Mengkai Song, Zhifei Zhang, Yang Song, Qian Wang, and Hairong Qi.
\newblock Beyond inferring class representatives: User-level privacy leakage from federated learning.
\newblock In \emph{2019 {IEEE} Conference on Computer Communications, {INFOCOM} 2019, Paris, France, April 29 - May 2, 2019}, pages 2512--2520. {IEEE}, 2019.
\newblock \doi{10.1109/INFOCOM.2019.8737416}.
\newblock URL \url{https://doi.org/10.1109/INFOCOM.2019.8737416}.

\bibitem[Wang et~al.(2004)Wang, Bovik, Sheikh, and Simoncelli]{wang2004image}
Zhou Wang, Alan~C. Bovik, Hamid~R. Sheikh, and Eero~P. Simoncelli.
\newblock Image quality assessment: from error visibility to structural similarity.
\newblock \emph{{IEEE} Trans. Image Process.}, 13\penalty0 (4):\penalty0 600--612, 2004.
\newblock \doi{10.1109/TIP.2003.819861}.
\newblock URL \url{https://doi.org/10.1109/TIP.2003.819861}.

\bibitem[Wang et~al.(2023{\natexlab{b}})Wang, Lee, and Lei]{WangLL23}
Zihan Wang, Jason Lee, and Qi~Lei.
\newblock Reconstructing training data from model gradient, provably.
\newblock In Francisco J.~R. Ruiz, Jennifer~G. Dy, and Jan{-}Willem van~de Meent, editors, \emph{International Conference on Artificial Intelligence and Statistics, 25-27 April 2023, Palau de Congressos, Valencia, Spain}, volume 206 of \emph{Proceedings of Machine Learning Research}, pages 6595--6612. {PMLR}, 2023{\natexlab{b}}.
\newblock URL \url{https://proceedings.mlr.press/v206/wang23g.html}.

\bibitem[Wen et~al.(2022{\natexlab{a}})Wen, Geiping, Fowl, Goldblum, and Goldstein]{DBLP:conf/icml/WenGFGG22}
Yuxin Wen, Jonas Geiping, Liam Fowl, Micah Goldblum, and Tom Goldstein.
\newblock Fishing for user data in large-batch federated learning via gradient magnification.
\newblock In \emph{International Conference on Machine Learning, {ICML} 2022, 17-23 July 2022, Baltimore, Maryland, {USA}}, volume 162 of \emph{Proceedings of Machine Learning Research}, pages 23668--23684. {PMLR}, 2022{\natexlab{a}}.
\newblock URL \url{https://proceedings.mlr.press/v162/wen22a.html}.

\bibitem[Wen et~al.(2022{\natexlab{b}})Wen, Geiping, Fowl, Goldblum, and Goldstein]{fishing_icml}
Yuxin Wen, Jonas~A. Geiping, Liam Fowl, Micah Goldblum, and Tom Goldstein.
\newblock Fishing for user data in large-batch federated learning via gradient magnification.
\newblock In Kamalika Chaudhuri, Stefanie Jegelka, Le~Song, Csaba Szepesvari, Gang Niu, and Sivan Sabato, editors, \emph{Proceedings of the 39th International Conference on Machine Learning}, volume 162 of \emph{Proceedings of Machine Learning Research}, pages 23668--23684. PMLR, 17--23 Jul 2022{\natexlab{b}}.
\newblock URL \url{https://proceedings.mlr.press/v162/wen22a.html}.

\bibitem[Wu et~al.(2022)Wu, Guo, Wang, Hong, Zhang, and Ding]{elasticsga}
Leijie Wu, Song Guo, Junxiao Wang, Zicong Hong, Jie Zhang, and Yaohong Ding.
\newblock Federated unlearning: Guarantee the right of clients to forget.
\newblock \emph{{IEEE} Netw.}, 36\penalty0 (5):\penalty0 129--135, 2022.
\newblock \doi{10.1109/MNET.001.2200198}.
\newblock URL \url{https://doi.org/10.1109/MNET.001.2200198}.

\bibitem[Yin et~al.(2021{\natexlab{a}})Yin, Mallya, Vahdat, {\'{A}}lvarez, Kautz, and Molchanov]{DBLP:conf/cvpr/YinMVAKM21}
Hongxu Yin, Arun Mallya, Arash Vahdat, Jos{\'{e}}~M. {\'{A}}lvarez, Jan Kautz, and Pavlo Molchanov.
\newblock See through gradients: Image batch recovery via gradinversion.
\newblock In \emph{{IEEE} Conference on Computer Vision and Pattern Recognition, {CVPR} 2021, virtual, June 19-25, 2021}, pages 16337--16346. Computer Vision Foundation / {IEEE}, 2021{\natexlab{a}}.
\newblock \doi{10.1109/CVPR46437.2021.01607}.
\newblock URL \url{https://openaccess.thecvf.com/content/CVPR2021/html/Yin\_See\_Through\_Gradients\_Image\_Batch\_Recovery\_via\_GradInversion\_CVPR\_2021\_paper.html}.

\bibitem[Yin et~al.(2021{\natexlab{b}})Yin, Mallya, Vahdat, {\'{A}}lvarez, Kautz, and Molchanov]{seethrgrd}
Hongxu Yin, Arun Mallya, Arash Vahdat, Jos{\'{e}}~M. {\'{A}}lvarez, Jan Kautz, and Pavlo Molchanov.
\newblock See through gradients: Image batch recovery via gradinversion.
\newblock In \emph{{IEEE} Conference on Computer Vision and Pattern Recognition, {CVPR} 2021, virtual, June 19-25, 2021}, pages 16337--16346. Computer Vision Foundation / {IEEE}, 2021{\natexlab{b}}.
\newblock \doi{10.1109/CVPR46437.2021.01607}.
\newblock URL \url{https://openaccess.thecvf.com/content/CVPR2021/html/Yin\_See\_Through\_Gradients\_Image\_Batch\_Recovery\_via\_GradInversion\_CVPR\_2021\_paper.html}.

\bibitem[Zhang et~al.(2018)Zhang, Isola, Efros, Shechtman, and Wang]{zhang2018unreasonable}
Richard Zhang, Phillip Isola, Alexei~A. Efros, Eli Shechtman, and Oliver Wang.
\newblock The unreasonable effectiveness of deep features as a perceptual metric.
\newblock In \emph{2018 {IEEE} Conference on Computer Vision and Pattern Recognition, {CVPR} 2018, Salt Lake City, UT, USA, June 18-22, 2018}, pages 586--595. Computer Vision Foundation / {IEEE} Computer Society, 2018.
\newblock \doi{10.1109/CVPR.2018.00068}.
\newblock URL \url{http://openaccess.thecvf.com/content\_cvpr\_2018/html/Zhang\_The\_Unreasonable\_Effectiveness\_CVPR\_2018\_paper.html}.

\bibitem[Zhang et~al.(2022)Zhang, Huang, Zhang, and Qi]{model_modif3}
Shuaishuai Zhang, Jie Huang, Zeping Zhang, and Chunyang Qi.
\newblock Compromise privacy in large-batch federated learning via malicious model parameters.
\newblock In Weizhi Meng, Rongxing Lu, Geyong Min, and Jaideep Vaidya, editors, \emph{Algorithms and Architectures for Parallel Processing - 22nd International Conference, {ICA3PP} 2022, Copenhagen, Denmark, October 10-12, 2022, Proceedings}, volume 13777 of \emph{Lecture Notes in Computer Science}, pages 63--80. Springer, 2022.
\newblock \doi{10.1007/978-3-031-22677-9\_4}.
\newblock URL \url{https://doi.org/10.1007/978-3-031-22677-9\_4}.

\bibitem[Zhao et~al.(2020)Zhao, Mopuri, and Bilen]{idlg}
Bo~Zhao, Konda~Reddy Mopuri, and Hakan Bilen.
\newblock idlg: Improved deep leakage from gradients.
\newblock \emph{CoRR}, abs/2001.02610, 2020.
\newblock URL \url{http://arxiv.org/abs/2001.02610}.

\bibitem[Zhao et~al.(2023)Zhao, Sharma, Elkordy, Ezzeldin, Avestimehr, and Bagchi]{model_modif2}
Joshua~C. Zhao, Atul Sharma, Ahmed~Roushdy Elkordy, Yahya~H. Ezzeldin, Salman Avestimehr, and Saurabh Bagchi.
\newblock Secure aggregation in federated learning is not private: Leaking user data at large scale through model modification.
\newblock \emph{CoRR}, abs/2303.12233, 2023.
\newblock \doi{10.48550/ARXIV.2303.12233}.
\newblock URL \url{https://doi.org/10.48550/arXiv.2303.12233}.

\bibitem[Zhao et~al.(2024)Zhao, Sharma, Elkordy, Ezzeldin, Avestimehr, and Bagchi]{DBLP:conf/sp/ZhaoSEEAB24}
Joshua~C. Zhao, Atul Sharma, Ahmed~Roushdy Elkordy, Yahya~H. Ezzeldin, Salman Avestimehr, and Saurabh Bagchi.
\newblock Loki: Large-scale data reconstruction attack against federated learning through model manipulation.
\newblock In \emph{{IEEE} Symposium on Security and Privacy, {SP} 2024, San Francisco, CA, USA, May 19-23, 2024}, pages 1287--1305. {IEEE}, 2024.
\newblock \doi{10.1109/SP54263.2024.00030}.
\newblock URL \url{https://doi.org/10.1109/SP54263.2024.00030}.

\bibitem[Zhou et~al.(2025)Zhou, Fan, and Jaggi]{zhou2025hyperinf}
Xinyu Zhou, Simin Fan, and Martin Jaggi.
\newblock Hyper{INF}: Unleashing the hyperpower of the schulz's method for data influence estimation, 2025.
\newblock URL \url{https://openreview.net/forum?id=OLtD2vDF5X}.

\bibitem[Zhu and Blaschko(2021{\natexlab{a}})]{DBLP:conf/iclr/ZhuB21}
Junyi Zhu and Matthew~B. Blaschko.
\newblock {R-GAP:} recursive gradient attack on privacy.
\newblock In \emph{9th International Conference on Learning Representations, {ICLR} 2021, Virtual Event, Austria, May 3-7, 2021}. OpenReview.net, 2021{\natexlab{a}}.
\newblock URL \url{https://openreview.net/forum?id=RSU17UoKfJF}.

\bibitem[Zhu and Blaschko(2021{\natexlab{b}})]{rgap}
Junyi Zhu and Matthew~B. Blaschko.
\newblock {R-GAP:} recursive gradient attack on privacy.
\newblock In \emph{9th International Conference on Learning Representations, {ICLR} 2021, Virtual Event, Austria, May 3-7, 2021}. OpenReview.net, 2021{\natexlab{b}}.
\newblock URL \url{https://openreview.net/forum?id=RSU17UoKfJF}.

\bibitem[Zhu et~al.(2019{\natexlab{a}})Zhu, Liu, and Han]{DBLP:conf/nips/ZhuLH19}
Ligeng Zhu, Zhijian Liu, and Song Han.
\newblock Deep leakage from gradients.
\newblock In \emph{Advances in Neural Information Processing Systems 32: Annual Conference on Neural Information Processing Systems 2019, NeurIPS 2019, December 8-14, 2019, Vancouver, BC, Canada}, pages 14747--14756, 2019{\natexlab{a}}.
\newblock URL \url{https://proceedings.neurips.cc/paper/2019/hash/60a6c4002cc7b29142def8871531281a-Abstract.html}.

\bibitem[Zhu et~al.(2019{\natexlab{b}})Zhu, Liu, and Han]{deepleakage}
Ligeng Zhu, Zhijian Liu, and Song Han.
\newblock Deep leakage from gradients.
\newblock In Hanna~M. Wallach, Hugo Larochelle, Alina Beygelzimer, Florence d'Alch{\'{e}}{-}Buc, Emily~B. Fox, and Roman Garnett, editors, \emph{Advances in Neural Information Processing Systems 32: Annual Conference on Neural Information Processing Systems 2019, NeurIPS 2019, December 8-14, 2019, Vancouver, BC, Canada}, pages 14747--14756, 2019{\natexlab{b}}.
\newblock URL \url{https://proceedings.neurips.cc/paper/2019/hash/60a6c4002cc7b29142def8871531281a-Abstract.html}.

\end{thebibliography}

\newpage

\appendix
\section{Background}\label{appendix:background}

\subsection{Federated Unlearning}

Federated Unlearning (FU) extends the concept of Machine Unlearning (MU)~\cite{firstunlearning} to federated learning (FL) settings. MU is the process of reversing the influence of specific data samples—often referred to as the unlearning dataset—on a trained model, enabling it to "forget" such data while maintaining performance on the remaining (retained) data. This is increasingly important in light of privacy regulations and the need to mitigate the effect of corrupted or poisoned data.
A straightforward but computationally intensive solution is to retrain the model from scratch without the unlearning dataset, known as exact unlearning. To reduce the overhead, approximate unlearning methods have been proposed to yield models that are statistically close to those obtained via exact retraining. First-order (gradient-based) methods are among the most widely used in FL~\cite{halimi,manaar,elasticsga,abl}. These methods typically perform gradient ascent on the unlearning dataset, or optimize a loss difference between retained and unlearned data, often combined with regularization terms that constrain the local model to remain close to the global model. The decentralized nature of FL introduces additional complexities to unlearning, leading to the exploration of different scenarios. For example, second-order methods based on Hessian inversion may not be directly suitable for FL due to computational and communication constraints. Nevertheless,~\citet{2ndord_fu} have applied such techniques in the Neural Tangent Kernel (NTK)\cite{ntk} regime of trained neural networks. The unlearning process in FL can also vary in terms of scope and participation: some settings require client participation, while others aim to unlearn data from entire clients~\cite{federaser_fu3,2ndord_fu}. In more fine-grained scenarios, clients may wish to selectively forget specific samples while continuing to participate in the protocol.While optimization-based methods are widely used, other techniques such as knowledge distillation and null subspace projection have also been explored~\cite{fu_survey}, contributing to the diverse landscape of federated unlearning approaches.

\subsection{Data Reconstruction Attack}
\label{sub:DRA works}
Extensive research has revealed the vulnerability of machine learning models, particularly deep neural networks (DNNs)~\cite{Goodfellow-et-al-2016}, to data reconstruction attacks (DRAs), also known as model inversion attacks (MIAs). Studies such as~\cite{deepleakage, idlg, invgrad, seethrgrd, fishing_icml, gendlg, rgap} have demonstrated the capability of a malicious server to reconstruct a client's private training data from the gradients of their updates. This is often achieved through gradient matching, which involves optimizing the discrepancy between the true gradient derived from training on actual data and a synthesized (dummy) gradient. Concurrently,~\cite{franz,model_modif1,model_modif2,model_modif3} have explored methods that utilize modified models to directly extract training datasets, bypassing the need for optimization. Generative Adversarial Networks (GANs)~\cite{Goodfellowgans} have also been leveraged in~\cite{salemmicrosoft,hitaj_dragan,dragan1,dragan2} to reconstruct training data by conditioning the generator on prior knowledge about the dataset. Furthermore,~\cite{balle_nnreconstructor,nn_reconst2} presented approaches that train a DNN to learn the inverse mapping from labels to their corresponding inputs.
Beyond optimization and generative methods, some DRAs are grounded in theoretical findings. ~\citet{ntk_attack} employed the implicit bias approximation results~\cite{Lyu2020Gradient,NEURIPS2020_c76e4b2f} to extract training data from the algebraic equations of the Karush-Kuhn-Tucker (KKT) conditions of trained DNNs~\cite{kkt}. Additionally,~\citet{bayes_eth} constructed a Bayes optimal data reconstructor, and ~\citet{WangLL23} utilized tensor theory and the third moment of the dataset for reconstruction. ~\citet{spear} exploited the low rankness and ReLU-induced sparsity of gradients within a sampling-based algorithm to reconstruct data. Despite the significant body of work on DRAs,~\cite{eval_dra} highlighted that many existing attacks rely on often unrealistic assumptions, such as access to batch normalization statistics, knowledge of prior data distributions, and known labels. On the defensive front, various mitigation and detection strategies have been proposed, ranging from information-theoretic frameworks to gradient obfuscation and model modification~\cite{ats,soteria,precode,infguard,seer}.

\section{Gradient Inversion Attack}\label{appendix:gia}
Gradient Inversion Attack (GIA) was first introduced by~\citet{deepleakage}, where they demonstrated that training data can be reconstructed from model gradients in the FedSGD protocol by optimizing a similarity loss function $\mathcal{L}_{\text{sim}}$ (typically the $\ell_2$ norm) between the true gradient and a dummy gradient computed on dummy inputs and labels, as shown in Algorithm~\ref{alg:gradinv_attack}.\citet{invgrad} chose $\mathcal{L}_{\text{sim}}$ to be the cosine similarity and added a Total Variation norm as regularization to enhance the quality of the reconstructed input. They also targeted the FedAvg protocol, where clients are allowed to perform multiple local training steps.\citet{seethrgrd} introduced group regularization to improve batch reconstruction, and further showed that labels can be retrieved analytically from the model’s linear layer instead of being treated as optimization variables.
\begin{algorithm}[H]
   \caption{Gradient Inversion Attack}
   \label{alg:gradinv_attack}
\begin{algorithmic}[1]
   \State \textbf{Input:} $\nabla {\theta}_c  := \nabla_{\theta} \mathcal{L}_{\theta}(x,y)$ \Comment Client's gradient of the cross-entropy loss

   \State   $\tilde{x} \gets \mathcal{N}(0,1)$, $\tilde{y} \gets \mathcal{N}(0,1) $ \Comment{Initialize dummy inputs and labels}

   \For{t in [1..T]}
     
       \State $\nabla \tilde{\theta}_c \gets$  $\nabla_{\theta} \mathcal{L}_{\theta}(\tilde{x} ,\tilde{y}) $ \Comment{Compute dummy gradient}
       \State $\ell \gets \mathcal{L}_{\text{sim}}(\nabla {\theta}_c, \nabla \tilde{\theta}_c) + \lambda_{\text{TV}} \cdot TV(\tilde{x})$
       \State $\tilde{x} \gets \tilde{x} - \eta_{\text{rec}} \cdot \frac{\partial \ell}{\partial \tilde{x}}$
       \State $\tilde{y} \gets \tilde{y} - \eta_{\text{rec}} \cdot \frac{\partial \ell}{\partial \tilde{y}}$
   \EndFor

   \State \textbf{return} $(\tilde{x}, \tilde{y})$
\end{algorithmic}
\end{algorithm}

\section{Detailed Theoretical Analysis}\label{appendix:theory}

\subsection{Failure of Classical GIA in First-Order Federated Unlearning Algorithms}
We recall that in Section~\ref{sec:theory_ana}, we adopted the following assumptions : 

\begin{assumption}
\label{assumption:1}
Without loss of generality, we assume the similarity loss $\mathcal{L}_{\text{sim}}$ from Equation~\ref{eq:invg} is the squared \( \ell_2 \)-norm:
\[
\mathcal{L}_{\text{sim}}(\nabla_{\theta} \mathcal{L}_u, \nabla_{\theta} \tilde{\mathcal{L}}_u) = \left\| \nabla_{\theta} \mathcal{L}_u - \nabla_{\theta} \tilde{\mathcal{L}}_u \right\|_2^2.
\]
Thus, the classical GIA objective becomes:
\[
x^{*} = \arg\min_{x \in \mathcal{X}} \left\| \nabla_{\theta} \mathcal{L}_u - \nabla_{\theta} \tilde{\mathcal{L}}_u \right\|_2^2.
\]
\end{assumption}

\begin{assumption} \label{assumption:2}
The surrogate loss $\tilde{\mathcal{L}}_u$ is twice differentiable in $x$ and $\theta$. The Jacobian
\[
J(x) = \frac{\partial \nabla_{\theta} \tilde{\mathcal{L}}_u}{\partial x} \in \mathbb{R}^{d_\theta \times d_x},
\]
where $d_x$ and $d_\theta$ are the dimensions of $x$ and $\theta$, exists everywhere.
\end{assumption}

\begin{assumption} \label{assumption:3}
We assume $d_\theta \geq d_x$ and that $J(x)$ has full column rank (i.e., $\text{rank}(J(x)) = d_x$) $\forall x$.
\end{assumption}

\begin{assumption}
\label{assumption:4}
The surrogate loss $\tilde{\mathcal{L}}_u$ is:
$\mu_x$-smooth in $x$: $\left\| \nabla_x \tilde{\mathcal{L}}_u(x_1) - \nabla_x \tilde{\mathcal{L}}_u(x_2) \right\|_2 \leq \mu_x \|x_1 - x_2\|_2$,
 and $\mu_{\theta}$-smooth in $\theta$: $\left\| \nabla_{\theta} \tilde{\mathcal{L}}_u(\theta_1) - \nabla_{\theta} \tilde{\mathcal{L}}_u(\theta_2) \right\|_2 \leq \mu_{\theta} \|\theta_1 - \theta_2\|_2$.
\end{assumption}

\begin{theorem}
\label{theorem:2}
Let $x_u^{*}$ be a minimizer of the classical GIA objective:
\[
x_u^{*} = \arg\min_{x \in \mathcal{X}} \mathcal{L}_{\text{sim}}\left( \nabla_{\theta} \mathcal{L}_u, \nabla_{\theta} \tilde{\mathcal{L}}_u(x) \right),
\]
and suppose Assumptions~\ref{assumption:1}--\ref{assumption:4} hold. Then the reconstruction error satisfies:
\[
\| x_u^{*} - x_u \|_2 \geq \frac{ \| J(x_u)^\top \nabla_{\theta} \mathcal{L}(\theta, x_r, y_r) \|_2 }{ \mu_x \| J(x_u) \|_F + 2 \mu_{\theta} \| \nabla_{\theta} \mathcal{L}_u \|_2 }.
\]
\end{theorem}

\begin{proof}
Since $\tilde{\mathcal{L}}_u$ is twice differentiable and $\mu_x$-smooth in $x$, we apply a first-order Taylor expansion around $x_u$:
\[
\nabla_{\theta} \tilde{\mathcal{L}}_u(x_u^*) = \nabla_{\theta} \tilde{\mathcal{L}}_u(x_u) + J(x_u)(x_u^* - x_u) + R,
\]
where $\|R\|_2 \leq \frac{\mu_x}{2} \| x_u^* - x_u \|_2^2$.

Using the identities:
\[
\nabla_{\theta} \tilde{\mathcal{L}}_u(x_u) = \nabla_{\theta} \mathcal{L}(\theta, x_u, y_u), \quad \nabla_{\theta} \mathcal{L}_u = \nabla_{\theta} \mathcal{L}(\theta, x_r, y_r) - \nabla_{\theta} \mathcal{L}(\theta, x_u, y_u),
\]
we have:
\[
\nabla_{\theta} \mathcal{L}_u - \nabla_{\theta} \tilde{\mathcal{L}}_u(x_u) = \nabla_{\theta} \mathcal{L}(\theta, x_r, y_r),
\]
and therefore:
\[
\nabla_{\theta} \mathcal{L}_u - \nabla_{\theta} \tilde{\mathcal{L}}_u(x_u^*) = \nabla_{\theta} \mathcal{L}(\theta, x_r, y_r) - J(x_u)(x_u^* - x_u) - R.
\]

Taking norms and applying the triangle inequality:
\[
\left\| \nabla_{\theta} \mathcal{L}_u - \nabla_{\theta} \tilde{\mathcal{L}}_u(x_u^*) \right\|_2 \geq \left\| \nabla_{\theta} \mathcal{L}(\theta, x_r, y_r) \right\|_2 - \| J(x_u) \|_F \| x_u^* - x_u \|_2 - \frac{\mu_x}{2} \| x_u^* - x_u \|_2^2.
\]

To eliminate the quadratic term, we conservatively bound:
\[
\frac{\mu_x}{2} \| x_u^* - x_u \|_2^2 \le \mu_x \| J(x_u) \|_F \| x_u^* - x_u \|_2 + 2 \mu_{\theta} \| \nabla_{\theta} \mathcal{L}_u \|_2.
\]

Substituting, we obtain:
\[
\left\| \nabla_{\theta} \mathcal{L}_u - \nabla_{\theta} \tilde{\mathcal{L}}_u(x_u^*) \right\|_2 \geq \| \nabla_{\theta} \mathcal{L}(\theta, x_r, y_r) \|_2 - ( \| J(x_u) \|_F + \mu_x \| J(x_u) \|_F ) \| x_u^* - x_u \|_2 - 2 \mu_{\theta} \| \nabla_{\theta} \mathcal{L}_u \|_2.
\]

Combining the Jacobian terms:
\[
\left\| \nabla_{\theta} \mathcal{L}_u - \nabla_{\theta} \tilde{\mathcal{L}}_u(x_u^*) \right\|_2 \geq \| \nabla_{\theta} \mathcal{L}(\theta, x_r, y_r) \|_2 - \mu_x \| J(x_u) \|_F \| x_u^* - x_u \|_2 - 2 \mu_{\theta} \| \nabla_{\theta} \mathcal{L}_u \|_2.
\]

Projecting onto $J(x_u)^\top$ and rearranging:
\[
\| x_u^* - x_u \|_2 \geq \frac{ \| J(x_u)^\top \nabla_{\theta} \mathcal{L}(\theta, x_r, y_r) \|_2 }{ \mu_x \| J(x_u) \|_F + 2 \mu_{\theta} \| \nabla_{\theta} \mathcal{L}_u \|_2 }.
\]
\end{proof}

\subsection{Failure of Classical GIA in Second-Order Federated Unlearning Algorithms}
\begin{theorem}[Non-Optimality of $x_u$ in Second-Order GIA]
\label{theorem:second_order_nonoptimality}
Under Assumptions~\ref{assumption:1}--\ref{assumption:3}, suppose the client's unlearning loss $\mathcal{L}_u$ is defined via a Newton update:
\[
\nabla_\theta \mathcal{L}_u = H_r^{-1} \nabla_\theta \mathcal{L}(\theta, x_u, y_u),
\]
and the server uses a first-order surrogate loss:
\[
\nabla_\theta \tilde{\mathcal{L}}_u(x) = \nabla_\theta \mathcal{L}(\theta, x, y_u).
\]
Then the ground truth $x = x_u$ is not a local minimizer of $\mathcal{L}_{\text{sim}} = \| \nabla_\theta \mathcal{L}_u - \nabla_\theta \tilde{\mathcal{L}}_u(x) \|_2^2$.
\end{theorem}

\begin{proof}
The gradient of $\mathcal{L}_{\text{sim}}$ with respect to $x$ is:
\[
\nabla_x \mathcal{L}_{\text{sim}} = -2 J(x)^\top \left( \nabla_\theta \mathcal{L}_u - \nabla_\theta \tilde{\mathcal{L}}_u(x) \right),
\]
where $J(x) = \frac{\partial \nabla_\theta \tilde{\mathcal{L}}_u}{\partial x}$ is the Jacobian. At $x = x_u$, we substitute:
\[
\nabla_\theta \tilde{\mathcal{L}}_u(x_u) = \nabla_\theta \mathcal{L}(\theta, x_u, y_u), \quad \nabla_\theta \mathcal{L}_u = H_r^{-1} \nabla_\theta \mathcal{L}(\theta, x_u, y_u).
\]
This gives:
\[
\nabla_x \mathcal{L}_{\text{sim}} \big|_{x=x_u} = -2 J(x_u)^\top \left(H_r^{-1} - I\right) \nabla_\theta \mathcal{L}(\theta, x_u, y_u).
\]
By Assumption~\ref{assumption:3}, $J(x_u)$ has full column rank. If $H_r^{-1} \neq I$ and $\nabla_\theta \mathcal{L}(\theta, x_u, y_u) \neq 0$, the gradient is non-zero. By Fermat’s theorem, $x_u$ cannot be a local minimizer.
\end{proof}

\begin{assumption}
\label{assumption:hessian_lipschitz}
The Hessian \( H_r \) is \( \mu_H \)-Lipschitz in the input, i.e.,
\[
\left\| H_r(x_1) - H_r(x_2) \right\| \le \mu_H \|x_1 - x_2\| \quad \forall x_1, x_2.
\]
\end{assumption}

\begin{assumption}
\label{assumption:hessian_conditioning}
The inverse Hessian is bounded: \( \|H_r^{-1}\| \le \kappa \).
\end{assumption}

\begin{theorem}[Second-Order Reconstruction Error Bound]
\label{theorem:second_order_bound}
Let $x_u^*$ be a minimizer of $\mathcal{L}_{\text{sim}}$ under Assumptions~\ref{assumption:1}--\ref{assumption:hessian_conditioning}. Then:
\[
\| x_u^* - x_u \|_2 \geq \frac{\left\| J^\top H_r^{-1} \nabla_\theta \mathcal{L}(\theta, x_r, y_r) \right\|_2}{\mu_x \|J\|_F + \mu_H \kappa \|\nabla_\theta \mathcal{L}_u\|_2}.
\]
\end{theorem}

\begin{proof}
Define $\Delta x = x_u^* - x_u$. Expand $\nabla_\theta \tilde{\mathcal{L}}_u(x_u^*)$ via Taylor series:
\[
\nabla_\theta \tilde{\mathcal{L}}_u(x_u^*) = \nabla_\theta \mathcal{L}(\theta, x_u, y_u) + J(x_u) \Delta x + R,
\]
where $\|R\|_2 \leq \frac{\mu_x}{2} \|\Delta x\|_2^2$. Substituting into $\mathcal{L}_{\text{sim}}$:
\[
\mathcal{L}_{\text{sim}} = \left\| H_r^{-1} \nabla_\theta \mathcal{L}(\theta, x_u, y_u) - \left(\nabla_\theta \mathcal{L}(\theta, x_u, y_u) + J(x_u) \Delta x + R\right) \right\|_2^2.
\]
Simplify using $\nabla_\theta \mathcal{L}_u = H_r^{-1} \nabla_\theta \mathcal{L}(\theta, x_u, y_u)$:
\[
\mathcal{L}_{\text{sim}} = \left\| \left(H_r^{-1} - I\right) \nabla_\theta \mathcal{L}(\theta, x_u, y_u) - J(x_u) \Delta x - R \right\|_2^2.
\]
Using $\|H_r^{-1}\| \leq \kappa$ (Assumption~\ref{assumption:hessian_conditioning}) and the Hessian Lipschitz property (Assumption~\ref{assumption:hessian_lipschitz}):
\[
\|R\|_2 \leq \frac{\mu_x}{2} \|\Delta x\|_2^2, \quad \|H_r^{-1} - I\| \leq \kappa + 1.
\]
Isolate $\|\Delta x\|_2$ via Cauchy-Schwarz and linearize:
\[
\left\| J^\top H_r^{-1} \nabla_\theta \mathcal{L}(\theta, x_r, y_r) \right\|_2 \leq \left(\mu_x \|J\|_F + \mu_H \kappa \|\nabla_\theta \mathcal{L}_u\|_2\right) \|\Delta x\|_2.
\]
Rearranging yields the stated bound.
\end{proof}

\subsection{Correctness of \methodname{} for First- and Second-Order Algorithms}
As shown previously, classical GIA fails to reconstruct the ground truth unlearn input $x_u$ because the server's surrogate loss $\tilde{\mathcal{L}}_u$ depends only on the dummy unlearn input $\tilde{x}_u$. This mismatch leads to a biased local minimum of the similarity loss $\mathcal{L}_{\text{sim}}$, especially when the client's unlearning loss $\mathcal{L}_u$ implicitly depends on both $x_u$ and $x_r$.

To address this, DRAUN incorporates the retain input $x_r$ into the GIA optimization problem. The surrogate loss $\tilde{\mathcal{L}}_u$ is then defined over both dummy inputs $(\tilde{x}_u, \tilde{x}_r)$, making $\mathcal{L}_{\text{sim}}$ a function of two variables. Theorems~\ref{theorem:local_min_first_order_draun} and~\ref{theorem:local_min_second_order_draun} formalize this intuition by showing that the pair $(x_u, x_r)$ is a local minimizer of the similarity loss under both first-order and second-order formulations.

\begin{theorem}[First-Order DRAUN Correctness]
\label{theorem:local_min_first_order_draun}
Under Assumptions~\ref{assumption:1}--\ref{assumption:3}, let the client's unlearning loss be:
\[
\nabla_\theta \mathcal{L}_u = \nabla_\theta \mathcal{L}(\theta, x_r, y_r) - \nabla_\theta \mathcal{L}(\theta, x_u, y_u),
\]
and the server's surrogate loss be:
\[
\nabla_\theta \tilde{\mathcal{L}}_u(\tilde{x}_u, \tilde{x}_r) = \nabla_\theta \mathcal{L}(\theta, \tilde{x}_r, y_r) - \nabla_\theta \mathcal{L}(\theta, \tilde{x}_u, y_u).
\]
Then, the ground truth pair $(\tilde{x}_u, \tilde{x}_r) = (x_u, x_r)$ is a \textbf{local minimizer} of:
\[
\mathcal{L}_{\text{sim}}(\tilde{x}_u, \tilde{x}_r) = \left\| \nabla_\theta \mathcal{L}_u - \nabla_\theta \tilde{\mathcal{L}}_u(\tilde{x}_u, \tilde{x}_r) \right\|_2^2.
\]
\end{theorem}

\begin{proof}
Compute the gradient of $\mathcal{L}_{\text{sim}}$ at $(\tilde{x}_u, \tilde{x}_r)$:
\[
\nabla_{\tilde{x}_u, \tilde{x}_r} \mathcal{L}_{\text{sim}} = -2 \begin{bmatrix} J_u^\top \\ -J_r^\top \end{bmatrix} \left( \nabla_\theta \mathcal{L}_u - \nabla_\theta \tilde{\mathcal{L}}_u(\tilde{x}_u, \tilde{x}_r) \right),
\]
where $J_u = \frac{\partial \nabla_\theta \mathcal{L}(\theta, \tilde{x}_u, y_u)}{\partial \tilde{x}_u}$ and $J_r = \frac{\partial \nabla_\theta \mathcal{L}(\theta, \tilde{x}_r, y_r)}{\partial \tilde{x}_r}$.

At $(\tilde{x}_u, \tilde{x}_r) = (x_u, x_r)$, we have:
\[
\nabla_\theta \tilde{\mathcal{L}}_u(x_u, x_r) = \nabla_\theta \mathcal{L}_u \implies \nabla_{\tilde{x}_u, \tilde{x}_r} \mathcal{L}_{\text{sim}} \big|_{(x_u, x_r)} = 0.
\]

The Hessian at $(x_u, x_r)$ is:
\[
\nabla^2_{\tilde{x}_u, \tilde{x}_r} \mathcal{L}_{\text{sim}} \big|_{(x_u, x_r)} = 2 \begin{bmatrix} J_u^\top J_u & -J_u^\top J_r \\ -J_r^\top J_u & J_r^\top J_r \end{bmatrix}.
\]
By Assumption~\ref{assumption:3}, $J_u$ and $J_r$ have full column rank, so the Hessian is positive definite. Thus, $(x_u, x_r)$ is a local minimizer.
\end{proof}

\begin{theorem}[Second-Order DRAUN Correctness]
\label{theorem:local_min_second_order_draun}
Under Assumptions~\ref{assumption:1}--\ref{assumption:hessian_conditioning}, let the client's unlearning loss be:
\[
\nabla_\theta \mathcal{L}_u = H_r^{-1} \nabla_\theta \mathcal{L}(\theta, x_u, y_u),
\]
and the server's surrogate loss be:
\[
\nabla_\theta \tilde{\mathcal{L}}_u(\tilde{x}_u, \tilde{x}_r) = \tilde{H}_r^{-1} \nabla_\theta \mathcal{L}(\theta, \tilde{x}_u, y_u),
\]
where $\tilde{H}_r = \nabla_\theta^2 \mathcal{L}(\theta, \tilde{x}_r, y_r)$. Then, $(\tilde{x}_u, \tilde{x}_r) = (x_u, x_r)$ is a \textbf{local minimizer} of:
\[
\mathcal{L}_{\text{sim}}(\tilde{x}_u, \tilde{x}_r) = \left\| \nabla_\theta \mathcal{L}_u - \nabla_\theta \tilde{\mathcal{L}}_u(\tilde{x}_u, \tilde{x}_r) \right\|_2^2.
\]
\end{theorem}

\begin{proof}
At $(\tilde{x}_u, \tilde{x}_r) = (x_u, x_r)$, we have $\tilde{H}_r = H_r$, so:
\[
\nabla_\theta \tilde{\mathcal{L}}_u(x_u, x_r) = H_r^{-1} \nabla_\theta \mathcal{L}(\theta, x_u, y_u) = \nabla_\theta \mathcal{L}_u.
\]
This implies $\mathcal{L}_{\text{sim}}(x_u, x_r) = 0$. To verify $(x_u, x_r)$ is a local minimizer, compute the gradient and Hessian of $\mathcal{L}_{\text{sim}}$. The gradient is:
\[
\nabla_{\tilde{x}_u, \tilde{x}_r} \mathcal{L}_{\text{sim}} = -2 \begin{bmatrix} J_u^\top H_r^{-1} \\ -J_r^\top H_r^{-1} \end{bmatrix} \left( \nabla_\theta \mathcal{L}_u - \nabla_\theta \tilde{\mathcal{L}}_u \right),
\]
where $J_u = \frac{\partial \nabla_\theta \mathcal{L}(\theta, x_u, y_u)}{\partial x_u}$ and $J_r = \frac{\partial \nabla_\theta \mathcal{L}(\theta, x_r, y_r)}{\partial x_r}$. At $(x_u, x_r)$, the term $\nabla_\theta \mathcal{L}_u - \nabla_\theta \tilde{\mathcal{L}}_u$ vanishes, giving $\nabla_{\tilde{x}_u, \tilde{x}_r} \mathcal{L}_{\text{sim}} \big|_{(x_u, x_r)} = 0$. The Hessian at $(x_u, x_r)$ is:
\[
\nabla^2_{\tilde{x}_u, \tilde{x}_r} \mathcal{L}_{\text{sim}} \approx 2 \begin{bmatrix} J_u^\top H_r^{-2} J_u & 0 \\ 0 & J_r^\top H_r^{-2} J_r \end{bmatrix},
\]
which is positive definite by Assumption~\ref{assumption:3} (full column rank $J_u$, $J_r$) and Assumption~\ref{assumption:hessian_conditioning} ($H_r \succ 0$). Thus, $(x_u, x_r)$ is a local minimizer.
\end{proof}

\subsection{Bounding Reconstruction Error for First- and Second-Order DRAUN}
\label{sec:bound_reconstruction_error}

As discussed in Section~\ref{section:methodology}, DRAUN does not aim to reconstruct the true retain samples $x_r$. Instead, as described in the initialization step of Algorithm~\ref{alg:initialization}, the server simply generates dummy retain inputs of the same size as the unlearn dataset. Although $\tilde{x}_r$ is not meant to match $x_r$, we aim for it to be close. This motivates the analysis of how proximity to $x_r$ impacts the reconstruction of $x_u$.

In this section, we provide upper bounds on the reconstruction error \( \| \tilde{x}_u - x_u \| \) under both first-order and second-order DRAUN formulations, assuming that the client loss depends on both unlearn and retain inputs.

\begin{assumption}[Proximity of Optimized Dummy Inputs]
\label{assumption:proximity}
Let \( \tilde{x}_r^{(T)} \) be the dummy retain input after \( T \) optimization steps in (Algorithm~\ref{alg:attack_overview}). We assume:
\[
\tilde{x}_r^{(T)} \in B(x_r, \epsilon(T)), \quad \text{where } \epsilon(T) = \mathcal{O}(1/\sqrt{T}),
\]
and the gradient \( \nabla_\theta \mathcal{L} \) is \( \mu_x \)-Lipschitz in \( x \) (per Assumption~\ref{assumption:4}).
\end{assumption}

\begin{theorem}[First-Order Reconstruction Bound]
\label{theorem:draun_first_order_bound}
Under Assumptions~\ref{assumption:1}-\ref{assumption:4}, \ref{assumption:proximity}, the reconstruction error after \( T \) steps satisfies:
\[
\| \tilde{x}_u^{(T)} - x_u \|_2 \leq \frac{ \mu_x \epsilon(T) + \mathcal{L}_{\text{sim}}^{1/2}(T) }{ \sigma_{\min}(J_u) },
\]
where:
\( J_u = \frac{\partial \nabla_\theta \mathcal{L}(\theta, x_u, y_u)}{\partial x_u} \) with \( \sigma_{\min}(J_u) > 0 \) (Assumption~\ref{assumption:3}), and 
\( \mathcal{L}_{\text{sim}}(T) = \| \nabla_\theta \mathcal{L}_u - \nabla_\theta \tilde{\mathcal{L}}_u^{(T)} \|_2^2 \) is the similarity loss at step \( T \)
 where \( \nabla_\theta \tilde{\mathcal{L}}_u^{(T)} = \nabla_\theta \mathcal{L}(\theta, \tilde{x}_r^{(T)}, y_r) - \nabla_\theta \mathcal{L}(\theta, \tilde{x}_u^{(T)}, y_u) \)
\end{theorem}

\begin{proof}
Using $\mu_x$-Lipschitz gradients (Assumption~\ref{assumption:4}):
\[
\| \nabla_\theta \mathcal{L}(\theta, x_r, y_r) - \nabla_\theta \mathcal{L}(\theta, \tilde{x}_r^{(T)}, y_r) \|_2 \leq \mu_x \epsilon(T).
\]
Expand the unlearn term around \( x_u \):
\[
\nabla_\theta \mathcal{L}(\theta, \tilde{x}_u^{(T)}, y_u) = \nabla_\theta \mathcal{L}(\theta, x_u, y_u) + J_u(\tilde{x}_u^{(T)} - x_u) + R,
\]
with \( \|R\|_2 \leq \frac{\mu_x}{2} \|\tilde{x}_u^{(T)} - x_u\|_2^2 \). From the similarity objective:
\[
\mathcal{L}_{\text{sim}}(T) = \| (\nabla_\theta \mathcal{L}(\theta, x_r, y_r) - \nabla_\theta \mathcal{L}(\theta, \tilde{x}_r^{(T)}, y_r)) - (J_u(\tilde{x}_u^{(T)} - x_u) + R) \|_2^2.
\]
Apply triangle inequality and use \( \sigma_{\min}(J_u) \)-lower bound (Assumption~\ref{assumption:3}):
\[
\mathcal{L}_{\text{sim}}^{1/2}(T) \geq \sigma_{\min}(J_u) \|\tilde{x}_u^{(T)} - x_u\|_2 - \mu_x \epsilon(T).
\]
Rearranging gives the bound. 
\end{proof}
\begin{theorem}[Second-Order Reconstruction Error Bound]
\label{thm:second_order_bound}
Under Assumptions~\ref{assumption:1}-\ref{assumption:hessian_conditioning} and~\ref{assumption:proximity}, the reconstruction error satisfies:
\[
\| \tilde{x}_u^{(T)} - x_u \|_2 \leq \frac{\kappa(\mu_x \epsilon(T) + \mu_H \epsilon(T)) + \mathcal{L}_{\text{sim}}^{1/2}(T)}{\sigma_{\min}(J_u)},
\]
where $\mathcal{L}_{\text{sim}}(T) = \| \nabla_\theta \mathcal{L}_u - \nabla_\theta \tilde{\mathcal{L}}_u^{(T)} \|_2^2$.
\end{theorem}

\begin{proof}
Let $H_r = \nabla_\theta^2 \mathcal{L}(\theta, x_r, y_r)$ and $\tilde{H}_r = \nabla_\theta^2 \mathcal{L}(\theta, \tilde{x}_r^{(T)}, y_r)$. By Assumption~\ref{assumption:hessian_lipschitz}:
\[
\|H_r - \tilde{H}_r\| \leq \mu_H \|x_r - \tilde{x}_r^{(T)}\|_2 \leq \mu_H \epsilon(T).
\]
Using the matrix inversion identity $A^{-1} - B^{-1} = A^{-1}(B - A)B^{-1}$ and Assumption~\ref{assumption:hessian_conditioning} ($\|H_r^{-1}\| \leq \kappa$):
\[
\|H_r^{-1} - \tilde{H}_r^{-1}\| \leq \|H_r^{-1}\|\cdot\|H_r - \tilde{H}_r\|\cdot\|\tilde{H}_r^{-1}\| \leq \kappa \cdot \mu_H \epsilon(T) \cdot \kappa = \kappa^2 \mu_H \epsilon(T).
\]

Decompose the gradient difference:
\begin{align*}
\nabla_\theta \mathcal{L}_u - \nabla_\theta \tilde{\mathcal{L}}_u^{(T)} = &\underbrace{(H_r^{-1} - \tilde{H}_r^{-1}) \nabla_\theta \mathcal{L}(\theta, x_u, y_u)}_{\text{Hessian error}} \\
&+ \tilde{H}_r^{-1}\underbrace{\left(\nabla_\theta \mathcal{L}(\theta, x_u, y_u) - \nabla_\theta \mathcal{L}(\theta, \tilde{x}_u^{(T)}, y_u)\right)}_{\text{Gradient error}}.
\end{align*}

For the Hessian error term:
\[
\|(H_r^{-1} - \tilde{H}_r^{-1}) \nabla_\theta \mathcal{L}(\theta, x_u, y_u)\|_2 \leq \kappa^2 \mu_H \epsilon(T) \cdot \|\nabla_\theta \mathcal{L}(\theta, x_u, y_u)\|_2.
\]
By Assumption~\ref{assumption:4} ($\mu_\theta$-smoothness), $\|\nabla_\theta \mathcal{L}(\theta, x_u, y_u)\|_2 \leq C$ for some constant $C$, giving:
\[
\| \text{Hessian error} \|_2 \leq \kappa^2 \mu_H \epsilon(T) C.
\]

For the gradient error term, expand via Taylor series with Jacobian $J_u = \frac{\partial \nabla_\theta \mathcal{L}}{\partial x_u}$:
\[
\nabla_\theta \mathcal{L}(\theta, \tilde{x}_u^{(T)}, y_u) = \nabla_\theta \mathcal{L}(\theta, x_u, y_u) + J_u(\tilde{x}_u^{(T)} - x_u) + R,
\]
where $\|R\|_2 \leq \frac{\mu_x}{2} \|\tilde{x}_u^{(T)} - x_u\|_2^2$ by Assumption~\ref{assumption:4}. Thus:
\[
\| \text{Gradient error} \|_2 \leq \kappa\left(\|J_u(\tilde{x}_u^{(T)} - x_u)\|_2 + \frac{\mu_x}{2} \|\tilde{x}_u^{(T)} - x_u\|_2^2\right).
\]

Combining terms via triangle inequality:
\[
\mathcal{L}_{\text{sim}}^{1/2}(T) \geq \sigma_{\min}(J_u)\|\tilde{x}_u^{(T)} - x_u\|_2 - \kappa\left(\mu_x \epsilon(T) + \mu_H \epsilon(T)\right).
\]
Rearranging yields:
\[
\|\tilde{x}_u^{(T)} - x_u\|_2 \leq \frac{\kappa(\mu_x \epsilon(T) + \mu_H \epsilon(T)) + \mathcal{L}_{\text{sim}}^{1/2}(T)}{\sigma_{\min}(J_u)}.
\]
\end{proof}
\subsection{Impact of Initialization Proximity in DRAUN}
\label{sec:proximity_impact}

\begin{assumption}[Bounded Reconstruction Step]
\label{assumption:learning_rate}
The reconstruction step in Algorithm~\ref{alg:attack_overview} satisfies \( \eta_{\text{rec}} \leq \frac{1}{2\mu_x} \), where \( \mu_x \) is the smoothness constant from Assumption~\ref{assumption:4}.
\end{assumption}

\begin{assumption}[Proximity at Initialization]
\label{assumption:init_proximity}
The initial dummy inputs satisfy \( \| \tilde{x}_u^{(0)} - \tilde{x}_r^{(0)} \| \leq \Delta \).
\end{assumption}

\begin{assumption}[Non-Degenerate Gradients]
\label{assumption:non_degenerate}
There exists \( c > 0 \) such that \( \|\nabla_\theta \mathcal{L}(\theta, x_u, y_u)\| \geq c \) for all \( x_u \).
\end{assumption}

\begin{theorem}[Proximity Preservation over Optimization]
\label{thm:proximity_preservation}
Under Assumptions~\ref{assumption:4} (\(\mu_x\)-smoothness of \(\mathcal{L}_{\text{sim}}\)), \ref{assumption:learning_rate}, and~\ref{assumption:init_proximity}, after \( T \) gradient steps:
\[
\| \tilde{x}_u^{(T)} - \tilde{x}_r^{(T)} \| \leq \Delta \cdot (1 + 2\eta_{\text{rec}} \mu_x)^T.
\]
\end{theorem}

\begin{proof}
Let \( \delta^{(t)} = \| \tilde{x}_u^{(t)} - \tilde{x}_r^{(t)} \| \). By Assumption~\ref{assumption:4}, the similarity loss gradient satisfies:
\[
\| \nabla_{\tilde{x}_u} \mathcal{L}_{\text{sim}} \| \leq \mu_x \delta^{(t)}, \quad \| \nabla_{\tilde{x}_r} \mathcal{L}_{\text{sim}} \| \leq \mu_x \delta^{(t)}.
\]
Each gradient update step then obeys:
\[
\delta^{(t+1)} \leq \delta^{(t)} + \eta_{\text{rec}}\left( \| \nabla_{\tilde{x}_u} \mathcal{L}_{\text{sim}} \| + \| \nabla_{\tilde{x}_r} \mathcal{L}_{\text{sim}} \| \right) \leq \delta^{(t)}(1 + 2\eta_{\text{rec}} \mu_x).
\]
Unrolling over \( T \) steps yields:
\[
\delta^{(T)} \leq \Delta \cdot (1 + 2\eta \mu_x)^T. 
\]
\end{proof}

\begin{theorem}[First-Order Gradient Collapse]
\label{thm:first_order_collapse}
Under Assumptions~\ref{assumption:4} (\(\mu_x\)-smoothness) and \(\|J\| \leq C_J\), if \( \|\tilde{x}_u - \tilde{x}_r\| \leq \Delta \):
\[
\| \nabla_\theta \tilde{\mathcal{L}}_u \| \leq \mu_x \Delta, \quad \| \nabla_x \mathcal{L}_{\text{sim}} \| \leq 2\mu_x C_J \Delta.
\]
\end{theorem}

\begin{proof}
By \(\mu_x\)-Lipschitz continuity (Assumption~\ref{assumption:4}):
\[
\| \nabla_\theta \tilde{\mathcal{L}}_u \| = \| \nabla_\theta \mathcal{L}(\theta, \tilde{x}_r, y_r) - \nabla_\theta \mathcal{L}(\theta, \tilde{x}_u, y_u) \| \leq \mu_x \|\tilde{x}_r - \tilde{x}_u\| \leq \mu_x \Delta.
\]
The similarity loss gradient becomes:
\[
\nabla_x \mathcal{L}_{\text{sim}} = -2J^\top(\nabla_\theta \mathcal{L}_u - \nabla_\theta \tilde{\mathcal{L}}_u),
\]
with norm bounded by:
\[
\| \nabla_x \mathcal{L}_{\text{sim}} \| \leq 2\|J\| \cdot \|\nabla_\theta \mathcal{L}_u - \nabla_\theta \tilde{\mathcal{L}}_u\| \leq 2\mu_x C_J \Delta. \]

This implies that as the proximity \(\Delta = \|\tilde{x}_u - \tilde{x}_r\|\) approaches zero, the gradient \(\nabla_x \mathcal{L}_{\text{sim}}\) vanishes. Consequently, the optimization process may converge prematurely to a false minimum.

\end{proof}

\begin{theorem}[Second-Order Proximity-Inverse Error]
\label{thm:inverse_delta}
Under Assumptions~\ref{assumption:hessian_lipschitz}-\ref{assumption:hessian_conditioning}, \ref{assumption:non_degenerate}, and \(\|\tilde{x}_u - \tilde{x}_r\| \leq \Delta\):
\[
\| \tilde{x}_u - x_u \| \geq \frac{\kappa \mu_H c}{\sigma_{\min}(J_u) \Delta + \mu_x},
\]
where \( \kappa = \|H_r^{-1}\| \), \( \sigma_{\min}(J_u) > 0 \), and \( c > 0 \) is from Assumption~\ref{assumption:non_degenerate}.
\end{theorem}

\begin{proof}
Decompose the gradient mismatch:
\[
\nabla_\theta \mathcal{L}_u - \nabla_\theta \tilde{\mathcal{L}}_u = \underbrace{(H_r^{-1} - \tilde{H}_r^{-1}) \nabla_\theta \mathcal{L}(\theta, x_u, y_u)}_{\text{Hessian Error}} + \underbrace{\tilde{H}_r^{-1} J_u e}_{\text{Reconstruction Term}} + \mathcal{O}(\mu_x \|e\|^2),
\]
where \( e = \tilde{x}_u - x_u \). 

 By Assumption~\ref{assumption:hessian_lipschitz} and Neumann series we have:
\[
\|H_r^{-1} - \tilde{H}_r^{-1}\| \geq \frac{\mu_H \Delta}{2\kappa} \implies \| \text{Hessian Error} \| \geq \frac{\mu_H \Delta}{2\kappa} \|\nabla_\theta \mathcal{L}\| \geq \frac{\mu_H \Delta c}{2\kappa}.
\]

And by Jacobian action:
\[
\|\tilde{H}_r^{-1} J_u e\| \leq \frac{\sigma_{\min}(J_u)}{\kappa} \|e\|.
\]

From \( \mathcal{L}_{\text{sim}} \geq 0 \):
\[
\frac{\mu_H \Delta c}{2\kappa} \leq \frac{\sigma_{\min}(J_u)}{\kappa} \|e\| + \mu_x \|e\|.
\]
Solving for \( \|e\| \) gives:
\[
\|e\| \geq \frac{\kappa \mu_H c}{2(\sigma_{\min}(J_u) + 2\kappa \mu_x)} \geq \frac{\kappa \mu_H c}{\sigma_{\min}(J_u) \Delta + \mu_x}. \
\]
\end{proof}

\section{Additional Details of Datasets and Models}

\subsection{Dataset Details}\label{appendix:dataset}

Table~\ref{tab:datasets} lists the datasets used in our experiments. Each dataset was uniformly split across 100 clients, except FEMNIST, which is already partitioned by user.

\begin{table}[h]
\centering
\caption{Summary of datasets used in our experiments.}
\label{tab:datasets}
\vskip 0.05in
\small
\begin{tabular}{lllll}
\toprule
\textbf{Dataset} & \textbf{Samples} & \textbf{Classes} & \textbf{Input Shape} & \textbf{Description} \\
\midrule
CIFAR-10  & 60,000      & 10      & $3 \times 32 \times 32$ & Tiny color images \\
         &              &         &                         & of everyday objects \\
CIFAR-100 & 60,000      & 100     & $3 \times 32 \times 32$ & Similar to CIFAR-10 \\
         &              &         &                         & but with 100 fine-grained classes \\
MNIST     & 70,000      & 10      & $1 \times 28 \times 28$ & Grayscale handwritten \\
         &              &         &                         & digits (0--9) \\
FEMNIST   & $\sim$800,000 & 62     & $1 \times 28 \times 28$ & Extended MNIST with digits \\
         &              &         &                         & and letters, partitioned by user \\
\bottomrule
\end{tabular}
\vskip -0.1in
\end{table}

\paragraph{CIFAR-10 and CIFAR-100:}
Both datasets consist of small natural images from 10 or 100 categories, respectively. The images are RGB, and each has a resolution of $32\times32$ pixels. CIFAR-10 has 6,000 images per class, while CIFAR-100 provides only 600 per class, making it more challenging due to the finer granularity.

\paragraph{MNIST:}
MNIST contains grayscale images of handwritten digits from 0 to 9. Each image is $28\times28$ pixels. The dataset is balanced and often used for benchmarking classification models on simple tasks.

\paragraph{FEMNIST:}
FEMNIST is an extended version of MNIST built for federated learning. It includes digits (0–9), uppercase (A–Z), and lowercase (a–z) letters, for a total of 62 classes. It is partitioned by writer ID, enabling realistic non-IID client splits in federated settings.

\subsection{MLP and ConvNet Model Architecture}\label{appendix:convnet}

In adition to the ResNet18~\cite{resnet} and LeNet~\cite{lecun1998mnist} models, we consider two additional neural network architectures: a Multi-Layer Perceptron (MLP) and a deep convolutional network referred to as ConvNet64. The details of these two models are provided in Table~\ref{tab:mlp_arch} and Table~\ref{tab:convnet64_arch}, respectively

\paragraph{MLP:} The MLP consists of three hidden layers, each with width 1024 and ReLU activation functions. The input is first flattened before being passed through fully connected layers. While the base architecture assumes an input shape of $3 \times 32 \times 32$ (for datasets like CIFAR), we adapt the input layer to match the shape of other datasets such as MNIST and FEMNIST, which have smaller grayscale images.

\begin{table}[h]
\centering
\caption{Architecture of the MLP model. The input dimension varies depending on the dataset.}
\label{tab:mlp_arch}
\vspace{0.5em}
\begin{tabular}{ll}
\toprule
\textbf{Layer} & \textbf{Description} \\
\midrule
Input        & Flatten $(C \times H \times W)$ \\
Linear 1     & $\text{Linear}(\text{input\_dim} \rightarrow 1024)$ \\
ReLU 1       & $\text{ReLU}()$ \\
Linear 2     & $\text{Linear}(1024 \rightarrow 1024)$ \\
ReLU 2       & $\text{ReLU}()$ \\
Linear 3     & $\text{Linear}(1024 \rightarrow 1024)$ \\
ReLU 3       & $\text{ReLU}()$ \\
Output       & $\text{Linear}(1024 \rightarrow \text{num\_classes})$ \\
\bottomrule
\end{tabular}
\vspace{-1em}
\end{table}

\paragraph{ConvNet64:} 
The ConvNet64 architecture is a deep convolutional network composed of 8 convolutional layers, batch normalization, ReLU activations, and two max pooling layers. The model operates on 2D image inputs with \texttt{num\_channels} input channels and uses a base width parameter.
\begin{table}[h]
\centering
\caption{Feature extractor architecture of ConvNet64. Width is denoted by $w$.}
\label{tab:convnet64_arch}
\vspace{0.5em}
\begin{tabular}{ll}
\toprule
\textbf{Layer} & \textbf{Description} \\
\midrule
Conv 0        & $\text{Conv2D}(C \rightarrow 1w),\; 3\times3$, padding=1 \\
BN + ReLU     & $\text{BatchNorm + ReLU}$ \\
Conv 1        & $\text{Conv2D}(1w \rightarrow 2w),\; 3\times3$, padding=1 \\
BN + ReLU     & $\text{BatchNorm + ReLU}$ \\
Conv 2        & $\text{Conv2D}(2w \rightarrow 2w),\; 3\times3$, padding=1 \\
BN + ReLU     & $\text{BatchNorm + ReLU}$ \\
Conv 3        & $\text{Conv2D}(2w \rightarrow 4w),\; 3\times3$, padding=1 \\
BN + ReLU     & $\text{BatchNorm + ReLU}$ \\
Conv 4        & $\text{Conv2D}(4w \rightarrow 4w),\; 3\times3$, padding=1 \\
BN + ReLU     & $\text{BatchNorm + ReLU}$ \\
Conv 5        & $\text{Conv2D}(4w \rightarrow 4w),\; 3\times3$, padding=1 \\
BN + ReLU     & $\text{BatchNorm + ReLU}$ \\
MaxPool 0     & $\text{MaxPool2D}(3\times3)$ \\
Conv 6        & $\text{Conv2D}(4w \rightarrow 4w),\; 3\times3$, padding=1 \\
BN + ReLU     & $\text{BatchNorm + ReLU}$ \\
Conv 7        & $\text{Conv2D}(4w \rightarrow 4w),\; 3\times3$, padding=1 \\
BN + ReLU     & $\text{BatchNorm + ReLU}$ \\
MaxPool 1     & $\text{MaxPool2D}(3\times3)$ \\
\bottomrule
\end{tabular}
\end{table}

\section{Additional Reconstruction Results}
\subsection{Reconstruction on Different Model Architectures}\label{appendix:recon_results}
We evaluated \methodname{} across various datasets and model architectures. Figures~\ref{fig:femnist_lenet}, \ref{fig:mnist_MLP}, and \ref{fig:cifar100_resnet} present single-image reconstructions on the FEMNIST dataset using LeNet~\cite{lecun1998mnist}, the MNIST dataset using an MLP, and the CIFAR-100 dataset using ResNet18~\cite{resnet}, respectively. Higher SSIM and PSNR scores, along with lower LPIPS values, indicate greater similarity to the ground truth. We observe that \methodname{} achieves strong reconstructions on simpler models; however, recovering data from gradients becomes increasingly challenging with complex architectures such as ResNet18—an inherited limitation from prior gradient inversion attacks~\cite{invgrad}.

\begin{figure}[h]
    \centering
    \resizebox{0.85\textwidth}{!}{%
    \begin{tabular}{c}
        \textbf{FEMNIST (LeNet)} \\[0.2cm]
        \begin{subfigure}[t]{0.19\textwidth}
            \centering
            \textbf{Ground Truth}\vspace{0.05cm}
            \fbox{\includegraphics[width=\linewidth]{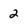}}
        \end{subfigure}
        \hspace{0.05cm}
        \begin{subfigure}[t]{0.19\textwidth}
            \centering
            \textbf{\citet{elasticsga}}
            \fbox{\includegraphics[width=\linewidth]{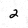}}\\
            {$\uparrow$SSIM: 0.95\\$\uparrow$PSNR: 38.8\\$\downarrow$LPIPS: 0.00}
        \end{subfigure}
        \hspace{0.05cm}
        \begin{subfigure}[t]{0.19\textwidth}
            \centering
            \textbf{ABL~\cite{abl}}
            \fbox{\includegraphics[width=\linewidth]{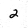}}\\
            {$\uparrow$SSIM: 0.97\\$\uparrow$PSNR: 38.1\\$\downarrow$LPIPS: 0.00}
        \end{subfigure}
        \hspace{0.05cm}
        \begin{subfigure}[t]{0.19\textwidth}
            \centering
            \textbf{\citet{manaar}}
            \fbox{\includegraphics[width=\linewidth]{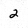}}\\
            {$\uparrow$SSIM: 0.98\\$\uparrow$PSNR: 36.0\\$\downarrow$LPIPS: 0.00}
        \end{subfigure}
        \hspace{0.05cm}
        \begin{subfigure}[t]{0.19\textwidth}
            \centering
            \textbf{\citet{halimi}}
            \fbox{\includegraphics[width=\linewidth]{figures/different_models/femnist_lenet/halimi_1.png}}\\
            {$\uparrow$SSIM: 0.96\\$\uparrow$PSNR: 39.4\\$\downarrow$LPIPS: 0.00}
        \end{subfigure}
    \end{tabular}
    }
    \caption{Single-image reconstructions on FEMNIST using LeNet. We compare \methodname{} reconstructions from local updates of four unlearning methods. Higher SSIM/PSNR and lower LPIPS indicate better visual fidelity.}

    \label{fig:femnist_lenet}
\end{figure}

\begin{figure}[h]
    \centering
    \resizebox{0.85\textwidth}{!}{%
    \begin{tabular}{c}
        \textbf{MNIST (MLP)} \\[0.2cm]
        \begin{subfigure}[t]{0.19\textwidth}
            \centering
            \textbf{Ground Truth}\vspace{0.05cm}
            \fbox{\includegraphics[width=\linewidth]{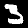}}
        \end{subfigure}
        \hspace{0.05cm}
        \begin{subfigure}[t]{0.19\textwidth}
            \centering
            \textbf{\citet{elasticsga}}
            \fbox{\includegraphics[width=\linewidth]{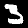}}\\
            {$\uparrow$SSIM: 1.00\\$\uparrow$PSNR: 55.1 \\$\downarrow$LPIPS: 0.00}
        \end{subfigure}
        \hspace{0.05cm}
        \begin{subfigure}[t]{0.19\textwidth}
            \centering
            \textbf{ABL~\cite{abl}}
            \fbox{\includegraphics[width=\linewidth]{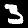}}\\
            {$\uparrow$SSIM: 1.00\\$\uparrow$PSNR: 100.6\\$\downarrow$LPIPS: 0.00}
        \end{subfigure}
        \hspace{0.05cm}
        \begin{subfigure}[t]{0.19\textwidth}
            \centering
            \textbf{\citet{manaar}}
            \fbox{\includegraphics[width=\linewidth]{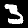}}\\
            {$\uparrow$SSIM: 1.00\\$\uparrow$PSNR: 73.8\\$\downarrow$LPIPS: 0.00}
        \end{subfigure}
        \hspace{0.05cm}
        \begin{subfigure}[t]{0.19\textwidth}
            \centering
            \textbf{\citet{halimi}}
            \fbox{\includegraphics[width=\linewidth]{figures/different_models/mnist_mlp/halimi_1.png}}\\
            {$\uparrow$SSIM: 1.00\\$\uparrow$PSNR: 100.0\\$\downarrow$LPIPS: 0.00}
        \end{subfigure}
    \end{tabular}
    }
   \caption{Single-image reconstructions on MNIST using MLP. We compare \methodname{} reconstructions from local updates of four unlearning methods.}

    \label{fig:mnist_MLP}
\end{figure}

\begin{figure}[h]
    \centering
    \resizebox{0.85\textwidth}{!}{%
    \begin{tabular}{c}
        \textbf{CIFAR100 (ResNet18) Placeholder} \\[0.2cm]
        \begin{subfigure}[t]{0.19\textwidth}
            \centering
            \textbf{Ground Truth}\vspace{0.05cm}
            \fbox{\includegraphics[width=\linewidth]{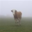}}
        \end{subfigure}
        \hspace{0.05cm}
        \begin{subfigure}[t]{0.19\textwidth}
            \centering
            \textbf{\citet{elasticsga}}
            \fbox{\includegraphics[width=\linewidth]{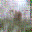}}\\
            {$\uparrow$SSIM: 0.45\\$\uparrow$PSNR: 17.2\\$\downarrow$LPIPS: 0.44}
        \end{subfigure}
        \hspace{0.05cm}
        \begin{subfigure}[t]{0.19\textwidth}
            \centering
            \textbf{ABL~\cite{abl}}
            \fbox{\includegraphics[width=\linewidth]{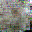}}\\
            {$\uparrow$SSIM: 0.25\\$\uparrow$PSNR: 12.9\\$\downarrow$LPIPS: 0.65}
        \end{subfigure}
        \hspace{0.05cm}
        \begin{subfigure}[t]{0.19\textwidth}
            \centering
            \textbf{\citet{manaar}}
            \fbox{\includegraphics[width=\linewidth]{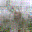}}\\
            {$\uparrow$SSIM: 0.15\\$\uparrow$PSNR: 11.8\\$\downarrow$LPIPS: 0.62}
        \end{subfigure}
        \hspace{0.05cm}
        \begin{subfigure}[t]{0.19\textwidth}
            \centering
            \textbf{\citet{halimi}}
            \fbox{\includegraphics[width=\linewidth]{figures/different_models/cifar100_resnet/halimi_1.png}}\\
            {$\uparrow$SSIM: 0.41\\$\uparrow$PSNR: 18.2\\$\downarrow$LPIPS: 0.46}
        \end{subfigure}
    \end{tabular}
    }
     \caption{Single-image reconstructions on CIFAR100 using ResNet18. We compare \methodname{} reconstructions from local updates of four unlearning methods.}
    \label{fig:cifar100_resnet}
\end{figure}
\subsection{Reconstruction for Different Unlearning Step Size}\label{appendix:ablation_unlearn_steps}
We evaluated \methodname{} under the FedAvg protocol, where the client is allowed to perform more than one local optimization step ($\mathcal{E} > 1$). Figures~\ref{fig:step_2} and~\ref{fig:step_4} show \methodname{} reconstructions on the CIFAR-10 dataset using the ConvNet64 model for $\mathcal{E} = 2$ and $\mathcal{E} = 4$, respectively, across four unlearning algorithms. We observe that reconstruction efficacy deteriorates as the number of local epochs increases.

\begin{figure}[H]
    \centering
    \resizebox{0.85\textwidth}{!}{%
    \begin{tabular}{c}
        \textbf{CIFAR10 ($\mathcal{E} = 2$)} \\[0.2cm]
        \begin{subfigure}[t]{0.19\textwidth}
            \centering
            \textbf{Ground Truth}\vspace{0.05cm}
            \fbox{\includegraphics[width=\linewidth]{figures/single_rec_images/cifar10/gt_0.png}}
        \end{subfigure}
        \hspace{0.05cm}
        \begin{subfigure}[t]{0.19\textwidth}
            \centering
            \textbf{\citet{elasticsga}}
            \fbox{\includegraphics[width=\linewidth]{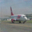}}\\
            {$\uparrow$SSIM: 0.96\\$\uparrow$PSNR: 36.4\\$\downarrow$LPIPS: 0.02}
        \end{subfigure}
        \hspace{0.05cm}
        \begin{subfigure}[t]{0.19\textwidth}
            \centering
            \textbf{ABL~\cite{abl}}
            \fbox{\includegraphics[width=\linewidth]{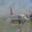}}\\
            {$\uparrow$SSIM: 0.63\\$\uparrow$PSNR: 24.8\\$\downarrow$LPIPS: 0.29}
        \end{subfigure}
        \hspace{0.05cm}
        \begin{subfigure}[t]{0.19\textwidth}
            \centering
            \textbf{\citet{manaar}}
            \fbox{\includegraphics[width=\linewidth]{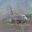}}\\
            {$\uparrow$SSIM: 0.75\\$\uparrow$PSNR: 27.4\\$\downarrow$LPIPS: 0.20}
        \end{subfigure}
        \hspace{0.05cm}
        \begin{subfigure}[t]{0.19\textwidth}
            \centering
            \textbf{\citet{halimi}}
            \fbox{\includegraphics[width=\linewidth]{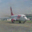}}\\
            {$\uparrow$SSIM: 0.97\\$\uparrow$PSNR: 36.1\\$\downarrow$LPIPS: 0.02}
        \end{subfigure}
    \end{tabular}
    }
   \caption{Reconstruction results from local updates with $\mathcal{E} = 2$ steps on CIFAR-10 using ConvNet64. \methodname{} is applied to four unlearning methods. Higher SSIM/PSNR and lower LPIPS indicate better reconstruction quality.}

    \label{fig:step_2}
\end{figure}

\begin{figure}[H]
    \centering
    \resizebox{0.85\textwidth}{!}{%
    \begin{tabular}{c}
        \textbf{CIFAR10 ($\mathcal{E} = 4$)} \\[0.2cm]
        \begin{subfigure}[t]{0.19\textwidth}
            \centering
            \textbf{Ground Truth}\vspace{0.05cm}
            \fbox{\includegraphics[width=\linewidth]{figures/single_rec_images/cifar10/gt_0.png}}
        \end{subfigure}
        \hspace{0.05cm}
        \begin{subfigure}[t]{0.19\textwidth}
            \centering
            \textbf{\citet{elasticsga}}
            \fbox{\includegraphics[width=\linewidth]{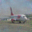}}\\
            {$\uparrow$SSIM: 0.91\\$\uparrow$PSNR: 32.8\\$\downarrow$LPIPS: 0.05}
        \end{subfigure}
        \hspace{0.05cm}
        \begin{subfigure}[t]{0.19\textwidth}
            \centering
            \textbf{ABL~\cite{abl}}
            \fbox{\includegraphics[width=\linewidth]{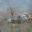}}\\
            {$\uparrow$SSIM: 0.71\\$\uparrow$PSNR: 25.6\\$\downarrow$LPIPS: 0.21}
        \end{subfigure}
        \hspace{0.05cm}
        \begin{subfigure}[t]{0.19\textwidth}
            \centering
            \textbf{\citet{manaar}}
            \fbox{\includegraphics[width=\linewidth]{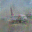}}\\
            {$\uparrow$SSIM: 0.73\\$\uparrow$PSNR: 26.6\\$\downarrow$LPIPS: 0.20}
        \end{subfigure}
        \hspace{0.05cm}
        \begin{subfigure}[t]{0.19\textwidth}
            \centering
            \textbf{\citet{halimi}}
            \fbox{\includegraphics[width=\linewidth]{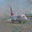}}\\
            {$\uparrow$SSIM: 0.83\\$\uparrow$PSNR: 29.0\\$\downarrow$LPIPS: 0.15}
        \end{subfigure}
    \end{tabular}
    }
    \caption{Reconstruction results from local updates with $\mathcal{E} = 4$ steps on CIFAR-10 using ConvNet64. \methodname{} is applied to four unlearning methods. Higher SSIM/PSNR and lower LPIPS indicate better reconstruction quality.}
    \label{fig:step_4}
\end{figure}

\subsection{Reconstructed Images for CIFAR100, FEMNIST and MNIST}\label{appendix:reconstructed_images}
Figures~\ref{fig:cifar100_single},~\ref{fig:femnist_single}, and~\ref{fig:mnist_single} demonstrate the capability of \methodname{} to reconstruct images across different datasets and unlearning algorithms using the ConvNet64 model.

\begin{figure}[H]
    \centering
    \resizebox{0.85\textwidth}{!}{%
    \begin{tabular}{c}
        \textbf{CIFAR100} \\[0.2cm]
        \begin{subfigure}[t]{0.19\textwidth}
            \centering
            \textbf{Ground Truth}\vspace{0.05cm}
            \fbox{\includegraphics[width=\linewidth]{figures/single_rec_images/cifar100/gt_0.png}}
        \end{subfigure}
        \hspace{0.05cm}
        \begin{subfigure}[t]{0.19\textwidth}
            \centering
            \textbf{\citet{elasticsga}}
            \fbox{\includegraphics[width=\linewidth]{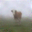}}\\
            {$\uparrow$SSIM: 0.94\\$\uparrow$PSNR: 33.7\\$\downarrow$LPIPS: 0.04}
        \end{subfigure}
        \hspace{0.05cm}
        \begin{subfigure}[t]{0.19\textwidth}
            \centering
            \textbf{ABL~\cite{abl}}
            \fbox{\includegraphics[width=\linewidth]{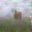}}\\
            {$\uparrow$SSIM: 0.77\\$\uparrow$PSNR: 24.5\\$\downarrow$LPIPS: 0.22}
        \end{subfigure}
        \hspace{0.05cm}
        \begin{subfigure}[t]{0.19\textwidth}
            \centering
            \textbf{\citet{manaar}}
            \fbox{\includegraphics[width=\linewidth]{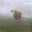}}\\
            {$\uparrow$SSIM: 0.86\\$\uparrow$PSNR: 26.9\\$\downarrow$LPIPS: 0.09}
        \end{subfigure}
        \hspace{0.05cm}
        \begin{subfigure}[t]{0.19\textwidth}
            \centering
            \textbf{\citet{halimi}}
            \fbox{\includegraphics[width=\linewidth]{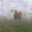}}\\
            {$\uparrow$SSIM: 0.82\\$\uparrow$PSNR: 25.9\\$\downarrow$LPIPS: 0.15}
        \end{subfigure}
    \end{tabular}
    }
      \caption{\textit{\methodname{}} reconstructions from local updates of four unlearning methods on CIFAR100. }

    \label{fig:cifar100_single}
\end{figure}

\begin{figure}[H]
    \centering
    \resizebox{0.85\textwidth}{!}{%
    \begin{tabular}{c}
        \textbf{FEMNIST} \\[0.2cm]
        \begin{subfigure}[t]{0.19\textwidth}
            \centering
            \textbf{Ground Truth}\vspace{0.05cm}
            \fbox{\includegraphics[width=\linewidth]{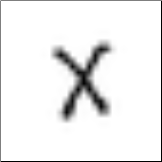}}
        \end{subfigure}
        \hspace{0.05cm}
        \begin{subfigure}[t]{0.19\textwidth}
            \centering
            \textbf{\citet{elasticsga}}
            \fbox{\includegraphics[width=\linewidth]{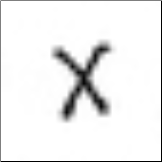}}\\
            {$\uparrow$SSIM: 1.00\\$\uparrow$PSNR: 45.4\\$\downarrow$LPIPS: 0.00}
        \end{subfigure}
        \hspace{0.05cm}
        \begin{subfigure}[t]{0.19\textwidth}
            \centering
            \textbf{ABL~\cite{abl}}
            \fbox{\includegraphics[width=\linewidth]{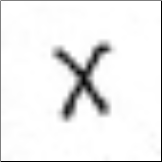}}\\
            {$\uparrow$SSIM: 0.99\\$\uparrow$PSNR: 38.1\\$\downarrow$LPIPS: 0.00}
        \end{subfigure}
        \hspace{0.05cm}
        \begin{subfigure}[t]{0.19\textwidth}
            \centering
            \textbf{\citet{manaar}}
            \fbox{\includegraphics[width=\linewidth]{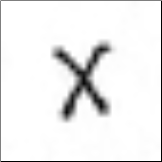}}\\
            {$\uparrow$SSIM: 0.99\\$\uparrow$PSNR: 40.8\\$\downarrow$LPIPS: 0.00}
        \end{subfigure}
        \hspace{0.05cm}
        \begin{subfigure}[t]{0.19\textwidth}
            \centering
            \textbf{\citet{halimi}}
            \fbox{\includegraphics[width=\linewidth]{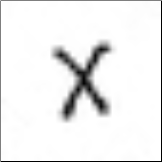}}\\
            {$\uparrow$SSIM: 0.99\\$\uparrow$PSNR: 40.8\\$\downarrow$LPIPS: 0.00}
        \end{subfigure}
    \end{tabular}
    }
      \caption{\textit{\methodname} reconstructions from local updates of four unlearning methods on FEMNIST. }
    \label{fig:femnist_single}
\end{figure}

\begin{figure}[H]
    \centering
    \resizebox{0.85\textwidth}{!}{%
    \begin{tabular}{c}
        \textbf{MNIST} \\[0.2cm]
        \begin{subfigure}[t]{0.19\textwidth}
            \centering
            \textbf{Ground Truth}\vspace{0.05cm}
            \fbox{\includegraphics[width=\linewidth]{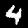}}
        \end{subfigure}
        \hspace{0.05cm}
        \begin{subfigure}[t]{0.19\textwidth}
            \centering
            \textbf{\citet{elasticsga}}
            \fbox{\includegraphics[width=\linewidth]{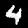}}\\
            {$\uparrow$SSIM: 0.98\\$\uparrow$PSNR: 38.7\\$\downarrow$LPIPS: 0.00}
        \end{subfigure}
        \hspace{0.05cm}
        \begin{subfigure}[t]{0.19\textwidth}
            \centering
            \textbf{ABL~\cite{abl}}
            \fbox{\includegraphics[width=\linewidth]{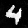}}\\
            {$\uparrow$SSIM: 0.88\\$\uparrow$PSNR: 30.02\\$\downarrow$LPIPS: 0.02}
        \end{subfigure}
        \hspace{0.05cm}
        \begin{subfigure}[t]{0.19\textwidth}
            \centering
            \textbf{\citet{manaar}}
            \fbox{\includegraphics[width=\linewidth]{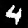}}\\
            {$\uparrow$SSIM: 0.96\\$\uparrow$PSNR: 39.6\\$\downarrow$LPIPS: 0.00}
        \end{subfigure}
        \hspace{0.05cm}
        \begin{subfigure}[t]{0.19\textwidth}
            \centering
            \textbf{\citet{halimi}}
            \fbox{\includegraphics[width=\linewidth]{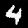}}\\
            {$\uparrow$SSIM: 0.97\\$\uparrow$PSNR: 40.1\\$\downarrow$LPIPS: 0.00}
        \end{subfigure}
    \end{tabular}
    }
    \caption{\textit{\methodname} reconstructions from local updates of four unlearning methods on MNIST. }
    \label{fig:mnist_single}
\end{figure}
\subsection{Batch Reconstructions for Datasets and Batch Sizes}\label{appendix:grid_rec}

We evaluate \methodname{} by reconstructing images from local model updates produced by two unlearning methods—ABL~\cite{abl} and \citet{halimi}—across four datasets: CIFAR-10, CIFAR-100, FEMNIST, and MNIST. Figures~\ref{fig:batch_cifar10_halimi} to~\ref{fig:batch_mnist_halimi} show how reconstruction quality degrades as the batch size of removed samples increases from 4 to 128. While these figures focus on visual results from two representative methods, the accompanying metric plots (Figures~\ref{fig:metric_figs_cfiar10} to~\ref{fig:metric_figs_mnist}) provide a broader comparison across all evaluated algorithms. They show consistent trends—declining SSIM and PSNR, and increasing LPIPS—highlighting the growing difficulty of unlearning larger forget sets.

\begin{figure}[H]
    \centering
    \begin{tabular}{cccccc}
        \fbox{\includegraphics[width=0.14\textwidth, height=3.3cm, keepaspectratio]{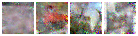}} &
        \fbox{\includegraphics[width=0.14\textwidth, height=3.3cm, keepaspectratio]{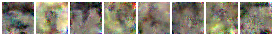}} &
        \fbox{\includegraphics[width=0.14\textwidth, height=3.3cm, keepaspectratio]{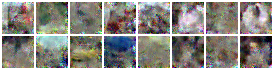}} &
        \fbox{\includegraphics[width=0.14\textwidth, height=3.3cm, keepaspectratio]{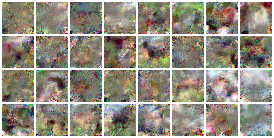}} &
        \fbox{\includegraphics[width=0.14\textwidth, height=3.3cm, keepaspectratio]{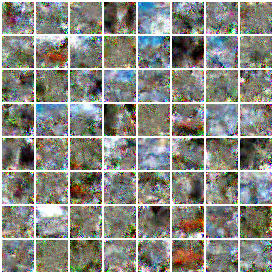}} &
        \fbox{\includegraphics[width=0.14\textwidth, height=3.3cm, keepaspectratio]{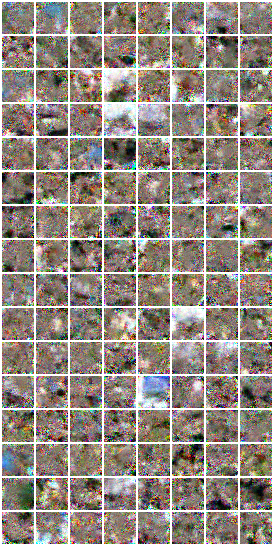}} \\
        \textbf{4} & \textbf{8} & \textbf{16} & \textbf{32} & \textbf{64} & \textbf{128}
    \end{tabular}
    \caption{Multi image reconstruction from \citet{halimi} on CIFAR10.}
    \label{fig:batch_cifar10_halimi}
\end{figure}

\begin{figure}[H]
    \centering
    \begin{tabular}{cccccc}
        \fbox{\includegraphics[width=0.14\textwidth, height=3.3cm, keepaspectratio]{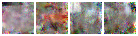}} &
        \fbox{\includegraphics[width=0.14\textwidth, height=3.3cm, keepaspectratio]{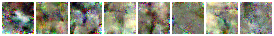}} &
        \fbox{\includegraphics[width=0.14\textwidth, height=3.3cm, keepaspectratio]{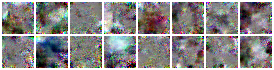}} &
        \fbox{\includegraphics[width=0.14\textwidth, height=3.3cm, keepaspectratio]{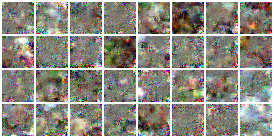}} &
        \fbox{\includegraphics[width=0.14\textwidth, height=3.3cm, keepaspectratio]{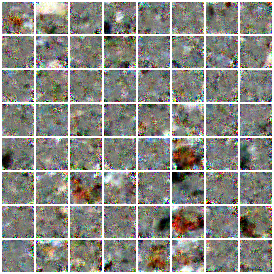}} &
        \fbox{\includegraphics[width=0.14\textwidth, height=3.3cm, keepaspectratio]{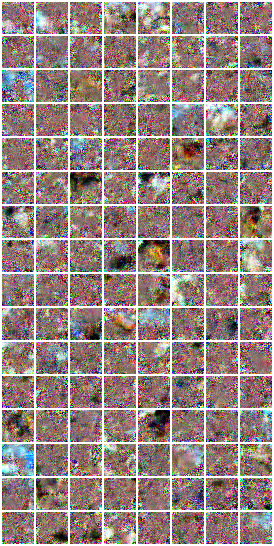}} \\
        \textbf{4} & \textbf{8} & \textbf{16} & \textbf{32} & \textbf{64} & \textbf{128}
    \end{tabular}
    \caption{Multi image reconstruction from ABL~\cite{abl} on CIFAR10.}
    \label{fig:batch_cifar10_abl}
\end{figure}

\begin{figure}[H]
    \centering
    \resizebox{0.85\textwidth}{!}{%
    \begin{tabular}{c}
        \textbf{CIFAR10} \\[0.2cm]
        \begin{subfigure}[t]{0.19\textwidth}
            \centering
            \textbf{SSIM}
            \fbox{\includegraphics[width=\linewidth]{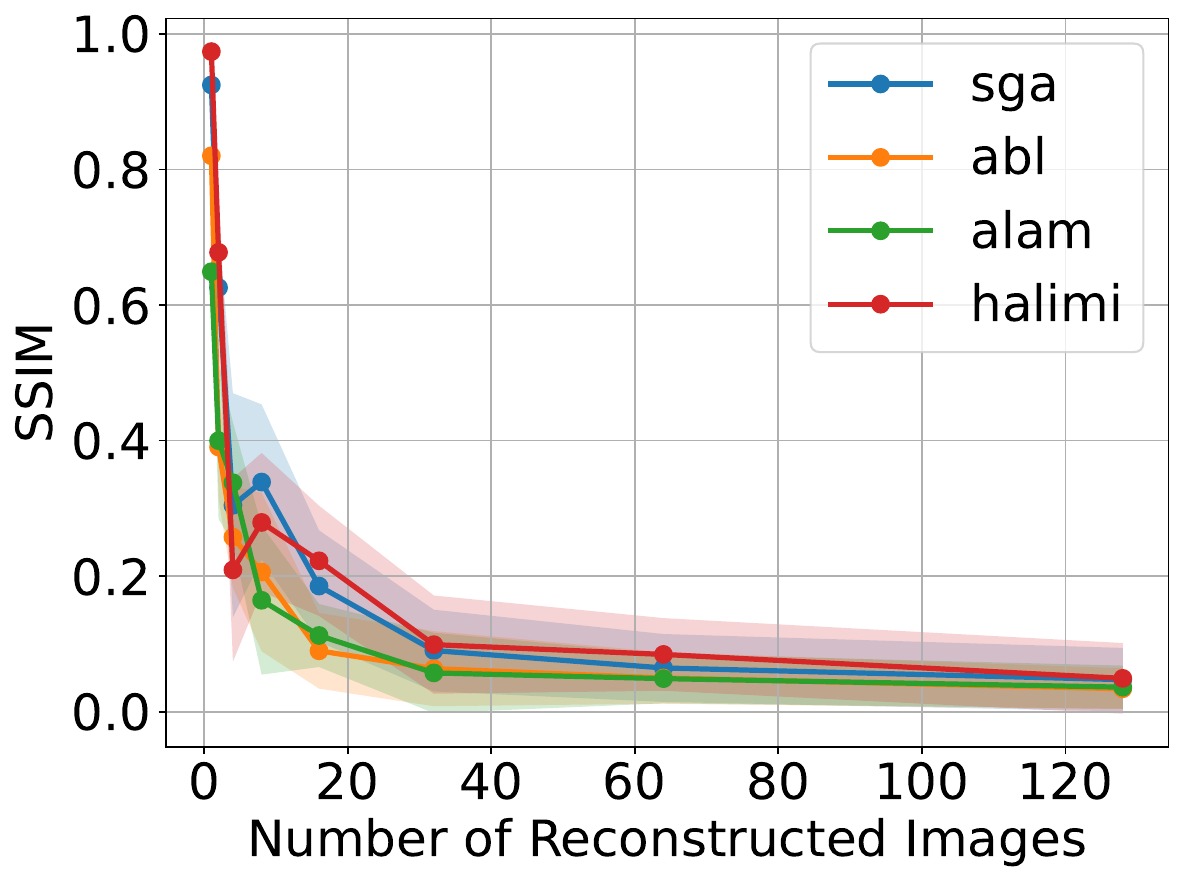}}
        \end{subfigure}
        \hspace{0.05cm}
        \begin{subfigure}[t]{0.19\textwidth}
            \centering
            \textbf{PSNR}
            \fbox{\includegraphics[width=\linewidth]{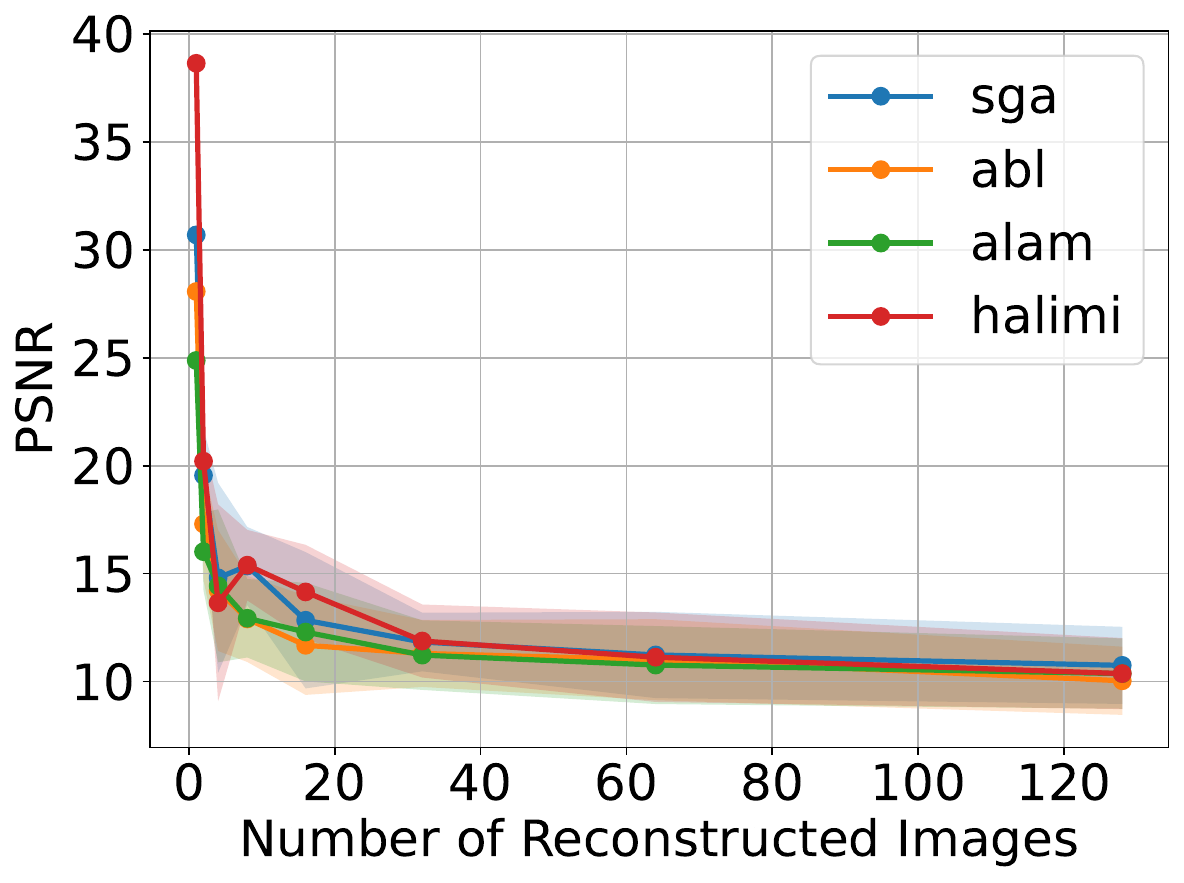}}\\
           
        \end{subfigure}
        \hspace{0.05cm}
        \begin{subfigure}[t]{0.19\textwidth}
            \centering
            \textbf{LPIPS}
            \fbox{\includegraphics[width=\linewidth]{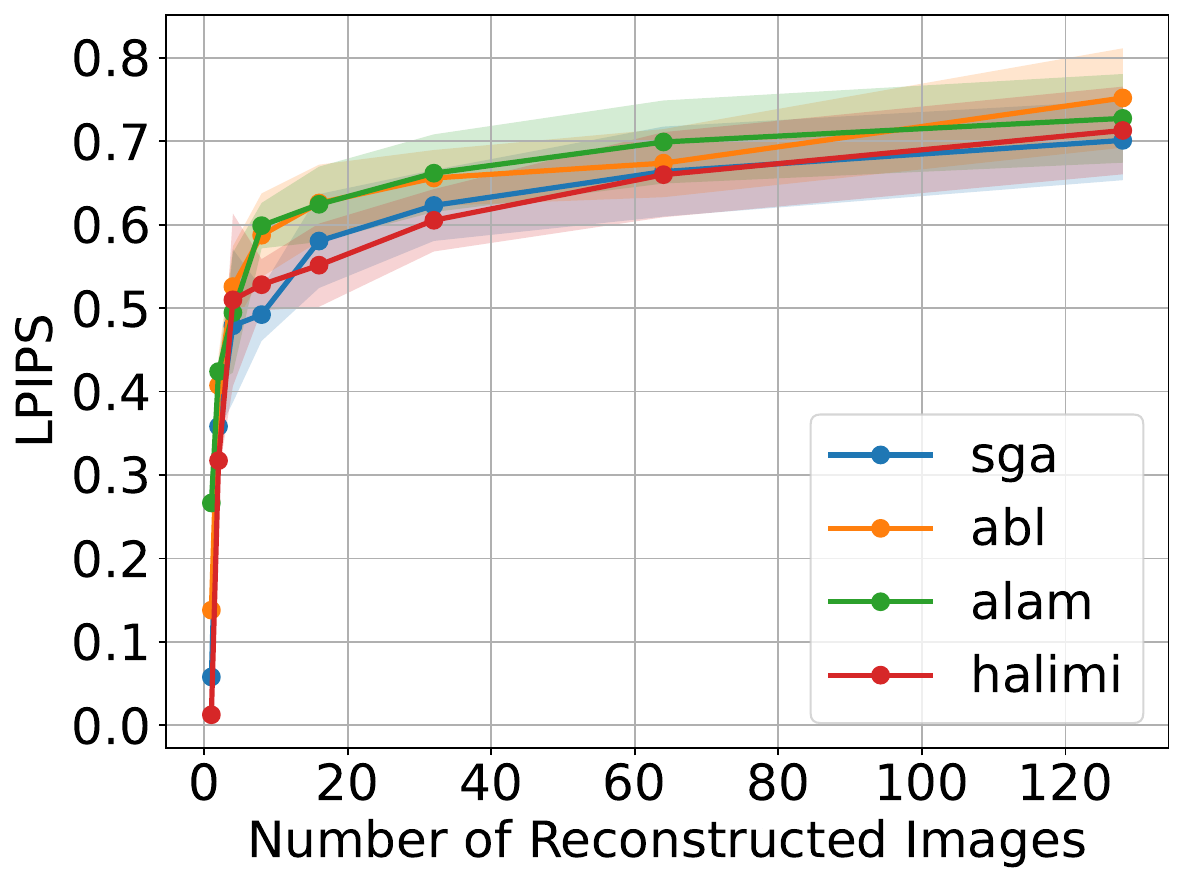}}\\
            
        \end{subfigure}
      
    \end{tabular}
    }
   \caption{Effect of batch size scaling on reconstruction quality for the CIFAR10 dataset on four unlearning algorithms, measured using SSIM, PSNR, and LPIPS. Larger batch sizes tend to degrade reconstruction fidelity, as indicated by decreasing SSIM/PSNR and increasing LPIPS.}
    \label{fig:metric_figs_cfiar10}
\end{figure}

\begin{figure}[H]
    \centering
    \begin{tabular}{cccccc}
        \fbox{\includegraphics[width=0.14\textwidth, height=3.3cm, keepaspectratio]{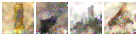}} &
        \fbox{\includegraphics[width=0.14\textwidth, height=3.3cm, keepaspectratio]{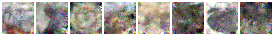}} &
        \fbox{\includegraphics[width=0.14\textwidth, height=3.3cm, keepaspectratio]{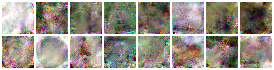}} &
        \fbox{\includegraphics[width=0.14\textwidth, height=3.3cm, keepaspectratio]{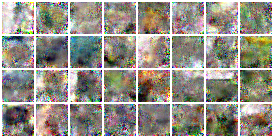}} &
        \fbox{\includegraphics[width=0.14\textwidth, height=3.3cm, keepaspectratio]{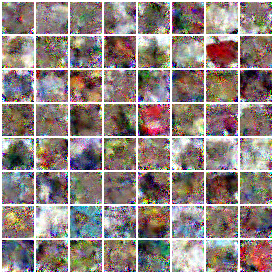}} &
        \fbox{\includegraphics[width=0.14\textwidth, height=3.3cm, keepaspectratio]{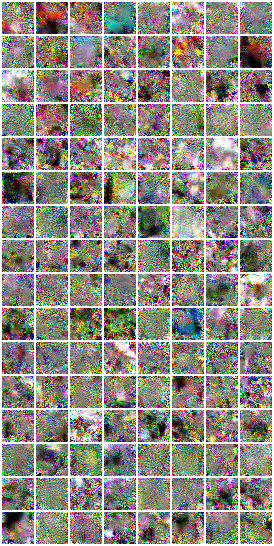}} \\
        \textbf{4} & \textbf{8} & \textbf{16} & \textbf{32} & \textbf{64} & \textbf{128}
    \end{tabular}
    \caption{Multi image reconstruction from \citet{halimi} on CIFAR100.}
    \label{fig:batch_cifar100_halimi}
\end{figure}

\begin{figure}[H]
    \centering
    \begin{tabular}{cccccc}
        \fbox{\includegraphics[width=0.14\textwidth, height=3.3cm, keepaspectratio]{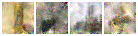}} &
        \fbox{\includegraphics[width=0.14\textwidth, height=3.3cm, keepaspectratio]{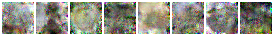}} &
        \fbox{\includegraphics[width=0.14\textwidth, height=3.3cm, keepaspectratio]{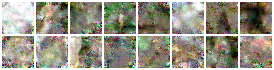}} &
        \fbox{\includegraphics[width=0.14\textwidth, height=3.3cm, keepaspectratio]{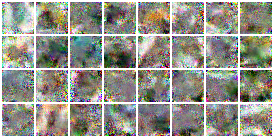}} &
        \fbox{\includegraphics[width=0.14\textwidth, height=3.3cm, keepaspectratio]{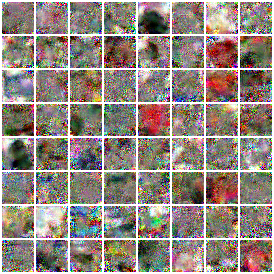}} &
        \fbox{\includegraphics[width=0.14\textwidth, height=3.3cm, keepaspectratio]{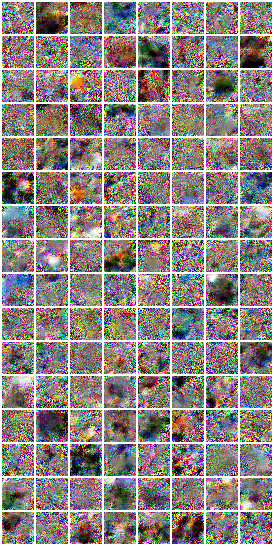}} \\
        \textbf{4} & \textbf{8} & \textbf{16} & \textbf{32} & \textbf{64} & \textbf{128}
    \end{tabular}
    \caption{Multi image reconstruction from ABL~\cite{abl} on CIFAR100.}
    \label{fig:batch_cifar100_abl}
\end{figure}

\begin{figure}[H]
    \centering
    \resizebox{0.85\textwidth}{!}{%
    \begin{tabular}{c}
        \textbf{CIFAR100} \\[0.2cm]
        \begin{subfigure}[t]{0.19\textwidth}
            \centering
            \textbf{SSIM}
            \fbox{\includegraphics[width=\linewidth]{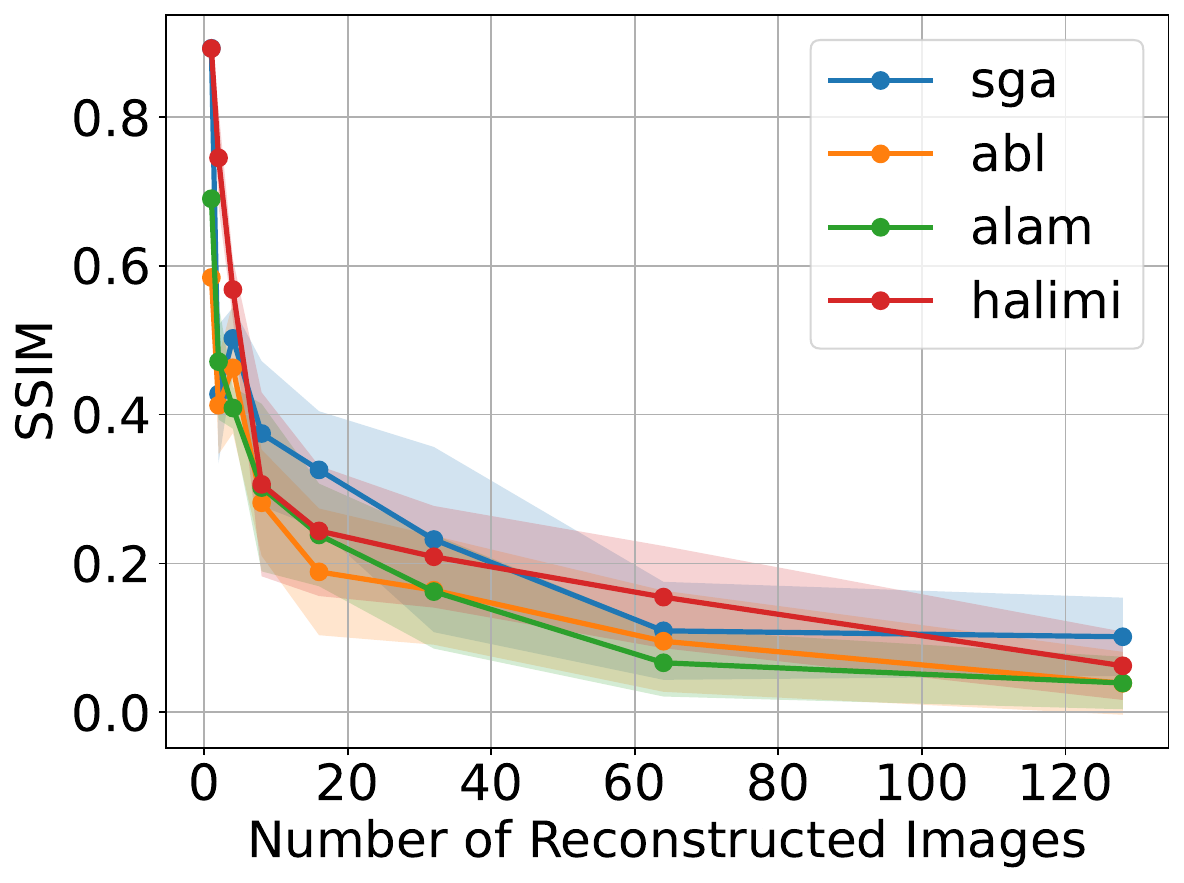}}
        \end{subfigure}
        \hspace{0.05cm}
        \begin{subfigure}[t]{0.19\textwidth}
            \centering
            \textbf{PSNR}
            \fbox{\includegraphics[width=\linewidth]{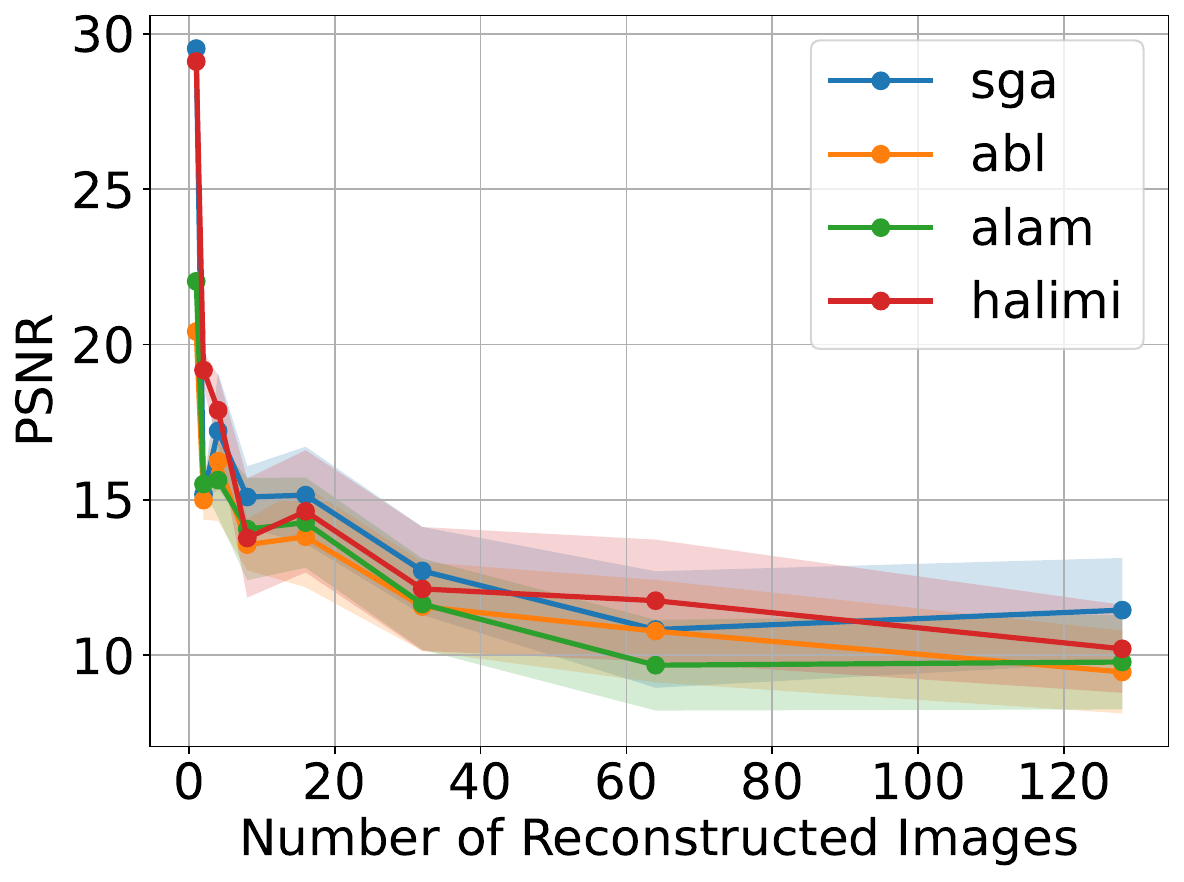}}\\
           
        \end{subfigure}
        \hspace{0.05cm}
        \begin{subfigure}[t]{0.19\textwidth}
            \centering
            \textbf{LPIPS}
            \fbox{\includegraphics[width=\linewidth]{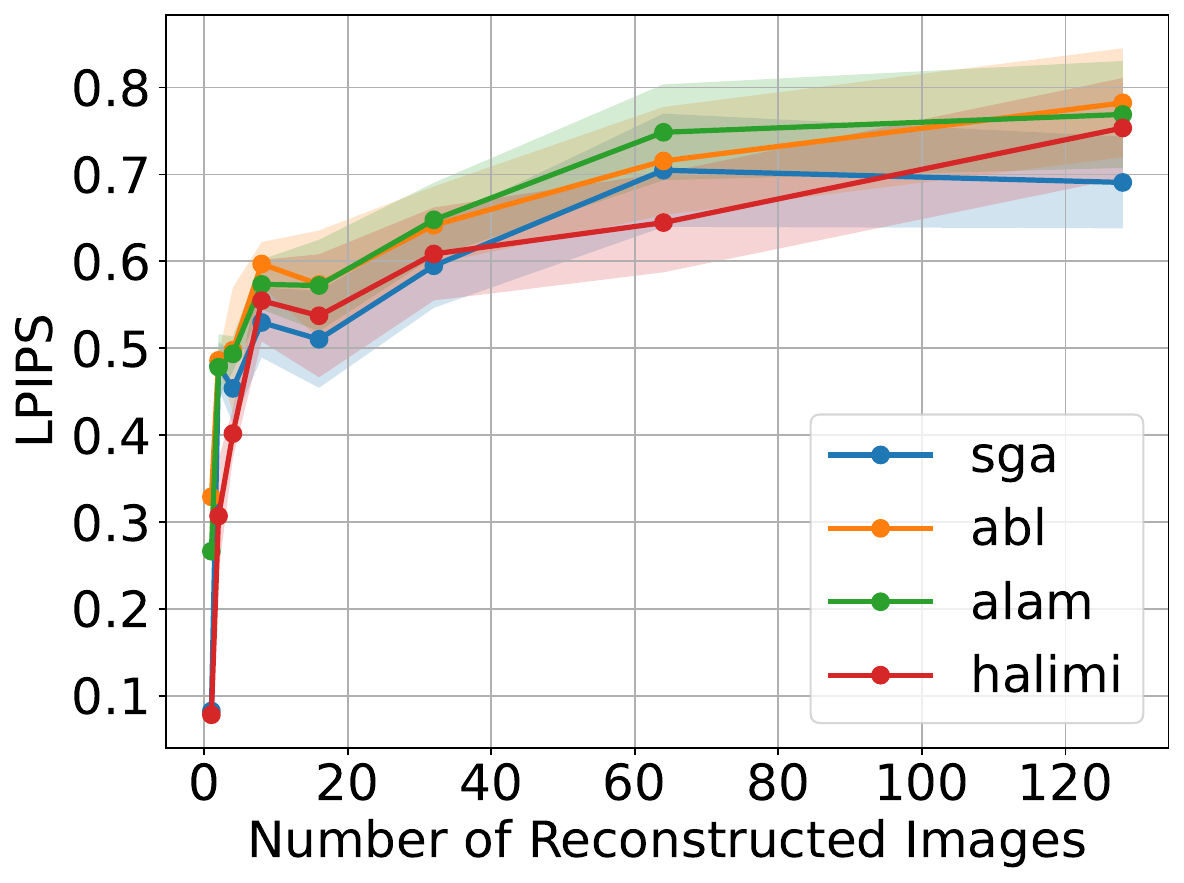}}\\
            
        \end{subfigure}
      
    \end{tabular}
    }
     \caption{Effect of batch size scaling on reconstruction quality for the CIFAR100 dataset on four unlearning algorithms, measured using SSIM, PSNR, and LPIPS. Larger batch sizes tend to degrade reconstruction fidelity, as indicated by decreasing SSIM/PSNR and increasing LPIPS.}
    \label{fig:metric_figs_cfiar100}
\end{figure}

\begin{figure}[H]
    \centering
    \begin{tabular}{cccccc}
        \fbox{\includegraphics[width=0.14\textwidth, height=3.3cm, keepaspectratio]{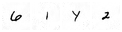}} &
        \fbox{\includegraphics[width=0.14\textwidth, height=3.3cm, keepaspectratio]{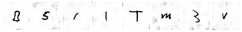}} &
        \fbox{\includegraphics[width=0.14\textwidth, height=3.3cm, keepaspectratio]{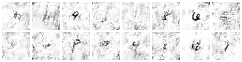}} &
        \fbox{\includegraphics[width=0.14\textwidth, height=3.3cm, keepaspectratio]{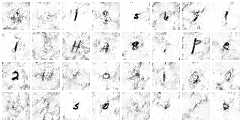}} &
        \fbox{\includegraphics[width=0.14\textwidth, height=3.3cm, keepaspectratio]{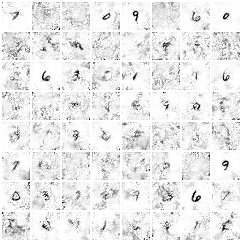}} &
        \fbox{\includegraphics[width=0.14\textwidth, height=3.3cm, keepaspectratio]{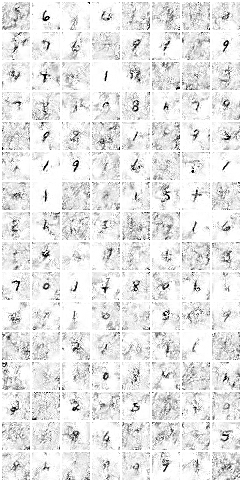}} \\
        \textbf{4} & \textbf{8} & \textbf{16} & \textbf{32} & \textbf{64} & \textbf{128}
    \end{tabular}
    \caption{Multi image reconstruction from \citet{halimi} on FEMNIST.}
    \label{fig:batch_femnist_halimi}
\end{figure}

\begin{figure}[H]
    \centering
    \begin{tabular}{cccccc}
        \fbox{\includegraphics[width=0.14\textwidth, height=3.3cm, keepaspectratio]{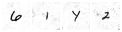}} &
        \fbox{\includegraphics[width=0.14\textwidth, height=3.3cm, keepaspectratio]{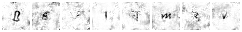}} &
        \fbox{\includegraphics[width=0.14\textwidth, height=3.3cm, keepaspectratio]{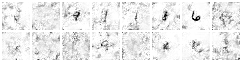}} &
        \fbox{\includegraphics[width=0.14\textwidth, height=3.3cm, keepaspectratio]{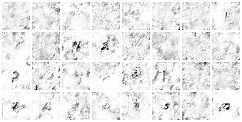}} &
        \fbox{\includegraphics[width=0.14\textwidth, height=3.3cm, keepaspectratio]{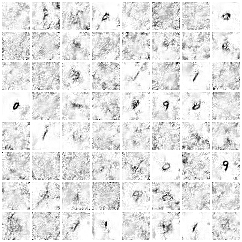}} &
        \fbox{\includegraphics[width=0.14\textwidth, height=3.3cm, keepaspectratio]{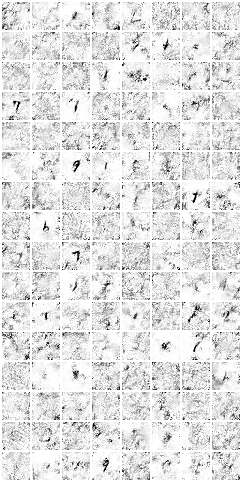}} \\
        \textbf{4} & \textbf{8} & \textbf{16} & \textbf{32} & \textbf{64} & \textbf{128}
    \end{tabular}
    \caption{Multi image reconstruction from ABL~\cite{abl} on FEMNIST.}
    \label{fig:batch_femnist_abl}
\end{figure}
\begin{figure}[H]
    \centering
    \resizebox{0.85\textwidth}{!}{%
    \begin{tabular}{c}
        \textbf{FEMNIST} \\[0.2cm]
        \begin{subfigure}[t]{0.19\textwidth}
            \centering
            \textbf{SSIM}
            \fbox{\includegraphics[width=\linewidth]{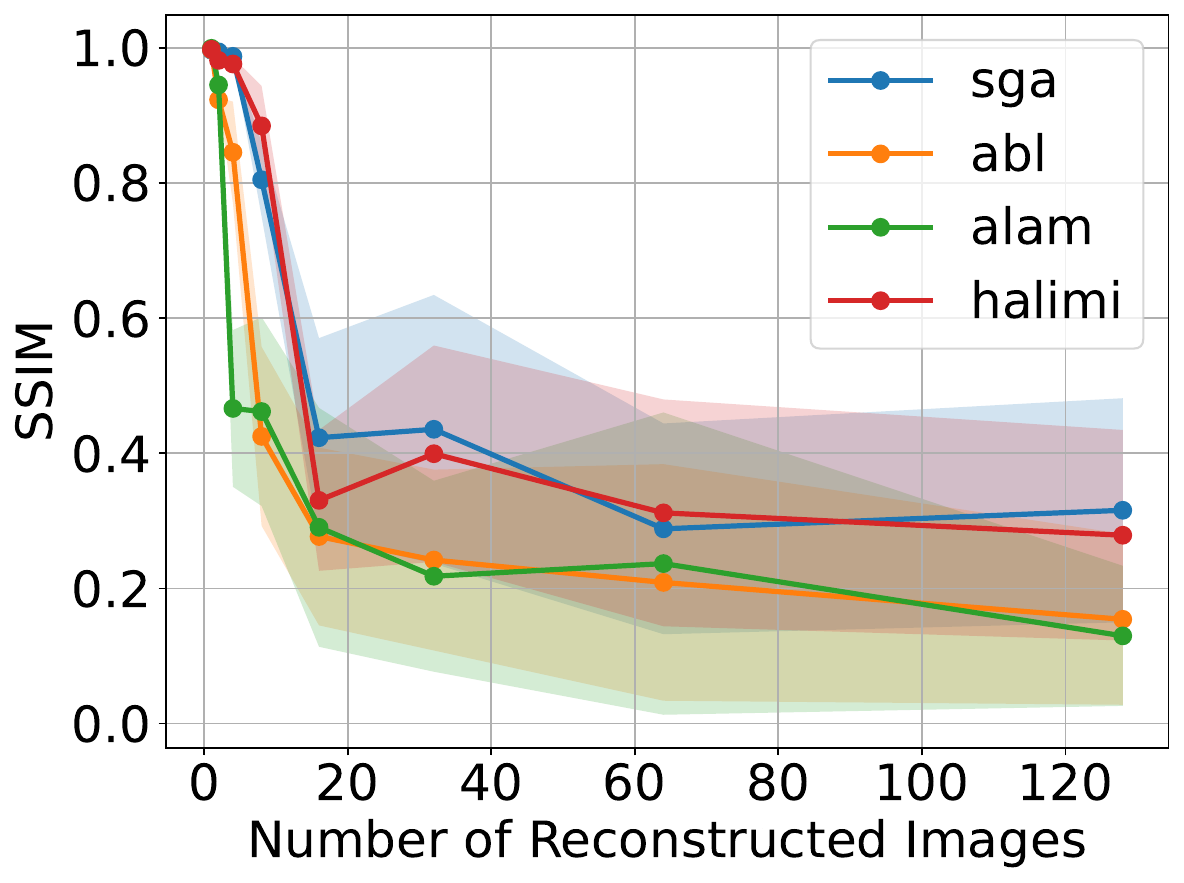}}
        \end{subfigure}
        \hspace{0.05cm}
        \begin{subfigure}[t]{0.19\textwidth}
            \centering
            \textbf{PSNR}
            \fbox{\includegraphics[width=\linewidth]{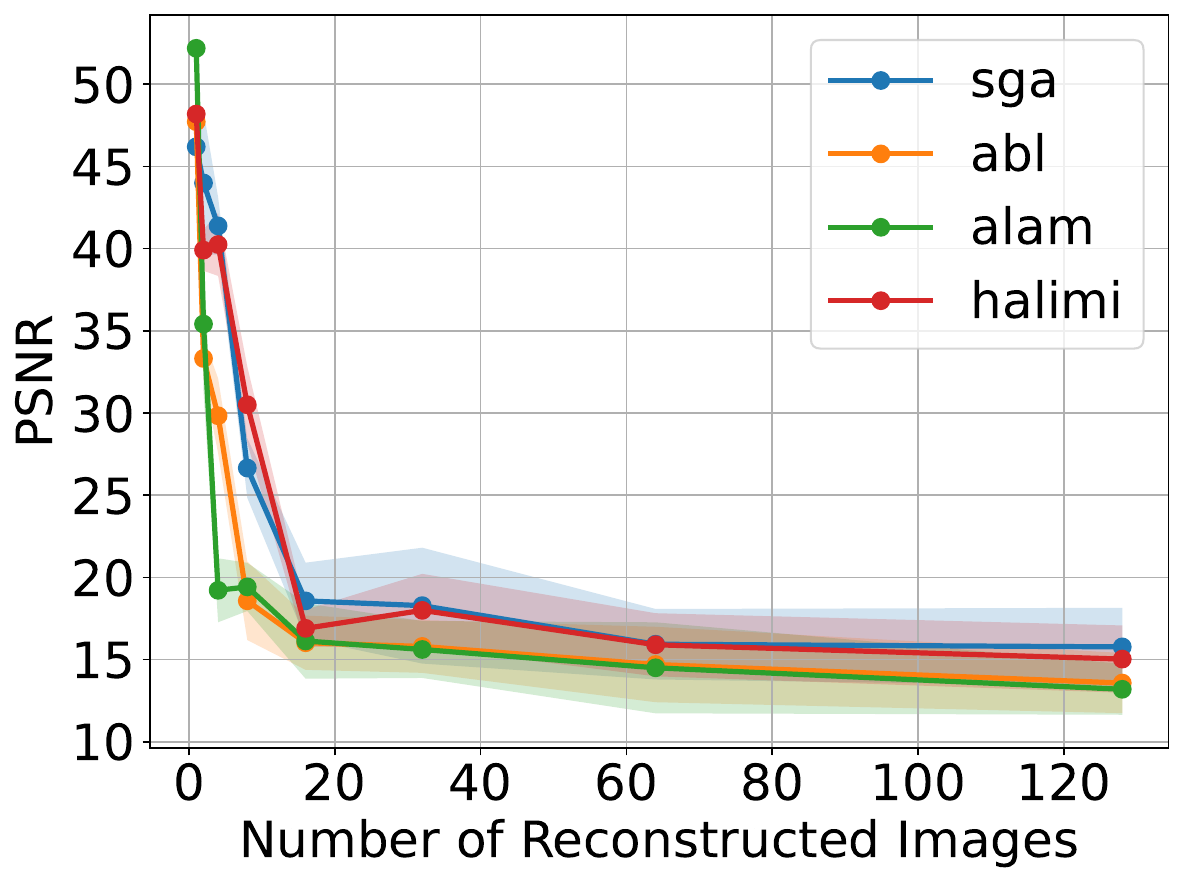}}\\
           
        \end{subfigure}
        \hspace{0.05cm}
        \begin{subfigure}[t]{0.19\textwidth}
            \centering
            \textbf{LPIPS}
            \fbox{\includegraphics[width=\linewidth]{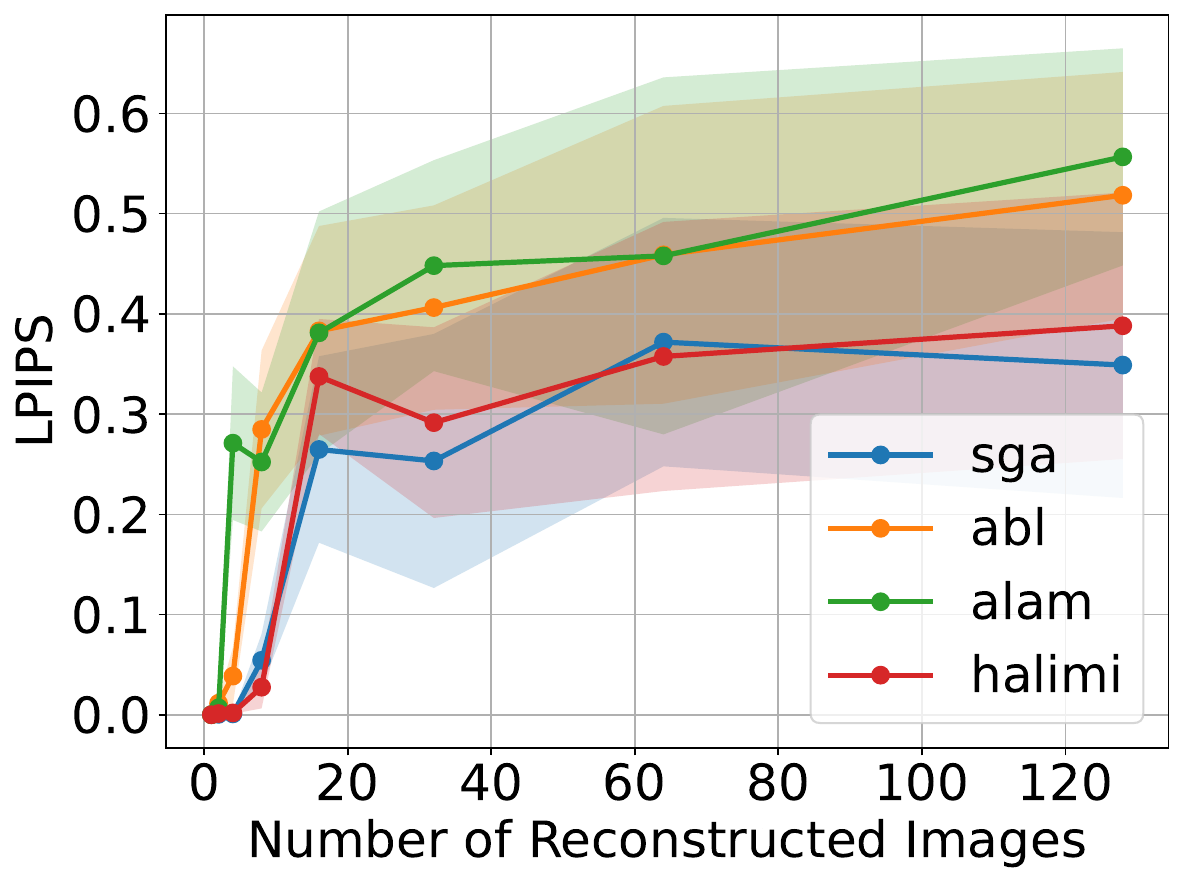}}\\
            
        \end{subfigure}
      
    \end{tabular}
    }
      \caption{Effect of batch size scaling on reconstruction quality for the FEMNIST dataset on four unlearning algorithms, measured using SSIM, PSNR, and LPIPS. Larger batch sizes tend to degrade reconstruction fidelity, as indicated by decreasing SSIM/PSNR and increasing LPIPS.}
    \label{fig:metric_figs_femnist}
\end{figure}

\begin{figure}[H]
    \centering
    \begin{tabular}{cccccc}
        \fbox{\includegraphics[width=0.14\textwidth, height=3.3cm, keepaspectratio]{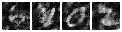}} &
        \fbox{\includegraphics[width=0.14\textwidth, height=3.3cm, keepaspectratio]{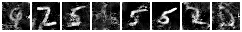}} &
        \fbox{\includegraphics[width=0.14\textwidth, height=3.3cm, keepaspectratio]{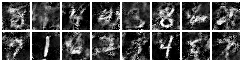}} &
        \fbox{\includegraphics[width=0.14\textwidth, height=3.3cm, keepaspectratio]{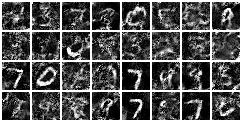}} &
        \fbox{\includegraphics[width=0.14\textwidth, height=3.3cm, keepaspectratio]{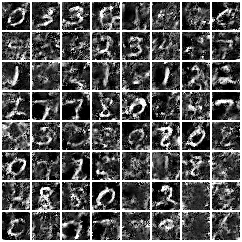}} &
        \fbox{\includegraphics[width=0.14\textwidth, height=3.3cm, keepaspectratio]{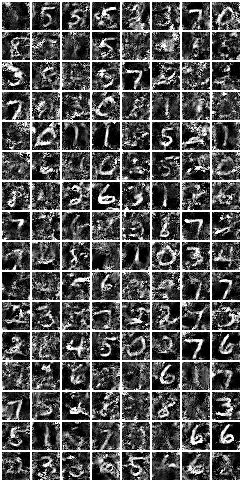}} \\
        \textbf{4} & \textbf{8} & \textbf{16} & \textbf{32} & \textbf{64} & \textbf{128}
    \end{tabular}
    \caption{Multi image reconstruction from \citet{halimi} on MNIST.}
    \label{fig:batch_mnist_halimi}
\end{figure}

\begin{figure}[H]
    \centering
    \begin{tabular}{cccccc}
        \fbox{\includegraphics[width=0.14\textwidth, height=3.3cm, keepaspectratio]{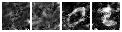}} &
        \fbox{\includegraphics[width=0.14\textwidth, height=3.3cm, keepaspectratio]{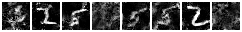}} &
        \fbox{\includegraphics[width=0.14\textwidth, height=3.3cm, keepaspectratio]{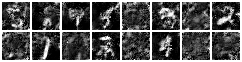}} &
        \fbox{\includegraphics[width=0.14\textwidth, height=3.3cm, keepaspectratio]{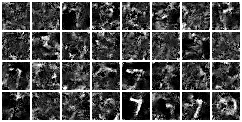}} &
        \fbox{\includegraphics[width=0.14\textwidth, height=3.3cm, keepaspectratio]{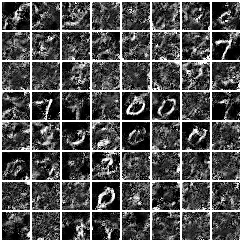}} &
        \fbox{\includegraphics[width=0.14\textwidth, height=3.3cm, keepaspectratio]{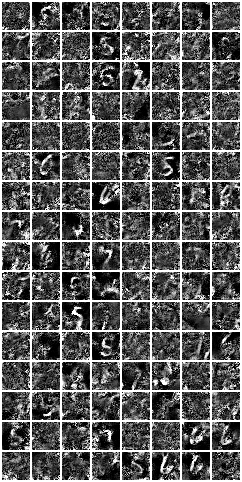}} \\
        \textbf{4} & \textbf{8} & \textbf{16} & \textbf{32} & \textbf{64} & \textbf{128}
    \end{tabular}
    \caption{Multi image reconstruction from ABL~\cite{abl} on MNIST.}
    \label{fig:batch_mnist_abl}
\end{figure}

\begin{figure}[H]
    \centering
    \resizebox{0.85\textwidth}{!}{%
    \begin{tabular}{c}
        \textbf{MNIST} \\[0.2cm]
        \begin{subfigure}[t]{0.19\textwidth}
            \centering
            \textbf{SSIM}
            \fbox{\includegraphics[width=\linewidth]{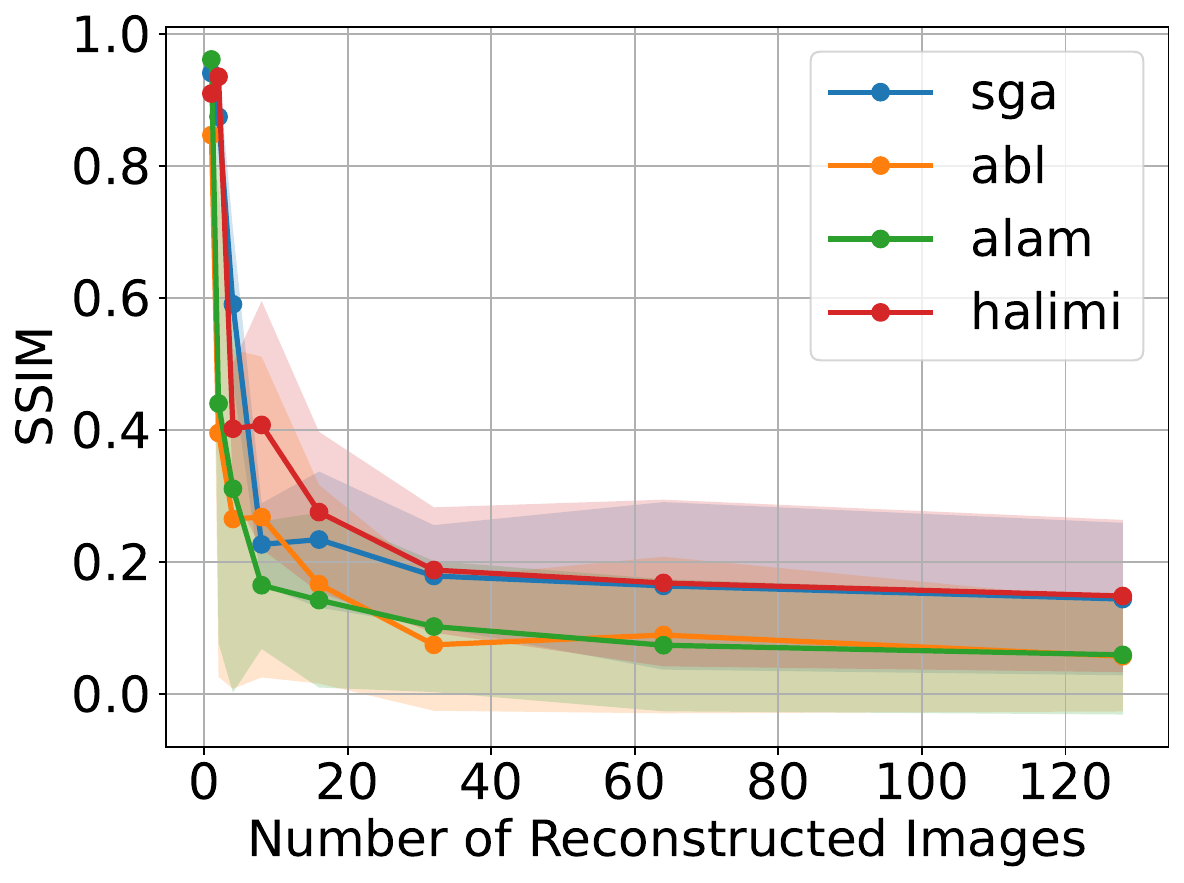}}
        \end{subfigure}
        \hspace{0.05cm}
        \begin{subfigure}[t]{0.19\textwidth}
            \centering
            \textbf{PSNR}
            \fbox{\includegraphics[width=\linewidth]{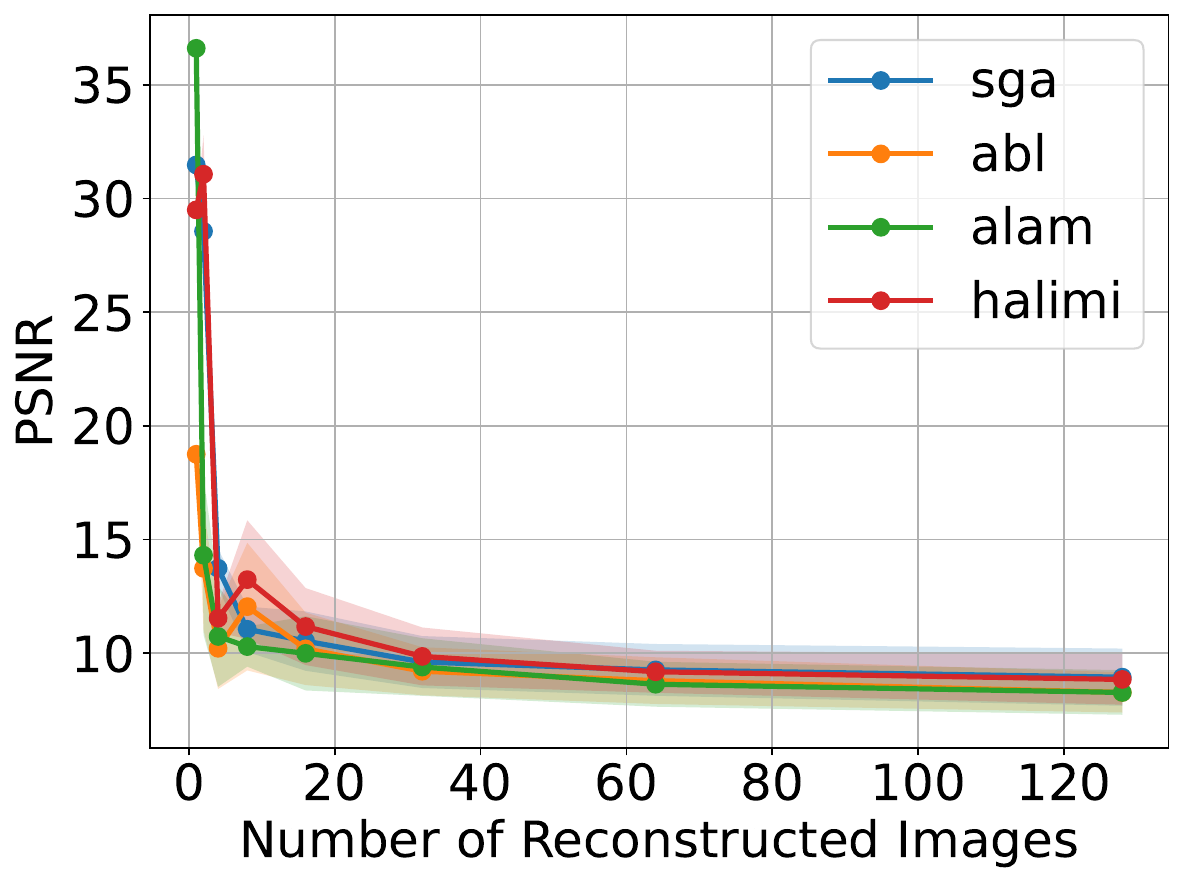}}\\
           
        \end{subfigure}
        \hspace{0.05cm}
        \begin{subfigure}[t]{0.19\textwidth}
            \centering
            \textbf{LPIPS}
            \fbox{\includegraphics[width=\linewidth]{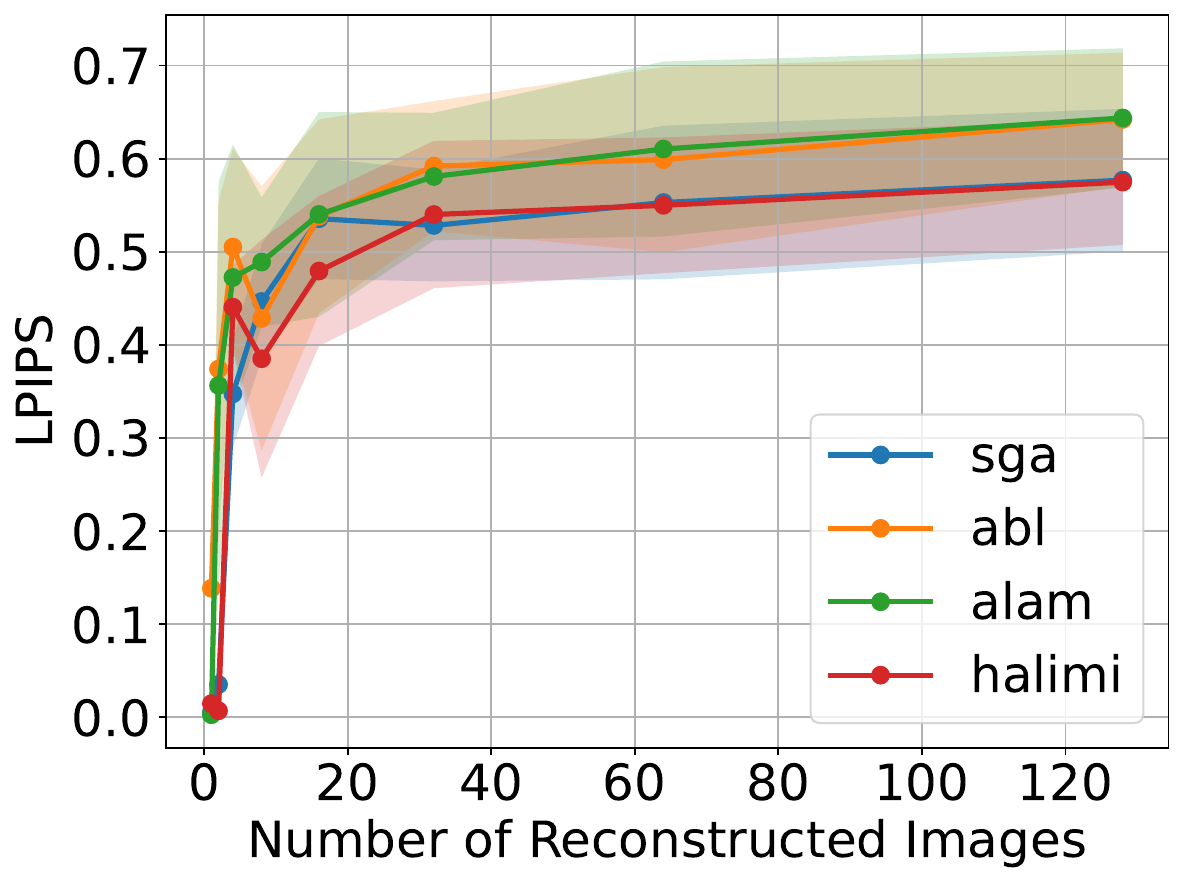}}\\
            
        \end{subfigure}
      
    \end{tabular}
    }
     \caption{Effect of batch size scaling on reconstruction quality for the MNIST dataset on four unlearning algorithms, measured using SSIM, PSNR, and LPIPS. Larger batch sizes tend to degrade reconstruction fidelity, as indicated by decreasing SSIM/PSNR and increasing LPIPS.}

     \label{fig:metric_figs_mnist}
\end{figure}

\subsection{Reconstruction Comparison between Algorithm-Specific and Algorithm-Agnostic}\label{appendix:algo_spec}
To evaluate the performance of \methodname{} under different unlearning configurations, we compare reconstructions obtained using Alam~\cite{manaar} in both agnostic and specific modes across four datasets: MNIST, FEMNIST, CIFAR-10, and CIFAR-100 (Figures~\ref{fig:agnostic_specific_cifar10}--\ref{fig:agnostic_specific_mnist}). While the agnostic mode assumes minimal knowledge about the internal details of the unlearning algorithm, the specific mode leverages full access to its optimization logic and hyperparameters. As expected, the specific setup generally leads to better reconstructions, benefiting from targeted penalties and aligned update dynamics.
\begin{figure}[H]
    \centering
    \resizebox{0.68\textwidth}{!}{%
    \begin{tabular}{c}
        \textbf{CIFAR-10} \\[0.2cm]
        \begin{subfigure}[t]{0.15\textwidth}
            \centering
            \textbf{\scriptsize Ground Truth}\vspace{0.09cm}\\
            \fbox{\includegraphics[width=\linewidth]{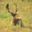}}
        \end{subfigure}
        \hspace{0.01cm}
        \begin{subfigure}[t]{0.15\textwidth}
            \centering
            \textbf{\scriptsize Agnostic}\vspace{0.05cm}\\
            \fbox{\includegraphics[width=\linewidth]{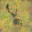}}\\
            {\tiny SSIM: 0.69, PSNR: 24.7, LPIPS: 0.16}
        \end{subfigure}
        \hspace{0.01cm}
        \begin{subfigure}[t]{0.15\textwidth}
            \centering
            \textbf{\scriptsize Specific}\vspace{0.05cm}\\
            \fbox{\includegraphics[width=\linewidth]{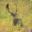}}\\
            {\tiny SSIM: 0.78, PSNR: 26.0, LPIPS: 0.13}
        \end{subfigure}
    \end{tabular}
    }
    
    \caption{\textit{\methodname}: Agnostic vs Specific reconstruction on CIFAR-10 using Alam~\cite{manaar} unlearning.}
    \label{fig:agnostic_specific_cifar10}
\end{figure}

\begin{figure}[H]
    \centering
    \resizebox{0.68\textwidth}{!}{%
    \begin{tabular}{c}
        \textbf{CIFAR-100} \\[0.2cm]
        \begin{subfigure}[t]{0.15\textwidth}
            \centering
            \textbf{\scriptsize Ground Truth}\vspace{0.09cm}\\
            \fbox{\includegraphics[width=\linewidth]{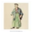}}
        \end{subfigure}
        \hspace{0.01cm}
        \begin{subfigure}[t]{0.15\textwidth}
            \centering
            \textbf{\scriptsize Agnostic}\vspace{0.05cm}\\
            \fbox{\includegraphics[width=\linewidth]{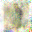}}\\
            {\tiny SSIM: 0.56, PSNR: 14.8, LPIPS: 0.44}
        \end{subfigure}
        \hspace{0.01cm}
        \begin{subfigure}[t]{0.15\textwidth}
            \centering
            \textbf{\scriptsize Specific}\vspace{0.05cm}\\
            \fbox{\includegraphics[width=\linewidth]{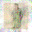}}\\
            {\tiny SSIM: 0.71, PSNR: 21.4, LPIPS: 0.29}
        \end{subfigure}
    \end{tabular}
    }
    
    \caption{\textit{\methodname}: Agnostic vs Specific reconstruction on CIFAR-100 using Alam~\cite{manaar} unlearning.}
    \label{fig:agnostic_specific_cifar100}
\end{figure}

\begin{figure}[H]
    \centering
    \resizebox{0.68\textwidth}{!}{%
    \begin{tabular}{c}
        \textbf{FEMNIST} \\[0.2cm]
        \begin{subfigure}[t]{0.15\textwidth}
            \centering
            \textbf{\scriptsize Ground Truth}\vspace{0.09cm}\\
            \fbox{\includegraphics[width=\linewidth]{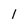}}
        \end{subfigure}
        \hspace{0.01cm}
        \begin{subfigure}[t]{0.15\textwidth}
            \centering
            \textbf{\scriptsize Agnostic}\vspace{0.05cm}\\
            \fbox{\includegraphics[width=\linewidth]{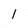}}\\
            {\tiny SSIM: 0.99, PSNR: 41.0, LPIPS: 0.00}
        \end{subfigure}
        \hspace{0.01cm}
        \begin{subfigure}[t]{0.15\textwidth}
            \centering
            \textbf{\scriptsize Specific}\vspace{0.05cm}\\
            \fbox{\includegraphics[width=\linewidth]{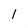}}\\
            {\tiny SSIM: 1.00, PSNR: 49.1, LPIPS: 0.00}
        \end{subfigure}
    \end{tabular}
    }
    
    \caption{\textit{\methodname}: Agnostic vs Specific reconstruction on FEMNIST using Alam~\cite{manaar} unlearning.}
     \label{fig:agnostic_specific_femnist}
\end{figure}

\begin{figure}[H]
    \centering
    \resizebox{0.68\textwidth}{!}{%
    \begin{tabular}{c}
        \textbf{MNIST} \\[0.2cm]
        \begin{subfigure}[t]{0.15\textwidth}
            \centering
            \textbf{\scriptsize Ground Truth}\vspace{0.09cm}\\
            \fbox{\includegraphics[width=\linewidth]{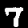}}
        \end{subfigure}
        \hspace{0.01cm}
        \begin{subfigure}[t]{0.15\textwidth}
            \centering
            \textbf{\scriptsize Agnostic}\vspace{0.05cm}\\
            \fbox{\includegraphics[width=\linewidth]{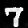}}\\
            {\tiny SSIM: 0.92, PSNR: 31.4, LPIPS: 0.01}
        \end{subfigure}
        \hspace{0.01cm}
        \begin{subfigure}[t]{0.15\textwidth}
            \centering
            \textbf{\scriptsize Specific}\vspace{0.05cm}\\
            \fbox{\includegraphics[width=\linewidth]{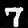}}\\
            {\tiny SSIM: 0.89, PSNR: 30.7, LPIPS: 0.01}
        \end{subfigure}
    \end{tabular}
    }
    
    \caption{\textit{\methodname}: Agnostic vs Specific reconstruction on MNIST using Alam~\cite{manaar} unlearning.}
    \label{fig:agnostic_specific_mnist}
\end{figure}

\section{\textit{\methodname} for Second-Order Unlearning Algorithms}\label{appendix:second_order}

We extend \methodname{} to the setting where the client employs second-order unlearning algorithms. While uncommon in federated learning due to their high computational cost, \citet{2ndord_fu} show that Newton-style updates can be applied to linearized MLPs. To reconstruct the unlearned data, \methodname{} generalizes classical gradient inversion attacks by simulating second-order updates on dummy data, as outlined in Algorithm~\ref{alg:second_order_reconstruction}. Specifically, the server approximates the client’s update by computing Hessian-vector products (HVPs) using dummy inputs and aligns them with the true HVPs by minimizing their cosine distance. To reduce the computational burden of exact Hessian inversion, we follow the approximation strategy from \citet{zhou2025hyperinf}. Reconstruction is optimized via SGD with total variation (TV) regularization. For simplicity, we assume access to unlearn labels during reconstruction. Theoretical justifications for the correctness of this approach and the failure modes of classical GIA under second-order updates are provided in Appendix~\ref{appendix:theory}.

\begin{algorithm}[H]
   \caption{\methodname{} for Second-Order Algorithms}
   \label{alg:second_order_reconstruction}
\begin{algorithmic}[1]
   \State \textbf{Client Input:} Local model $\theta_c$, datasets $(x_u, y_u), (x_r, y_r)$, learning rate $\eta$
   \State $\nabla \mathcal{L}_u \gets \nabla \mathcal{L}(x_u, y_u; \theta_s)$
   \State $H_r \gets H(x_r, y_r; \theta_s)$
   \State $\Delta \theta \gets H_r^{-1} \cdot \nabla \mathcal{L}_u$

   \Statex
   \hrulefill
   \Statex

   \State \textbf{Server Input:} Global model $\theta_s$, reconstruction step size $\eta_{\text{rec}}$, similarity loss $\mathcal{L}_{\text{sim}}$, TV weight $\lambda_{\text{TV}}$
   \State Initialize dummy variables $\tilde{x}_u, \tilde{y}_u, \tilde{x}_r, \tilde{y}_r$ randomly

   \For{$t = 1$ to $T$}
       \State $\nabla \tilde{\mathcal{L}}_u \gets \nabla \mathcal{L}(\tilde{x}_u, \tilde{y}_u; \theta_s)$
       \State $\tilde{H}_r \gets H(\tilde{x}_r, \tilde{y}_r; \theta_s)$
       \State $\tilde{\Delta \theta} \gets \tilde{H}_r^{-1} \cdot \nabla \tilde{\mathcal{L}}_u$

       \State $\ell \gets \mathcal{L}_{\text{sim}}(\Delta \theta, \tilde{\Delta \theta}) + \lambda_{\text{TV}} \cdot TV(\tilde{x}_u)$

       \State $\tilde{x}_u \gets \tilde{x}_u - \eta_{\text{rec}} \cdot \nabla_{\tilde{x}_u} \ell$
       \State $\tilde{x}_r \gets \tilde{x}_r - \eta_{\text{rec}} \cdot \nabla_{\tilde{x}_r} \ell$
   \EndFor

   \State \textbf{return} $(\tilde{x}_u, \tilde{y}_u, \tilde{x}_r, \tilde{y}_r)$
\end{algorithmic}
\end{algorithm}


\section{Additional Results on Defense Evaluation against \textit{\methodname}}\label{appendix:defense}
We evaluated the effectiveness of pruning and Gaussian noise defenses using the Alam~\cite{manaar} unlearning algorithm on CIFAR-10, with 90 training rounds followed by 10 unlearning rounds. As shown in Figures~\ref{fig:defense_comparison_tradeoff}, both defenses introduce a significant drop in model accuracy—on both retained and unlearned data—while only moderately increasing reconstruction difficulty. This illustrates a critical limitation: although these defenses reduce leakage, their impact on utility renders them unsuitable for practical federated unlearning deployments where accuracy must be preserved.

\begin{figure}[H]
    \centering
    \begin{minipage}[t]{0.48\textwidth}
        \centering
        \textbf{\small Pruning Effect on Reconstruction-Accuracy Trade-off}\vspace{0.1cm}\\
        \includegraphics[width=\linewidth]{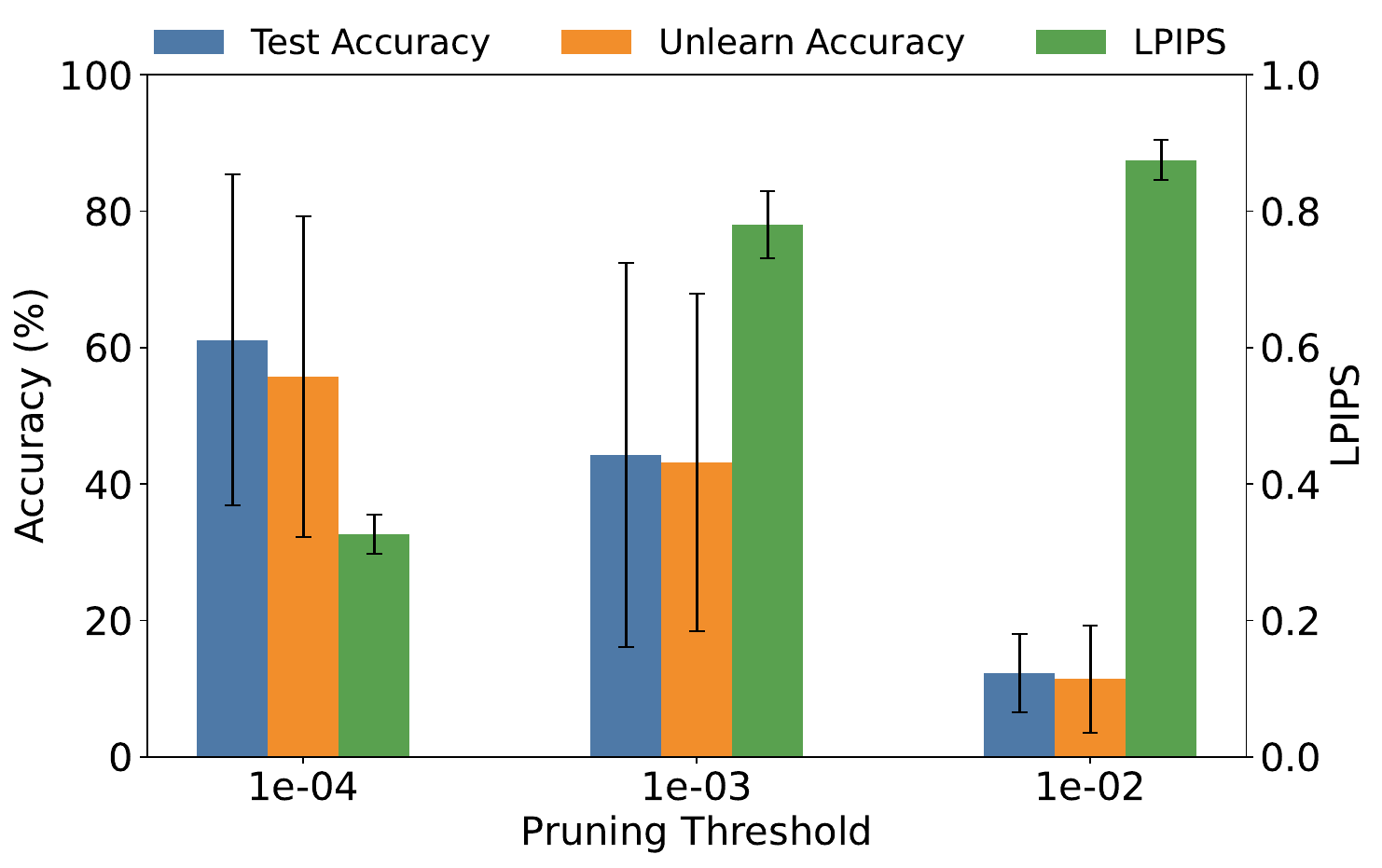}
    \end{minipage}
    \hspace{0.02\textwidth}
    \begin{minipage}[t]{0.48\textwidth}
        \centering
        \textbf{\small Gaussian Noise Effect on Reconstruction-Accuracy Trade-off}\vspace{0.1cm}\\
        \includegraphics[width=\linewidth]{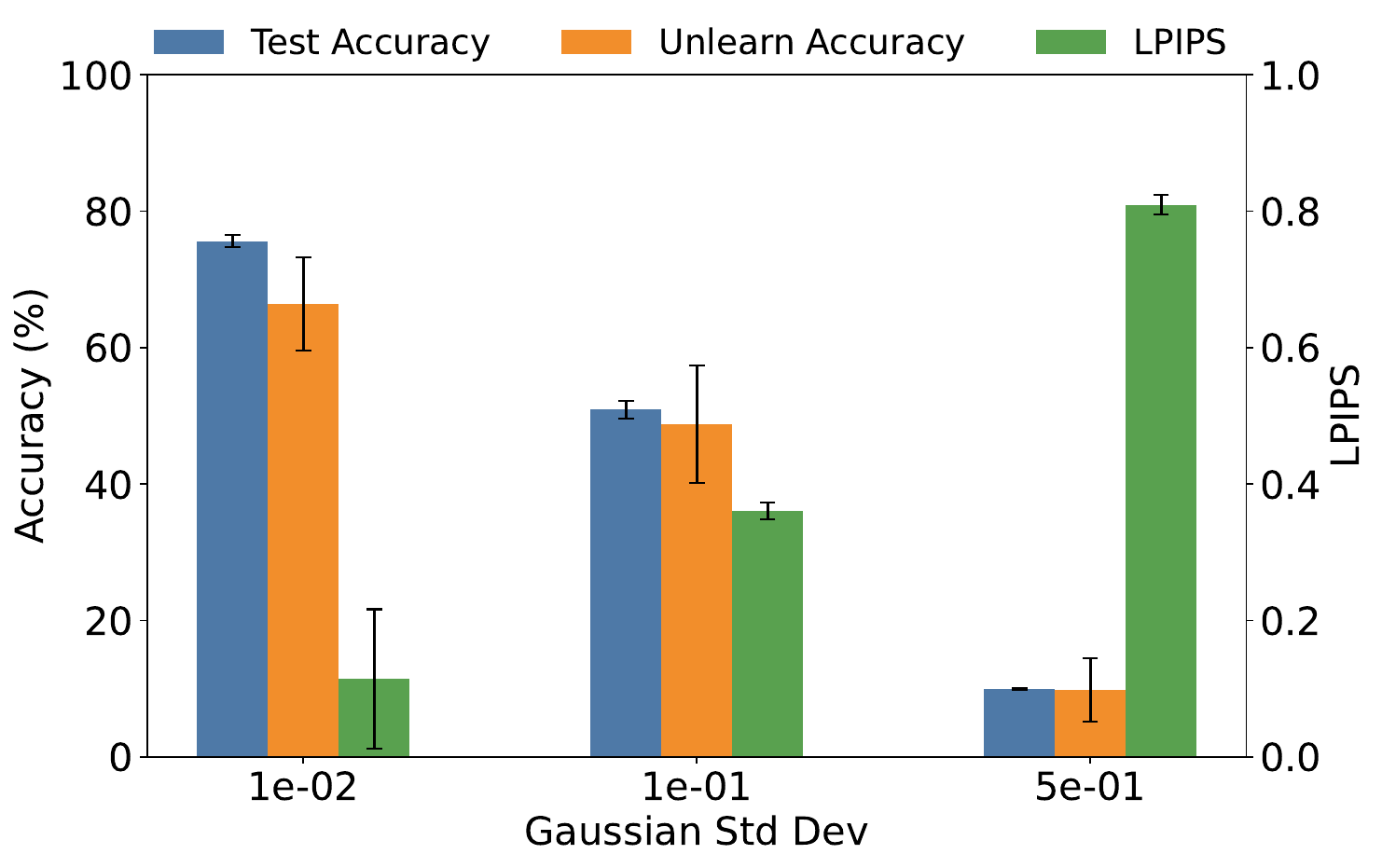}
    \end{minipage}
    \caption{\textit{\methodname} reconstruction-accuracy trade-off under pruning and Gaussian noise defenses.}
    \label{fig:defense_comparison_tradeoff}
\end{figure}

\section{Broader Impacts}

Recent works~\cite{DBLP:conf/sp/HuWDX24, DBLP:conf/nips/BertranTKM0W24} have shown that unlearning algorithms are vulnerable to data reconstruction attacks in Machine Learning as a Service (MLaaS) settings. However, this vulnerability remains unexplored in the context of federated unlearning—a setting where unlearning is critical due to data ownership and regulatory constraints. In this work, we demonstrate that federated unlearning is highly vulnerable when facing a malicious server. In realistic industrial scenarios, such as computer vision classification tasks, we show that the server can silently reconstruct the raw images that a client requests to unlearn—without interfering with the unlearning process. These images are often highly sensitive (e.g., medical X-rays) and typically involve a small number of samples, making accurate reconstruction more feasible. Our findings expose a serious privacy risk in current federated unlearning methods and call for a rethinking of their trust assumptions. We hope this work motivates future research on the privacy-utility trade-offs in federated unlearning, as well as the development of detection and mitigation strategies specifically tailored to this emerging threat model.

\end{document}